\pdfoutput=1
\documentclass[11pt]{article}
\usepackage{enumerate}
\usepackage[OT1]{fontenc}
\usepackage{amsmath,amssymb}
\usepackage{natbib}
\usepackage[usenames]{color}

\usepackage{pgfplots}

\usepackage{smile}
\usepackage{multirow}
\usepackage{rotating}
\usepackage{enumerate}
\usepackage{esvect}
\usepackage{tikz}
\usetikzlibrary{patterns}
\usetikzlibrary{arrows}

\def\supp{\mathop{\text{supp}}}

\long\def\comment#1{}

\def\cS{{\mathcal{S}}}

\def\oper{\mathop{\text{op}}}

\newcommand{\bel}{\begin{eqnarray}\label}
\newcommand{\eel}{\end{eqnarray}}
\newcommand{\bes}{\begin{eqnarray*}}
\newcommand{\ees}{\end{eqnarray*}}

\newcommand{\fR}{\mathfrak{R}}

\let\hat\widehat
\let\tilde\widetilde

\def\supp{\mathop{\text{supp}\kern.2ex}}
\def\argmin{\mathop{\text{\rm arg\,min}}}
\def\argmax{\mathop{\text{\rm arg\,max}}}

\def\given{{\,|\,}}

\def\supp{\mathop{\text{supp}}}

\usepackage{mathrsfs}
\usepackage{fullpage}

\renewcommand{\Pr}{\mathbf{Pr}}
\usepackage{float} 

\usepackage[protrusion=false,expansion=true]{microtype}

\def\##1\#{\begin{align}#1\end{align}}
\def\$#1\${\begin{align*}#1\end{align*}}

\usepackage{undertilde}
\usepackage{mathtools}
\usepackage{enumitem}
\usepackage{setspace}

\makeatletter
\def\@fnsymbol#1{\ensuremath{\ifcase#1\or *\or $1$\or $2$\or
  $3$\or $4$\or $5$ \or $6$
  \or $7$ \else\@ctrerr\fi}}
\makeatother
\usepackage[colorlinks,
linkcolor=red,
anchorcolor=blue,
citecolor=blue
]{hyperref} 
\usepackage{algorithm}
\usepackage{algorithmic}

\title{Learning Dynamic Mechanisms in Unknown Environments: A Reinforcement Learning Approach}

\begin{document}
\author{Shuang Qiu\footnote{Equal Contribution. Random order.} \thanks{The Hong Kong University of Science and Technology.
	Email: \texttt{masqiu@ust.hk}.} 
\qquad 
Boxiang Lyu${}^*$\thanks{University of Chicago.
Email: \texttt{blyu@chicagobooth.edu}.} 
\qquad Qinglin Meng${}^*$\thanks{Purdue University.  
Email: \texttt{meng160@purdue.edu}.} 
\qquad
Zhaoran Wang\thanks{Northwestern University.
Email: \texttt{zhaoranwang@gmail.com}.}
\\
Zhuoran Yang\thanks{Yale University.
Email: \texttt{zhuoran.yang@yale.edu}.}
\qquad 
Michael I. Jordan\thanks{University of California, Berkeley.
Email: \texttt{jordan@cs.berkeley.edu}.} 
}

\maketitle

\addtocontents{toc}{\protect\setcounter{tocdepth}{0}}


\begin{abstract}
	Dynamic mechanism design studies how mechanism designers should allocate resources among agents in a time-varying environment. We consider the problem where the agents interact with the mechanism designer according to an unknown Markov Decision Process (MDP), where agent rewards and the mechanism designer's state evolve according to an episodic MDP with unknown reward functions and transition kernels. We focus on the online setting with linear function approximation and propose novel learning algorithms to recover the dynamic Vickrey-Clarke-Grove (VCG) mechanism over multiple rounds of interaction. A key contribution of our approach is incorporating reward-free online Reinforcement Learning (RL) to aid exploration over a rich policy space to estimate prices in the dynamic VCG mechanism. We show that the regret of our proposed method is upper bounded by $\tilde{\cO}(T^{2/3})$ and further devise a lower bound to show that our algorithm is efficient, incurring the same $\Omega(T^{2 / 3})$ regret as the lower bound, where $T$ is the total number of rounds. Our work establishes the regret guarantee for online RL in solving dynamic mechanism design problems without prior knowledge of the underlying model.
\end{abstract}


\section{Introduction}
Mechanism design is a branch of economics studying the allocation of goods among rational agents~\citep{myerson1989mechanism}. Its sub-field, dynamic mechanism design, focuses on the setting where the environment, such as agents' preferences, may vary with time~\citep{bergemann2019dynamic}. It has attracted significant research interest from economists and computer scientists~\citep{pavan2014dynamic, parkes2003mdp} over decades. Many real-world problems, such as Uber's surge pricing, the wholesale energy market, and congestion control, have all been studied under this framework~\citep{chen2016dynamic, bejestani2014dynamic, barrera2014dynamic}. However, existing work usually requires prior knowledge of key parameters or functionals in the problem, such as the optimal policy or the agents' valuations of goods~\citep{parkes2003mdp, pavan2009dynamic}. Such requirements may be unrealistic in real life.

A promising emerging research direction is learning dynamic mechanisms from repeated interactions with the environment. Drawing inspiration from~\citet{bergemann2010dynamic} and~\citet{parkes2003mdp}, we propose the first algorithm that can learn a dynamic mechanism from repeated interactions via reinforcement learning (RL) with no prior knowledge of the problem. 

As a first attempt, we focus on learning a dynamic generalization of the classic Vickrey-Clarke-Groves (VCG) mechanism~\citep{vickrey1961counterspeculation, clarke1971multipart,groves1979efficient}. More specifically, we consider the case where the interaction between a group of agents and a single seller is  modeled as an episodic linear Markov Decision Process (MDP)~\citep{2019Provably, yang2019sample, 2020Is},
where the seller takes actions to determine the allocation of a class of scarce resources among agents. 
Our task is to learn an ideal mechanism from repeated interactions via online RL~\citep{2019Provably, 2019Provably0}. The mechanism we consider implements the policy that maximizes social welfare and charges each agent according to the celebrated Clarke pivot rule~\citep{clarke1971multipart}. A slight variant of the mechanism has been discussed under known MDP dynamics in~\citet{parkes2007online}, and we describe the mechanism in full detail in Section~\ref{sec:prob}.

A key challenge we resolve is estimating the VCG price without prior knowledge of the MDP.
In particular, the VCG price charged to each agent $i$ is characterized by the externality of that agent, that is, 
the difference between the maximum social welfare of the whole group and that when agent $i$ is absent ~\citep{karlin2017game, groves1979efficient}. 
In other words, it is the loss that an agent's participation incurs on other agents' welfare. Estimating the VCG price in our dynamic setting requires learning the optimal policy of the fictitious problem where agent $i$ is absent. 
Such a policy is never executed by the seller, and thus it is challenging to assess its uncertainty from data. 
Existing methods target to estimate the optimal policy well. However, they have no guarantees on how well they estimate the fictitious policies. Therefore it is impossible to accurately estimate VCG prices via a direct application of prior online RL algorithms~\citep{2019Provably, 2019Provably0, zhou2021nearly}.

To address this challenge, our algorithm incorporates a reward-free exploration subroutine to ensure sufficient coverage over the policy space, thereby reducing the uncertainty of all policies, ensuring that we can even reduce the uncertainty about the fictitious policies~\citep{jin2020reward, 2020On, 2021On, zhang2021near, kaufmann2021adaptive}. 
However, such a reward-free approach comes at a price---our proposed approach attains $\tilde{\cO}(T^{2/3})$ regret in terms of social welfare, agent utility, and seller utility, as opposed to the common $\tilde {\cO} (T^{1/2}) $ regret in online RL \citep{2019Provably}.
Moreover, we further derive a matching lower bound for the regrets, showing that our algorithm is minimax optimal up to multiplicative factors of problem-dependent terms. 

To summarize, our contributions are threefold. First, we develop the first reinforcement learning algorithm that can recover an optimal dynamic mechanism with no prior knowledge of the problem.
In particular, our algorithm is separated into two phases, namely, exploration and exploitation. In the exploration phase, we propose to learn the underlying model via reward-free exploration. Then, in the exploitation phase, the algorithm executes a data-driven policy by solving a planning problem using the collected dataset. Moreover, our algorithm is able to handle large state spaces by incorporating linear function approximation. Second, we prove that the proposed algorithm achieves sublinear regret upper bounds in terms of the various regret notions, such as the welfare regret and individual regret of the seller and buyers. Our algorithm is proven to approximately satisfy the three key mechanism design desiderata --- truthfulness, individual rationality, and efficiency. 
Finally, we demonstrate that the  $\tilde{\cO}(T^{2/3})$ regret has the minimax optimal dependency in $T$ by establishing a matching regret lower bound.
To our knowledge, we seem to establish the first provably efficient reinforcement learning algorithm for learning a dynamic mechanism. 

\subsection{Related Works}
There is a wealth of literature on dynamic mechanism design.~\citet{parkes2003mdp, parkes2004approximately} are two of the earliest works that analyze dynamic mechanism design from an MDP perspective, and the proposed mechanism is applied to a real-world problem in~\citet{friedman2003pricing}.~\citet{bergemann2006efficient} generalize the VCG mechanism based on the marginal contribution of each agent and derives a mechanism that is truth-telling in every period.~\citet{bapna2005efficient} focus on the dynamic auction setting and formulate the problem as a multi-arm bandit problem.~\citet{athey2013efficient} adapt the d'Aspremont-Gerard-Varet (AGV) mechanism~\citep{d1979incentives} to the dynamic setting and design an efficient, budget balanced, and Bayesian incentive compatible mechanism.~\citet{pavan2009dynamic} derive the first order conditions of incentive compatibility in dynamic mechanisms.~\citet{cavallo2008efficiency} devises a dynamic allocation rule for auctions in the multi-arm bandits setting, where a single good is distributed among agents over multiple rounds.~\citet{cavallo2009efficient} study the truthful implementation of efficient policies when agents have dynamic types.~\citet{pavan2014dynamic} extend the seminal work of~\citet{myerson1989mechanism} and characterize perfect Bayesian equilibrium-implementable allocation rules in the dynamic regime.~\citet{cavallo2009mechanism,bergemann2015introduction,bergemann2019dynamic} provide useful surveys of dynamic mechanism research. \citet{2020Mechanism} studies online learning of the VCG mechanism with stationary multi-arm bandits. Our work considers a more challenging setting modeled by an episodic MDP, where the agents' rewards are state-dependent and may evolve over time within each episode. More importantly,~\citet{2020Mechanism} estimates the VCG price via uniformly exploring over all arms, 
which cannot be directly applied to the dynamic setting \citep{2020On}.
Rather than uniformly bounding the uncertainty over all \emph{actions}, our approach bounds the uncertainty over all implementable \emph{policies} via a variant of least-squares value iteration and enjoys provable efficiency under the function approximation setting. Distinct from the major focus of our work, \citet{simchowitz2023exploration} studies online mechanism design with MDPs from a rather different angle. In their work, the mechanism designer encourages exploration by sending specific information to the agents. More specifically, the agents initially have beliefs or prior distributions over the MDP's parameters. The mechanism designer can reveal to the agents some information, such as information about the MDP's transition and reward. The agents then update their beliefs about the underlying MDP and execute the optimal policy according to the updated beliefs or their posterior distribution over the MDP's parameters. The goal is to incentivize agents to explore by controlling the information they receive. However, our work focuses on implementing the welfare-maximizing policy among a group of agents by controlling the price that each agent pays. In other words, theirs focuses on adjusting information, whereas ours focuses on adjusting price. Additionally, our work focuses on a more general linear MDP than the tabular MDP studied in their work.

There are many recent works concerning provably efficient RL for MDPs with linear structures in the absence of generative models~\citep{yang2019sample,du2019good,yang2020reinforcement,2019Provably, 2019Provably0}.~\citet{2019Provably} provides the first provably efficient RL algorithm for linear MDPs that incorporates exploration.~\citet{zhou2021provably} provides a provably efficient algorithm for infinite-horizon discounted linear MDPs.~\citet{ayoub2020model} studies a model-based regime where the transition kernel belongs to a family of models known to the learning agent.~\citet{zhou2021nearly} proposes a computationally efficient nearly minimax optimal algorithm for the linear MDP whose transition kernel is a linear mixture model. These works require (noisy) feedback of the reward function in the learning process.  

Reward-free exploration in reinforcement learning has recently attracted a lot of attention, in which the agents explore the environment without any feedback of the reward. Specifically, \citet{jin2020reward} introduces the problem of reward-free exploration in RL and proposes a sample-efficient algorithm for tabular MDPs.~\citet{menard2021fast,kaufmann2021adaptive} provide improved algorithms and tighter rates, also for tabular MDPs.~\citet{zhang2021near} further improves the analysis and obtains nearly minimax-optimal sample complexity bounds.~\citet{2020On,zanette2020provably,chen2021near,wagenmaker2022reward} study reward-free RL algorithms for linear or linear mixture MDPs and~\citet{2021On} for kernel and neural function approximations. Moreover, \citet{kong2021online} proposes reward-free algorithms for RL with general function approximation under the setting of bounded eluder dimension.
\citet{miryoosefi2021simple} investigates the problem of reward-free RL with constraints. 
\citet{wu2021accommodating} then proposes a reward-free algorithm for the multi-objective RL problem. In addition, \citet{bai2020provable,liu2021sharp,2021On} further study the reward-free RL algorithms under the multi-agent setting.

Furthermore, we would like to emphasize that directly extending the existing results on reward-free exploration (see, e.g., \citet{2020On,2021On}) to learning the dynamic VCG mechanism seems infeasible.
The main reason is that these works focus only on estimating the optimal value functions corresponding to different reward functions.  
In contrast, in the context of mechanism design, we have multiple desiderata, namely truthfulness, individual rationality, and efficiency, which mathematically translates into the various regret notions, such as the welfare regret and individual regret of the seller and the buyer. 
Showing the proposed algorithm approximately satisfies these desiderata requires bounding these regrets using the properties of the dynamic VCG mechanism as well as the results of reward-free exploration. 
Finally, the recent work \citet{lyu2022pessimism} focuses on learning the Markov VCG mechanism via offline RL from a set of collected trajectories.
Under the offline setting, exploration is out of the scope, and thus our core challenge caused by the fictitious policy is absent in \citet{lyu2022pessimism}.

\section{Problem Setup}\label{sec:prob}
Consider an episodic MDP defined by $\cM(\cS,\cA,H,\cP,r)$, where $\cS$ and $\cA$ are state and action spaces, $H$ the length of each episode, $\cP=\{\cP_{h}\}_{h=1}^{H}$ the transition kernel, and $r=\{r_{i,h}\}_{i=0,h=1}^{n,H}$ the reward functions.
We use $r_{0,h}:\cS\times\cA\mapsto [0, R_{\max}]$ to denote the reward function of the seller at the step $h$ 
and let $r_{i,h}:\cS\times\cA\mapsto [0,1]$ be the reward function of agent (buyer) $i$ at the step $h$ for $i \in [n]$, where $n$ is the number of agents and $[n]$ denotes $\{1,2,\cdots,n\}$. In addition, we assume that the reward observation is stochastic and the underlying reward function is the expectation of its stochastic observation, i.e., the reward observation at $(s,a)\in \cS\times\cA$ can be represented by $r_{i,h}(s, a; \omega)$ with $r_{i,h}(s,a)=\EE_\omega[r_{i,h}(s, a; \omega)]$, where $\omega$ is an independent random variable indicating the exogenous randomness for the reward observation. We further assume that the boundedness holds for the reward observation as $r_{0,h}(\cdot, \cdot; \omega):\cS\times\cA\mapsto [0, R_{\max}]$ and $r_{i,h}(\cdot, \cdot; \omega):\cS\times\cA\mapsto [0, 1], \forall i \in [n]$ at all steps $h\in [H]$, where rewards for the seller and agents may have different scales\footnote{We allow different reward scales for greater flexibility within our framework.}.

Let $\pi=\{\pi_{h}\}_{h=1}^{H}$ denote the seller's policy, where for each $h\in[H]$, $\pi_{h}:\cS\mapsto \cA$ maps a given state to an action. For each step $h\in[H]$, reward function $r= \{r_{h}\}_{h=1}^H$, and a given policy $\pi$, we define the \textit{value function} $V_{h}^{\pi}(\cdot; r):\cS\mapsto \RR$ for all $x \in \cS$ as
$V_{h}^{\pi}(x;r):=\sum_{h^{\prime}=h}^{H}\EE\sbr{r_{h^{\prime}}(x_{h^{\prime}},\pi_{h^{\prime}}(x_{h^{\prime}}))|x_{h}=x}$, where the expectation is taken over states $x_{h+1}\sim \cP_h(\cdot|x_h,\pi_h(x_h)), x_{h+2}\sim \cP_h(\cdot|x_{h+1},\pi_{h+1}(x_{h+1})),\ldots,x_H\sim \cP_H(\cdot|x_H,\pi_H(x_H))$ conditioned on a starting state $x_h=x$ at step $h$.
Here we write $V_{h}^{\pi}(\cdot;r)$ to highlight that the definition of the value function depends on a given reward function $r$.
We also define the corresponding Q-function $Q_{h}^{\pi}(\cdot,\cdot; r):\cS\times\cA\mapsto \RR$ for all $(x, a) \in \cS \times \cA$ as
$Q_{h}^{\pi}(x,a;r) :=r_{h}(x,a)+\sum_{h^{\prime}=h+1}^{H}\EE\sbr{r_{h^{\prime}}(x_{h^{\prime}},\pi_{h^{\prime}}(x_{h^{\prime}}))\big|(x_{h},a_{h})=(x,a)}$, where the expectation is also taken over states $x_{h+1}, \ldots, x_H$ sampled from the transition model $\cP$, conditioned on a starting state-action pair $(x_h,a_h)=(x,a)$ at step $h$. 

We stress that while the MDP we consider contains multiple reward functions and interaction between multiple agents, our setting differs from the Markov game setting, as we assume that the seller is the only participant who can take actions~\citep{littman1994markov}.

\paragraph{Dynamic Mechanism Design.} We now describe how agents interact with the mechanism designer (seller) in our setting. At the beginning of each episode, the mechanism starts from the initial state $x_{1}$. At each step $h\in[H]$, the seller observes some state $x_{h}\in\cS$, picks an action $a_{h}\in\cA$, and receives a stochastic reward with mean $r_{0,h}(x_{h},a_{h})$. 
Each agent (buyer) receives their own reward, each with an expected value of $r_{i,h}(x_{h}, a_{h})$, and reports a stochastic reward with a mean $\tilde{r}_{i,h}(x_{h},a_{h})$, given by some potentially untruthful reward function $\tilde{r}_{i,h}(\cdot, \cdot)$.
At the end of each episode, the seller charges each customer some price $p_{i}$. For any policy $\pi$ and prices $\{p_{i }\}_{i=1}^{n}$, we define agent $i$'s utility function as
\begin{align}
u_{i }:= \EE\bigg[\sum_{h = 1}^H r_{i, h}( x_h, a_h)\bigg] - p_i = V_{1}^{\pi}(x_{1};r_{i})-p_{i }. \label{eq:utilityi}
\end{align}
That is, agent $i$'s utility equals the difference between the expected total reward and the charged price. 
The seller's utility is then defined as
\begin{align}
u_{0 }:=V_{1}^{\pi}(x_{1};r_{0})+\sum_{i=1}^{n}p_{i }. \label{eq:utility0}
\end{align}
The social welfare, $W^{\pi}$, is defined as the sum of the agents and the seller's utilities, given by
\begin{align}
W^{\pi}(x_1)=\sum_{i=0}^{n}V_{1}^{\pi}(x_{1};r_{i}) = V^{\pi}\bigg(x_1; \sum_{i = 0}^n r_i\bigg), \label{eq:so-welf}
\end{align}
which is equivalent to the expectation of the sum of all rewards as the prices cancel out. 

\paragraph{Markov VCG Mechanism.} Suppose that the transition kernel is known, all agents and the seller know their own reward functions $r_{i,h}$ for all $(i,h)\in [n]\times[H]$, and the agents' reward functions are known by the seller. The VCG mechanism demands that we choose the welfare-maximizing policy $\pi_{*}$ that the seller executes each episode. Each agent $i$ is subsequently charged a price $p_{i *}$, which is the loss her presence causes to others. Hence we have the following mechanism:
\begin{align}
	\begin{aligned}\label{eq:vcg-mech}
		&\pi_{*}:=\argmax_{\pi}V_{1}^{\pi}(x_{1};R),~~\pi_{*}^{-i}:=\argmax_{\pi}V_{1}^{\pi}(x_{1};R^{-i}),\\
		& p_{i *}:=V_{1}^{\pi_{*}^{-i}}(x_{1};R^{-i})-V_{1}^{\pi_{*}}(x_{1};R^{-i}),
	\end{aligned}
\end{align}
where we define the total reward function $R$ and the sum of reward except agent $i$, $R^{-i}$, as 
\begin{align*}
	R=\sum_{i=0}^{n}r_{i}\qquad \text{and} \qquad R^{-i}=\sum_{j=0,j\neq i}^{n}r_{j}.   
\end{align*}
Here $\pi_*$ is the welfare-maximizing policy, i.e., the optimal policy for the reward function $R$, while $\pi_*^{-i}$ is the fictitious policy that maximizes welfare when agent $i$ is absent. Estimating the latter and their corresponding value functions requires the algorithm to explore in directions not aligned with the social welfare maximizing policy, $\pi_*$, thus necessitating the reward-free component of our algorithm. These prices, namely $p_{i*}$, can be estimated by following Equation \eqref{eq:vcg-mech} once the value functions corresponding to policies $\pi_*$, $\pi_*^{-i}$ and reward functions $R, R^{-i}$ are estimated sufficiently well via our algorithm. As these value functions are deterministic, the resulting pricing function is also deterministic.

The following lemma introduces the properties of the  Markov VCG mechanism. 

\begin{lemma}\label{lemma:Markov VCG mechanism}
	The Markov VCG mechanism satisfies the following desiderata in mechanism design:
	\begin{enumerate} [leftmargin=*,itemsep=0pt,parsep=1pt]
		\item\textit{Truthfulness}: A mechanism is truthful if the utility $u_{i}$ of agent $i$ is maximized when, regardless of other agents' reported rewards, agent $i$ reports her rewards truthfully.
		\item\textit{Individual rationality}: A mechanism is individually rational if the utility $u_{i}$ of agent $i$ is non-negative when agent $i$ is truthful.
		\item\textit{Efficiency}: A mechanism is efficient if the mechanism maximizes the welfare when all agents are truthful.
	\end{enumerate}
	An agent is truthful if she submits her reward functions truthfully. 
\end{lemma} 
Please see Appendix \ref{sec:proof_sketch_of_lemma_markov_vcg_mechanism} for the proof. 
Our proposed pricing formula $p_{i *}:=V_{1}^{\pi_{*}^{-i}}(x_{1};R^{-i})-V_{1}^{\pi_{*}}(x_{1};R^{-i})$ is not the only pricing rule that ensures Lemma~\ref{lemma:Markov VCG mechanism}. Nevertheless, our proposed algorithm can be generalized to any pricing rule of the form $p_{i}' = V_{1}^{\pi^{-i}}(x_{1};R^{-i})-V_{1}^{\pi_{*}}(x_{1};R^{-i})$, where $\pi^{-i}$ is not necessarily the $\pi_*^{-i}$ defined above, but can be any arbitrary policy independent of agent $i$. Intuitively, as our algorithm makes use of reward-free exploration, we can sufficiently accurately estimate the value functions for arbitrary policies, including both $\pi^{-i}$ and $\pi_*^{-i}$. Consequently, our approach can be extended to a general class of pricing functions that use different policies' value functions as prices.
\paragraph{Mechanism Design with an Unknown MDP.} Consider the setting where the agents' value functions and the MDP's transition kernel are unknown, and the procedure is repeated for multiple rounds. 
At round $t$, the mechanism choose a policy $\pi^t$ and set prices $\{p_{i t}\}_{i=1}^{n}$ for the agents. Following Equations \eqref{eq:utilityi} and \eqref{eq:utility0}, the utilities of agent $i$ and the seller at round $t$ are 
\begin{align*}
u_{i t}=V_{1}^{\pi^t}(x_{1};r_{i})-p_{it} \qquad\text{and}\qquad u_{0 t}=V_{1}^{\pi^t}(x_{1};r_{0})+\sum_{i=1}^{n}p_{i t}.    
\end{align*}
We then denote their summations over $T$ rounds as 
\begin{align*}
U_{i T}=\sum_{t=1}^{T}u_{i t}\qquad \text{and} \qquad U_{0 T}=\sum_{t=1}^{T}u_{0 t}.
\end{align*}
Our goal is to design an algorithm that respects the three mechanism design desiderata over multiple rounds even when the true reward functions and transition kernels are unknown, as well as achieving sublinear regret for the agents, the seller, and the welfare. The following metrics are used to quantify the algorithm's performance:
\begin{equation}
\label{equa:regret}
\begin{aligned}
	&\mathrm{Reg}_{T}^{W}=T  V_1^{\pi_*}(x_{1};R)-\sum_{t=1}^{T}V_{1}^{\pi^t}(x_{1};R)\\
	&\mathrm{Reg}_{0 T}=T u_{0 *}-U_{0 T}, \qquad \mathrm{Reg}_{i T}=T u_{i *}-U_{i T},\qquad
	\mathrm{Reg}_{T}^{\sharp}=\sum_{i=1}^{n}\mathrm{Reg}_{i T}.
\end{aligned}
\end{equation}
Here we let $u_{0 *}=V_{1}^{\pi_*}(x_{1};r_{0})+\sum_{i=1}^{n}p_{i *}$ and $u_{i *}= V_{1}^{\pi_*}(x_{1};r_{i})-p_{i *}$ be the utilities of the seller and agent $i$ respectively in the VCG mechanism. Moreover, $\mathrm{Reg}_{T}^W$ is the welfare regret over $T$ rounds, $\mathrm{Reg}_{0 T}$ the seller regret, and $\mathrm{Reg}_{i T}$ the agent $i$'s regret, respectively. We let $\mathrm{Reg}_{T}^{\sharp}$ be the summation of regrets over all agents. 

Although the Markov VCG mechanism that we learn is welfare-maximizing, we focus on how this mechanism can be recovered. Consequently, the learning algorithm's objective is not welfare maximization alone. Maximizing welfare increases the total utility by definition and, therefore, increases the total utility that the agents and the seller share. As our learning process involves the seller and multiple agents, we also need to ensure that it faithfully respects their utilities over $T$ rounds of interaction. Otherwise, it may be unfair to either the agents or the seller. Therefore, we measure the performance of our learning algorithm through the three terms, $\mathrm{Reg}_{T}^{W}, \mathrm{Reg}_{T}^{\sharp}, \mathrm{Reg}_{0 T}$, rather than any single objective by itself.  We note that all three regrets are $0$ under the Markov VCG mechanism.

Due to our need to approximate the VCG price $p_{i*}$, the welfare regret $\textrm{Reg}^W_T$ differs in scale from both $\textrm{Reg}^\sharp_T$ and $\textrm{Reg}_{0T}$, whereas the latter two are of the same scale. Notice that estimating $p_{i*}$ involves estimating the maximum welfare that the remaining $n - 1$ agents achieve when agent $i$ is absent and the welfare that these agents receive under $\widehat{\pi}^{t}$. Thus, the estimation error for $p_{i*}$ is roughly in the same order as the instantaneous welfare regret $V_1^{\pi_*}(x_1; R) - V_1^{\widehat{\pi}^t}(x_1; R)$ at round $t$, since both require good estimates of the summation of the value functions over all agents rather than a single agent. Consequently, recalling $\textrm{Reg}^\sharp_T$ is the summation of all agents' regrets and $\textrm{Reg}_{0T}$ equals the summation of the price estimation error across all $n$ agents, the terms $\textrm{Reg}^\sharp_T$ and $\textrm{Reg}_{0T}$ are in fact in the order of $n$ times the welfare regret $\textrm{Reg}^W_T$. Therefore, we add a scaling factor $n$ in front of the welfare regret, and our learning algorithms focus on minimizing
\[
\max\{n\textrm{Reg}^W_T, \textrm{Reg}_T^\sharp, \textrm{Reg}_{0T}\}.
\]



In addition to attaining small regret bounds, we aim to approximately satisfy the desiderata in Lemma \ref{lemma:Markov VCG mechanism} for the mechanism design. We define the approximate versions of truthfulness, individual rationality, and efficiency concerning the agent's \emph{cumulative} utility $U_{iT}$ as follows:
\begin{enumerate} [leftmargin=*,itemsep=0pt,parsep=1pt]
\item \textit{Approximate truthfulness}: Let $U_{i T}$ be the cumulative utility when agent $i$ is truthful and $\Tilde{U}_{i T}$ that when agent $i$ is untruthful. The mechanism is $\delta$-\textit{approximately truthful} if $\tilde{U}_{i T} - U_{i T} \leq \delta$, regardless of others' truthfulness.
\item\textit{Approximate individual rationality}: When agent $i$ reports truthfully, the mechanism is $\delta$-\textit{approximately individually rational} if $U_{it} \geq -\delta$, regardless of others' truthfulness.
\item\textit{Approximate efficiency}: The mechanism is $\delta$-\textit{approximately efficient} if $\mathrm{Reg}_{T}^{W} \leq \delta$ when all agents are truthful.
\end{enumerate}

When an agent adopts an untruthful reward-reporting strategy, it means that this agent reports her rewards under a different reward function $\tilde{r}_{i h}$ rather than the true reward function $r_{i h}$. As the algorithm interacts with the environment over $T$ rounds, these approximate desiderata can have a dependence on $T$. Our definition generalizes the asymptotic versions of the desiderata defined in ~\citet{2020Mechanism} since the approximate desiderata naturally imply their asymptotic counterparts when $\delta$ is sublinear in $T$. More specifically, as long as $\lim_{T \to \infty} f(T) / T = 0$, if a mechanism is $f(T)$-approximate truthful, when amortized over these $T$ rounds of interaction, agents' utility gain from untruthful reports vanishes. In other words, in the long run, agents cannot improve upon their average per-episode utility by untruthfulness, thus deterring rational agents from attempting to alter the learning process via untruthfulness. Similarly, if $f(T)$ is sublinear and the mechanism is $f(T)$-approximately individually rational, then in the long run, agents' average episodic utility is lower-bounded by a number tending to zero (i.e., $\lim_{T \to \infty} \frac{1}{T}U_{iT} \geq -\lim_{T \to \infty}f(T) / T = 0$), ensuring they will not be worse-off from participating.

Since approximate truthfulness implies, for suitable $f(T)$, that agents will not benefit from untruthful reporting in the long run, our definition of approximate efficiency focuses only on truthful agents. Indeed, consider the extreme case where all agents report $1 - r_{i, h}(x, a)$ instead of $r_{i, h}(x, a)$ and the seller reward is always 0. Under this extreme case of untruthful behavior, the welfare-maximizing policy under the untruthful report is in fact the welfare-minimizing policy under truthful reports, showing that it is in general hard to obtain efficiency guarantees without assuming truthful behavior. 
Such an approach, namely, first showing that the mechanism is approximately truthful and then providing guarantees under the assumption that the reports are truthful, is common in existing literature at the intersection of mechanism design and learning~\citep{nazerzadeh2008dynamic, 2020Mechanism}. We refer interested readers to~\citet{epasto2018incentive}, which justifies in further detail why agents will behave truthfully under approximately truthful mechanisms.

To handle the potentially large state and action spaces $\cS, \cA$, our work focuses on the linear function approximation setting, where the linear MDP is considered.

\paragraph{Linear MDP.} We assume that there exist a feature map $\phi :\cS\times\cA\mapsto\RR^{d}$, $d$ unknown measures $\boldsymbol{\mu}_{h}=(\mu_{h}^{1},\cdots,\mu_{h}^{d})$ over $\cS$ for any $h\in[H]$, and $n+1$ unknown vectors $\{\boldsymbol{\theta}_{i h}\}_{i=0}^{n}$ with each $\boldsymbol{\theta}_{i h}\in\RR^{d}$ for all $h\in[H]$. For any $(x,a, x')\in\cS\times\cA\times\cS$, the transition kernel and reward function can be linearly represented as
\begin{equation}\label{equa:Linear MDP}
\begin{aligned}
	&\cP_{h}(x'|x,a)=\langle\phi(x,a),\boldsymbol{\mu}_{h}(x')\rangle\\
	&r_{i,h}(x,a)=\langle\phi(x,a),\boldsymbol{\theta}_{i h}\rangle,\quad \forall i=0,1,\cdots,n. 
\end{aligned}
\end{equation}
Following standard assumptions in the prior literature~\citep{2019Provably, 2020Is}, we assume $\|\phi(x,a)\|\leq 1$ for all $(x,a)\in\cS\times\cA$, $\max\{\|\boldsymbol{\mu}_{h}(\cS)\|,\|\boldsymbol{\theta}_{i h}\|\}\leq \sqrt{d}$ for all $h\in[H] , 0\leq i\leq n$. Recall that the linear MDP assumption implies that the value functions and action-value functions are both linear in the feature space defined by $\phi$~\citep{2019Provably}. When the problem reduces to the tabular setting, we have $d = |\cS| |\cA|$ with $\phi(x,a)=\mathbf{e}_{x,a}\in \RR^{|\cS| |\cA|}$ being an indicator vector.

\begin{remark}
When linear function approximation is considered, a typical assumption is that the underlying MDP has a linear structure. Here we assume the MDP satisfies Equation \eqref{equa:Linear MDP}. As discussed above, the tabular MDP can be covered as a special case of the linear MDP. Thus, our method for the linear MDP can also solve problems modeled by the tabular MDP. In realistic and complex scenarios, the underlying MDP may not be strictly linear. One can still apply the linear function approximation along with introducing a misspecification error. This error can be characterized by $\sup_{x,a}\|\cP_{h}(\cdot|x,a)-\langle\phi(x,a),\boldsymbol{\mu}_{h}(\cdot)\|_\mathrm{TV} \leq \cE_\cP$ and $\sup_{i,x,a}\|r_{i,h}(x,a)-\langle\phi(x,a),\boldsymbol{\theta}_{i h}\rangle\|_\mathrm{TV} \leq \cE_r$ as commonly discussed in prior RL literature (e.g., Jin et al. (2019)), where $\|\cdot\|_\mathrm{TV}$ denotes the total variation. By making small changes to our current analysis, extra misspecification terms containing $\cE_\cP$ and $\cE_r$ will be added to our regret bounds. If both $\cE_\cP$ and $\cE_r$ are small, the underlying MDP is approximately linear such that the extra terms can be considered minor.
\end{remark}

\subsection{Motivating Examples}
\label{subsec:motivating_examples}
We provide several motivating examples for the dynamic mechanism design introduced above, which are the potential application areas for our proposed algorithm. 

\paragraph{Dynamic Sponsored Search Auction.} We assume the state $x$ includes information on the agents’ remaining budgets for the episode. Let $H$ be a fiscal year. As advertisements’ values change within a single year (e.g.,  value increases around Black Friday), agents’ rewards from advertising naturally change with time. The seller’s action would affect the agents’ budgets, which would further affect their valuations: an agent who did not win any auction in previous rounds would have a high remaining budget near the end of the year and, therefore, would be willing to pay more for each advertisement slot in an effort to increase their odds of winning.

\paragraph{Dynamic Platform-as-a-Service (PaaS).} We assume there are multiple users using the same computing cluster and a central planner who allocates computation resources to these users. The state $x$ includes information on the server’s current load, and action $a$ reflects how the central planner allocates these resources among users. Naturally, the planner’s action affects the server load in the next state. While a higher server load would provide users with immediate satisfaction, it would also incur higher electricity costs for the planner. As the users’ demands may fluctuate within a day (for instance, demands are lower during the night), it is a significant challenge for the planner to balance electricity costs and user satisfaction in an environment with the users’ valuations and demands constantly changing. The problem is further complicated by the fact that the service provider only learns user satisfaction after the resources are allocated, justifying our setup above.

\paragraph{Dynamic Public Service.} This example is inspired by Section 9.3.5.5 in \citet{nisan2007algorithmic}. Here the seller takes the form of a government body, and the agents are the citizens. The seller wishes to provide public services to benefit the general population, and the agents pay the seller in the form of taxation. The state $x$ contains information on the seller’s remaining budget for the year as well as the agents’ satisfaction with the seller. When the seller does not provide sufficient public service, agents will become less satisfied and have more urgent demands for public services in later steps, exhibiting natural transition dynamics. As the seller can only learn the agents’ valuation after the service has been provided, the problem fits naturally within the setting considered above.

\paragraph{Relationship to~\citet{parkes2003mdp}.} Finally, our work could address several key problems raised by prior works on dynamic mechanism design without assuming prior knowledge of the underlying model. \citet{parkes2003mdp} studies an online mechanism design problem by formulating the problem as an MDP and proposes Wi-Fi pricing at Starbucks as a motivating example. \citet{parkes2003mdp} assumes that the welfare-maximizing policy is known a priori. However, the MDP in \citet{parkes2003mdp} is an infinite-horizon, un-discounted, and non-average reward one, and we are not aware of any existing literature that can provably learn nearly optimal policies in this setting. We thus leave the question as a future direction of independent interest. Nevertheless, our work takes a first step towards relaxing the assumption by requiring the mechanism designer to recover the policy from repeated interaction in the finite horizon case.

\section{Algorithm} \label{sec:alg}

In this section, we introduce our proposed algorithm for VCG mechanism learning on linear MDPs (\texttt{VCG-LinMDP}). The general learning framework of our algorithm is summarized in Algorithm \ref{algorithm:LMVL}, comprising two phases: the exploration phase and the exploitation phase. The exploration and exploitation phases are summarized in Algorithms \ref{algorithm:L1}, \ref{algorithm:L3}, and \ref{algorithm:L4}.

\subsection{Algorithmic Framework} \label{subsec:algframe}

\paragraph{Markov VCG with Function Approximation.} 
In order to learn the Markov VCG mechanism, we consider a learning framework with function approximation, in which the reward-free exploration phase aims to efficiently explore the environment with wide coverage over the underlying policy space. The exploitation phase targets at utilizing the collected data to update the seller's policy and estimate the prices charged to the agents. We remark that this learning framework is general and can fit \emph{any linear or nonlinear} function approximators. We summarize it as follows:

\begin{enumerate}[leftmargin=*,itemsep=0pt,parsep=1pt]
	\item Exploration for multiple rounds to collect an initial dataset. The exploration is performed via a reward-free least-square value iteration (LSVI) with function approximation \citep{jin2020reward,2020On,2021On}.
	
	\item Exploitation with the collected data. At each round $t$ of the exploitation phase:
	\begin{itemize}[leftmargin=*,itemsep=0pt,parsep=1pt]
		\item Update the seller's policy $\widehat{\pi}^t$ via a planning subroutine implemented as optimistic LSVI with function approximation w.r.t. the reward function $R$.
		
		\item Update $F_{t}^{-i}$ by the value function from a planning subroutine implemented as optimistic or pessimistic LSVI with function approximation w.r.t. $R^{-i}$.
		
		\item Update $G_{t}^{-i}$ by the value function from a policy evaluation subroutine by optimistic or pessimistic evaluation with function approximation at the learned policy $\widehat{\pi}^t$ w.r.t. $R^{-i}$.
		
		\item Estimate the price $p_{i t}= F_{t}^{-i}-G_{t}^{-i}$ for all $i\in[n]$.
		
		\item Take actions following $\widehat{\pi}^t$ and charge each agent $i$ a price $p_{i t}$ for $i\in[n]$.
		\item Determine whether we should update the dataset with the new trajectory.
	\end{itemize}
	
\end{enumerate}

Here $\widehat{\pi}^t$ is the learned policy aiming to estimate $\pi_*$, the function $F_{t}^{-i}$ can be viewed as an estimate of the value function under the fictitious policy, i.e., $V_{1}^{\pi_*^{-i}}(x_{1};R^{-i})$, and $G_{t}^{-i}$ estimates $V_{1}^{\widehat{\pi}^t}(x_{1};R^{-i})$ under the policy $\widehat{\pi}^t$. 
In particular, the hyperparameters $\zeta_2, \zeta_3$ control whether such an estimation by $F_{t}^{-i}$ and $G_{t}^{-i}$ is optimistic or pessimistic. Moreover, since $\widehat{\pi}^t$ estimates $\pi_*$, then $G_{t}^{-i}$ can further be considered as an approximation of $V_{1}^{\pi_*}(x_{1};R^{-i})$, which implies that the price $p_{i*}$ is estimated by $p_{it}$ according to its definition.  At a higher level, the algorithm decomposes learning the Markov VCG mechanism into two parts: 1) learning an efficient, social welfare-maximizing policy, and 2) estimating the suitable prices to charge the agents.

This paper focuses on a special case, i.e., Markov VCG with linear function approximation named \texttt{VCG-LinMDP}, as shown in Algorithm \ref{algorithm:LMVL}. The associated exploration phase is implemented in Algorithm \ref{algorithm:L1}, and the exploitation phase is implemented in Algorithms  \ref{algorithm:L3} and \ref{algorithm:L4}, where we adopt LSVI with linear function approximation. In particular, Algorithms \ref{algorithm:L3} and \ref{algorithm:L4} are the planning and policy evaluation subroutines respectively. As we can see from the overall framework, learning the price requires both planning to learn a fictitious policy (the function $F_t^{-1}$) and function evaluation on the learned policy $G_t^{-i}$ in order to estimate the price, necessitating the inclusion of both Algorithm~\ref{algorithm:L3} and Algorithm~\ref{algorithm:L4}.

\begin{algorithm}[t]
	\begin{algorithmic}[1]
		\caption{\texttt{VCG-LinMDP}}
		\label{algorithm:LMVL}
		\setstretch{1.1}
		\REQUIRE $\zeta_{1}\in\{\texttt{ETC}, \texttt{EWC}\}$, $\zeta_2,\zeta_3\in\{\texttt{OPT},\texttt{PES}\}$, 
		$\fR\in\{R, R^{-i}\}$, and $K$.
		
		{\color{blue} //Exploration Phase}
		\STATE Reward-free exploration for $K$ rounds via Algorithm \ref{algorithm:L1} and obtain $\cD=\{(x_h^k,a_h^k)\}_{h,k} \cup \{r_{i,h}^k(x_h^k,a_h^k)\}_{i,h,k}$. 
		
		{\color{blue} //Exploitation Phase}
		\FOR{$t=K+1,\cdots,T$}
		\STATE Update policy $\widehat{\pi}^t$ by the returned policy of Algorithm  \ref{algorithm:L3} with input  parameters $(R, \zeta_{1}, \texttt{OPT}, \cD)$.
		
		\STATE Update $F_{t}^{-i}$ by the returned value function of Algorithm \ref{algorithm:L3} with  parameters $(R^{-i}, \zeta_{1}, \zeta_2, \cD)$ for all $i\in [n]$.

		\STATE Update $G_{t}^{-i}$ by the returned value function of Algorithm  \ref{algorithm:L4} with  parameters $(R^{-i}, \zeta_{1}, \zeta_3, \cD, \widehat{\pi}^t)$ for all $i\in [n]$.

		
		\STATE Calculate the price $p_{i t}= F_{t}^{-i}-G_{t}^{-i}$ for all $i\in[n]$.
		\STATE Take action $a_h^t=\widehat{\pi}_h^t(x_h^t)$, receive rewards $\{r_{i,h}^t(x_h^t,a_h^t)\}_{i}$, and observe $x_{h+1}^t\sim\cP_h(\cdot|x_h^t,a_h^t)$ from $h=1$ to $H$.
		\STATE Charge each agent $i$ a price $p_{i t}$ for all $i\in[n]$.
		\IF{$\zeta_{1}=\texttt{EWC}$}
		\STATE $\cD\leftarrow\cD\cup\{(x_h^t,a_h^t)\}_{t,h} \cup\{r_{i,h}^t(x_h^t,a_h^t)\}_{i,h,t}$
		\ELSIF{$\zeta_{1}=\texttt{ETC}$} 
		\STATE Keep $\cD$ unchanged as collected in the exploration phase.
		\ENDIF
		\ENDFOR
	\end{algorithmic}
\end{algorithm}

As shown in Algorithm \ref{algorithm:LMVL}, there are multiple hyper-parameters. Specifically, $\zeta_1$ controls the overall learning strategy of \texttt{VCG-LinMDP} with options \texttt{ETC} and \texttt{EWC}. The option \texttt{ETC} indicates the \emph{explore-then-commit} strategy, where we exploit using only the data generated during the exploration phase. \texttt{EWC} indicates \emph{explore-while-commit} strategy, where we exploit using data generated during both the exploration phase and the exploitation phase. The options $\texttt{OPT}$ and $\texttt{PES}$ for the hyper-parameters $\zeta_2$ and $\zeta_3$ refer to optimistic and pessimistic exploitation approaches respectively, which control the trade-off between the seller's and the agents' utilities. Finally, for Algorithms~\ref{algorithm:L3} and~\ref{algorithm:L4}, the hyper-parameter $\fR$ controls whether the input reward function is $R$ or $R^{-i}$. In these algorithms, for abbreviation, we denote by $r_{i,h}^k(s_h^k,a_h^k):=r_{i,h}(s_h^k,a_h^k;\omega_h^k)$ the reward collected at step $h$ of time $k$ in the exploration phase and by 
$r_{i,h}^t(s_h^t,a_h^t):=r_{i,h}(s_h^t,a_h^t;\omega_h^t)$ a reward collected at step $h$ of time $t$ in the exploitation phase, where $\omega_h^k$ and $\omega_h^t$ represent the randomness in the reward observation.

\begin{remark}
	We remark that in our proposed algorithms in Section \ref{sec:alg}, with a slight abuse of notation, we do not require the reports of the rewards to be truthful when setting $\fR = R$ or $\fR = R^{-i}$. One can think of $R$ and $R^{-i}$ as input arguments if no specific discussion on truthfulness is involved. The rewards in the algorithms can be either truthful or untruthful. Whether the rewards are needed to be truthful or not will be explicitly highlighted in our theoretical results and the associated proofs.
\end{remark}


\begin{remark}\label{re:zeta}
	Intuitively, the hyperparameters $\zeta_2$ and $\zeta_3$ control whether the price favors the sellers or the buyers. There are two extreme cases for the setting of $(\zeta_{2},\zeta_{3})$, namely $(\texttt{PES},\texttt{OPT})$ and $(\texttt{OPT}, \texttt{PES})$. The configuration $(\zeta_{2},\zeta_{3})=(\texttt{PES},\texttt{OPT})$ that favors agents potentially leads to a low price $p_{it}$ and high agent utilities, resulting in a low agent regret and a high seller regret. The configuration $(\zeta_{2},\zeta_{3})=(\texttt{OPT},\texttt{PES})$ will favor the seller with a high price $p_{it}$ and a high seller utility, which results in a high agent regret and low seller regret.
	The prices charged under other configurations would fall somewhere between the aforementioned high and low prices. Consequently, the agents' and the seller's regrets would naturally be somewhere in the middle between the two representative cases, which we will expand in depth in our theoretical results. Such flexibility can be crucial in practice. For instance, the seller in the dynamic sponsored search auction or the dynamic PaaS setting discussed in Section~\ref{subsec:motivating_examples} favors a high price obtained by setting $\zeta_2 = \texttt{OPT}, \zeta_3 = \texttt{PES}$, while the social good provider in the dynamic public service setting may prefer a lower price when we set $\zeta_2 = \texttt{PES}, \zeta_3 = \texttt{OPT}$. 
\end{remark}

\paragraph{Least-Square Value Iteration.} With the overarching framework defined, we now introduce a key technique heavily used by our algorithm. For any function approximation class $\cF$, at the $t$-th episode, we have $t-1$ transition tuples, $\{(x_h^\tau, a_h^\tau, x_{h+1}^\tau)\}_{\tau\in[t-1]}$, and LSVI with function approximation \citep{2019Provably,yang2020function,2020Is} estimates the Q-function using $\tilde{f}_h^t$, obtained from the least-squares regression problem below.
\begin{align*}
	&\tilde{f}_h^t=\argmin_{f\in\cF}\sum_{\tau=1}^{t-1}\big[r_h^\tau(x_h^\tau, a_h^\tau)+V_h^t\big(x_h^{\tau})-f_h(x_h^\tau, a_h^\tau)\big)\big]^{2}+\mathrm{pen}(f),\\
	&f_h^t = \mathrm{truncate}\{\tilde{f}_h^t\},
\end{align*}
where $\mathrm{pen}(f)$ is some arbitrary regularizer, $r_h$ is some reward function, $\mathrm{truncate}\{\cdot\}$ is some truncation operator to guarantee that the approximation function is in a correct scale such that it does not violate the boundedness assumptions we place on the Q-function. For optimistic LSVI, we construct \emph{optimistic} Q-function as 
\begin{align*}
	Q_h^t = \mathrm{truncate}\{f_h^t + u_h^t\},    
\end{align*}
where we again truncate the estimated Q-function, and $u_h^t$ is an associated UCB bonus term constructed using the collected trajectories. Similarly, the \emph{pessimistic} Q-function is constructed as
\begin{align*}
	Q_h^t = \mathrm{truncate}\{f_h^t - u_h^t\}.    
\end{align*}
We update the value function by a greedy strategy as
\begin{align*}
	V_h^t(\cdot) = \argmax_{a\in\cA} Q_h^t(\cdot,a),
\end{align*}
for optimistic Q-function or pessimistic Q-function respectively. For the linear function approximation in our algorithm, according to our setting of linear MDPs, we let $f(\cdot,\cdot)=w^\top \phi(\cdot,\cdot)$ for any $f\in \cF$ and $\mathrm{pen}(f)$ be $\lambda \|w\|^2$ where $w$ is the parameter to learn. 

With the key ideas sketched out, we then proceed with fleshing out the proposed algorithms.

\subsection{Exploration Phase}

\begin{algorithm}[t]
	\begin{algorithmic}[1]
		\caption{\texttt{Exploration}} \label{algorithm:L1}
		\REQUIRE Failure probability $\delta > 0$, $K$, and $\lambda>0$
		\STATE $\beta= \hat{c}(n+R_{\max}) d H\sqrt{\log(36 n d H T/\delta)}$.
		\FOR{$k =1,2\cdots,K$}
		\STATE Set $V_{H+1}^k(\cdot)=0$.
		\FOR{$h =H,H-1\cdots,1$}
		\STATE $\Lambda_h^k=\sum_{\tau=1}^{k-1}\phi(x_h^{\tau},a_h^{\tau})\phi(x_h^{\tau},a_h^{\tau})^{\top}+\lambda I$.
		\STATE $u_h^k(\cdot,\cdot)= \Pi_{[0,H(n+R_{\max})]}\big[\beta[\phi(\cdot,\cdot)(\Lambda_h^k)^{-1}\phi(\cdot,\cdot)]^{1/2}\big]$.
		\STATE Define an exploration-driven reward function $l_h^k(\cdot,\cdot)= u_h^k(\cdot,\cdot)/H$.
		\STATE $w_h^k=(\Lambda_h^k)^{-1}\sum_{\tau=1}^{k-1}\phi(x_h^{\tau},a_h^{\tau})V_{h+1}^k(x_{h+1}^{\tau})$.
		\STATE $Q_h^k(\cdot,\cdot)=\min\{\Pi_{[0,H(n+R_{\max})]}[(w_h^k)^{\top}\phi(\cdot,\cdot)]+l_h^k(\cdot.\cdot)+u_h^k(\cdot,\cdot), H(n+R_{\max})\}$.
		\STATE $V_h^k(\cdot)=\max_{a\in\cA}Q_h^k(\cdot,a)$.
		\STATE $\pi_h^k(\cdot)=\argmax_{a\in\cA}Q_h^k(\cdot,a)$.
		\ENDFOR
		\STATE Take action $a_h^k=\pi_h^k(x_h^k)$, receive rewards $\{r_{i,h}^k(x_h^k,a_h^k)\}_i$, and observe the state transition $x_{h+1}^k\sim\cP_h(\cdot|x_h^k,a_h^k)$ from $h=1$ to $H$.
		\ENDFOR
		\STATE {\bfseries return} $\cD=\{(x_h^k,a_h^k)\}_{(h, k)\in [H]\times[K]} \cup \{r_{i,h}^k(x_h^k,a_h^k)\}_{(i,h,k)\in (\{0\}\cup[n])\times[H]\times[K]}$
	\end{algorithmic}
\end{algorithm}

Our first component is the exploration phase. Recall that $F_t^{-i}$ estimates the value function of the fictitious policy that maximizes welfare when agent $i$ is absent. Obtaining high-quality $F_t^{-i}$ for all $n$ agents then requires the algorithm to explore in the direction of multiple policies rather than only in a single policy's direction. This challenge necessitates reward-free reinforcement learning, where the learning algorithm seeks to explore the environment in the directions of all possible policies as opposed to only a single one.

Inspired by \citet{2020On}, we design a reward-free exploration algorithm as in Algorithm \ref{algorithm:L3}, incorporating the linear structure of the MDP. Specifically, to handle multiple reward functions from the seller and $n$ agents, we propose to explore the environment without using the observed rewards from it. Instead, we define an exploration-driven reward $l_h^k$ as a scaled bonus term $u_h^k$ to encourage exploration by further taking into account the uncertainty of estimating the environment. The bonus term computed in Line 6 quantifies the uncertainty of estimation with a linear function approximator.
Based on the exploration-driven rewards $l_h^k = u_h^k/H$ and the bonus term $u_h^k$ as well as the linear function approximation, we calculate an optimistic Q-function and perform the optimistic 
reward-free LSVI to generate the exploration policy. Note that in Algorithm \ref{algorithm:L1} and the subsequent Algorithms \ref{algorithm:L3} and \ref{algorithm:L4}, we define a truncation operator $\Pi_{[0,x]}[\cdot]:= \max \{\min\{\cdot,x\}, 0\}$. Distinguished from the standard LSVI introduced above, the reward-free LSVI only considers the value function as the regression target, i.e., we solve a least-square regression problem in the following form
\begin{align*}
	\argmin_{f\in\cF_{\mathrm{lin}}}\sum_{\tau=1}^{k-1}\big[V_h^k\big(x_h^{\tau})-f_h(x_h^\tau, a_h^\tau)\big)\big]^{2}+\mathrm{pen}(f),
\end{align*}
where $\cF_{\mathrm{lin}}$ is the linear function class. Then, we obtain the coefficient vector $w_h^k$ for linear function approximation. 

Moreover, for the optimistic Q-function in Line 9, we construct it by combining not only the linear approximation function and the exploration bonus $u_h^k$ but also the exploration-driven reward $l_h^k$.
Meanwhile, we collect the trajectories $\cD$ of visited state-action pairs and the corresponding reward feedbacks of $r_i, \forall i=0,1,\ldots, n$, for the subsequent exploitation phase in Algorithms \ref{algorithm:L3} and \ref{algorithm:L4}.


\subsection{Exploitation Phase}\label{sec:alg_L3_L4}

The exploitation phase is separated into two subroutines, namely \texttt{Planning} for planning in Algorithm \ref{algorithm:L3} and \texttt{PolicyEval} for policy evaluation in Algorithm \ref{algorithm:L4}. The two algorithms are general subroutines that are instantiated by the inputs.

\begin{table}[!t]
	\vspace{-0.39cm}
	
	\begin{minipage}[b]{0.495\linewidth}
		
		\begin{algorithm}[H]
				\begin{algorithmic}[1]
					\caption{\texttt{Exploitation}: \texttt{Planning}  } \label{algorithm:L3}
					\REQUIRE 
					$(\fR, \zeta, \zeta', \cD,t)$. 
					\STATE $V_{H+1}^t(\cdot;\fR)=0$.
					\FOR{$h =H,H-1\cdots,1$}
					\STATE $Q_h^t(\cdot,\cdot;\fR) = \texttt{Est-Q}(\fR, \zeta, \zeta', \cD, h,t)$ 
					
					\STATE\label{def:pi_t} $\pi_h^t(\cdot)=\argmax_{a\in\cA}Q_h^t(\cdot,a;\fR)$.
					
					\STATE $V_h^t(\cdot;\fR)= Q_h^t(\cdot,\pi_h^t(\cdot);\fR)$. 
					
					\ENDFOR
					
					\STATE {\bfseries return} $\{{\pi}^t_h\}_{h=1}^{H}$, $V_1^t(x_1;\fR)$
					
				\end{algorithmic}
		\end{algorithm}
		
	\end{minipage}
	\hfill
	\begin{minipage}[b] {0.495\linewidth}
		\begin{algorithm}[H]
			\setstretch{1.175}
				\begin{algorithmic}[1]
					\caption{\texttt{Exploitation}: \texttt{PolicyEval}} \label{algorithm:L4}
					\REQUIRE 
					$(\fR, \zeta, \zeta', \cD, t, \pi)$. 
					\STATE $V_{H+1}^t(\cdot;\fR)=0$.
					\FOR{$h =H,H-1\cdots,1$}
					\STATE $Q_h^t(\cdot,\cdot;\fR) = \texttt{Est-Q}(\fR, \zeta, \zeta', \cD, h,t)$
					\STATE\label{def:V} $V_h^t(\cdot;\fR)= Q_h^t(\cdot,\pi_h(\cdot);\fR)$.
					\ENDFOR
					
					\STATE {\bfseries return} $V_1^t(x_1;\fR)$
					
				\end{algorithmic}
		\end{algorithm}
		
	\end{minipage}
	\vspace{-1.2cm}
	\begin{minipage}[b]{1\linewidth}
		
		\begin{algorithm}[H]
				\begin{algorithmic}[1]
					\caption{\texttt{Est-Q}: One-Step Optimistic/Pessimistic Estimation of Q-Function} \label{algorithm:one-step}
					\REQUIRE 
					$(\fR, \zeta, \zeta', \cD, h,t)$. 
					\STATE Set $\alpha_h(\fR)$ as \eqref{eq:truc_R} and $\beta = \hat{c}(n+R_{\max}) d H\sqrt{\log(36 n d H T/\delta)}$.
					\STATE $P_{t}:=\begin{cases} 
						\{1,2,\cdots,K\} &\text{if $\zeta=\texttt{ETC}$}\\
						\{1,2,\cdots,t-1\} &\text{if $\zeta=\texttt{EWC}$}. 
					\end{cases}$
					\STATE $\Lambda_h^t=\sum_{\tau\in P_{t}}\phi(x_h^{\tau},a_h^{\tau})\phi(x_h^{\tau},a_h^{\tau})^{\top}+\lambda I$.
					\STATE $u_h^t(\cdot,\cdot)=\Pi_{[0,H(n+R_{\max})]}\big[\beta[\phi(\cdot,\cdot)(\Lambda_h^t)^{-1}\phi(\cdot,\cdot)]^{1/2}\big]$.
					\STATE $w_h^t=(\Lambda_h^t)^{-1}  \allowbreak \sum_{\tau\in P_{t}}\phi(x_h^{\tau},a_h^{\tau})[\fR_h^{\tau}(x_h^{\tau},a_h^{\tau})+V_{h+1}^t(x_{h+1}^{\tau};\fR)]$.
					\STATE $f_h^t(\cdot,\cdot) = \Pi_{[0,H(n+R_{\max})]}[(w_h^t)^{\top}\phi(\cdot,\cdot)]$.
					\STATE $Q_h^t(\cdot,\cdot;\fR)=\begin{cases} 
						\Pi_{[0, \alpha_h(\fR)]}[(f_h^t+ u_h^t)(\cdot,\cdot)] &\text{if $\zeta'=\texttt{OPT}$}\\
						\Pi_{[0, \alpha_h(\fR)]}[(f_h^t- u_h^t)(\cdot,\cdot)] &\text{if $\zeta'=\texttt{PES}$}. 
					\end{cases}$
					\STATE {\bfseries return} $Q_h^t(\cdot,\cdot;\fR)$
					
				\end{algorithmic}
		\end{algorithm}
		
	\end{minipage}
	\vspace{0.2cm}
\end{table}

The \texttt{Planning} subroutine in Algorithm \ref{algorithm:L3} is an optimistic or pessimistic LSVI with linear function approximation, which generates a greedy policy and its associated value function. Different from Algorithm \ref{algorithm:L3}, \texttt{PolicyEval} subroutine in Algorithm \ref{algorithm:L4} only evaluates any input policy $\pi$ by computing the value function under $\pi$ with linear function approximation. Both of the two algorithms will call Algorithm \ref{algorithm:one-step}, which is an optimistic or pessimistic estimation of the Q-function for a reward function $\fR \in\{R, R^{-i}\}$ at step $h$. Algorithm \ref{algorithm:one-step} can be viewed as an instantiation of LSVI in Section \ref{subsec:algframe} for linear function approximation. 
In Line 4 of Algorithm \ref{algorithm:one-step}, we compute a bonus $u_h^t$ to quantify the uncertainty in estimation. In Lines 5 and 6, we obtain the coefficient vector $w_h^t$ for linear function approximation and the approximator $f_h^t$. Line 7 yields optimistic and pessimistic Q-functions respectively determined by $\zeta'=\texttt{OPT}$ or $\texttt{PES}$. 

The argument $\zeta$ in these algorithms determines the composition of the data index set $P_{t}$ in Line 2 and thus indicates whether we will use the original exploration dataset or the updated dataset to construct the bonus term $u_h^t$ and the linear function approximator $f_h^t$. More formally, only the data collected in the exploration phase of Algorithm \ref{algorithm:LMVL} will be used if we let $\zeta=\texttt{ETC}$, and the data generated in both exploration and exploitation phases is used when we let $\zeta=\texttt{EWC}$.

The function $\alpha_h(\fR)$ in these algorithms controls the truncation constant, which equals the supremum of the corresponding reward function. Precisely, we have 
\begin{align}
	\begin{aligned}\label{eq:truc_R}
		\alpha_h(\fR):=\begin{cases} 
			(n+R_{\max}) (H-h+1) &\text{if $\fR = R$}\\
			(n-1+R_{\max}) (H-h+1) &\text{if $\fR = R^{-i}$ for any $i\in[n]$ }. 
		\end{cases}
	\end{aligned}
\end{align}

	Note that Algorithm \ref{algorithm:L3} and Algorithm \ref{algorithm:L4} are two generic subroutines for the exploitation phase, whose concrete implementation is contingent on the input arguments. For brevity, we denote all the value functions and Q-functions in Algorithm \ref{algorithm:L3} and Algorithm \ref{algorithm:L4} calculated in step $t$ by $V_h^t(\cdot;\cdot)$ and $Q_h^t(\cdot,\cdot;\cdot)$ respectively. Specifically, in the rest of this work, we let $\{\widehat{V}_h^{t,*}(\cdot;\fR), \hat{Q}_h^{t,*}(\cdot,\cdot;\fR)\}$ and $\{\check{V}_h^{t,*}(\cdot;\fR), \check{Q}_h^{t,*}(\cdot,\cdot;\fR)\}$ be the realization of $V_h^t(\cdot;\cdot)$ and $Q_h^t(\cdot,\cdot;\cdot)$ generated by Algorithm \ref{algorithm:L3} for $\zeta' = \texttt{OPT}$ and $\zeta' = \texttt{PES}$ respectively, with different options for $\fR$; and let $\{\widehat{V}_h^{t,\pi}(\cdot;\fR), \hat{Q}_h^{t,\pi}(\cdot,\cdot;\fR)\}$ and $\{\check{V}_h^{t,\pi}(\cdot;\fR), \check{Q}_h^{t,\pi}(\cdot,\cdot;\fR)\}$ be associated with $\zeta'=\texttt{OPT}$ and $\zeta'=\texttt{PES}$ respectively, which are generated by Algorithm \ref{algorithm:L4} with arbitrary input policy $\pi$. In the sequel, in Algorithm \ref{algorithm:LMVL}, we have 
	\begin{equation*}
		F_{t}^{-i}=\begin{cases}\widehat{V}_{1}^{t,*}\big(x_{1};R^{-i}\big) &\text{if $\zeta_2=\texttt{OPT}$}\\
			\check{V}_{1}^{t,*}\big(x_{1};R^{-i}\big) &\text{if $\zeta_2=\texttt{PES}$},
		\end{cases}   
		\qquad 
		G_{t}^{-i}=\begin{cases}\widehat{V}_{1}^{t,\widehat{\pi}^t}\big(x_{1};R^{-i}\big) &\text{if $\zeta_3=\texttt{OPT}$}\\
			\check{V}_{1}^{t,\widehat{\pi}^t}\big(x_{1};R^{-i}\big) &\text{if $\zeta_3=\texttt{PES}$}.
		\end{cases}   
	\end{equation*}
	These functions then in turn estimate the price that is to be charged to the agents. The exact formulation can be found in Algorithm~\ref{algorithm:LMVL}. 
	
	Our proposed algorithms have the potential of being extended to other nonlinear function approximations following the LSVI steps in Section \ref{subsec:algframe}, such as the kernel function approximation and neural function approximation built on the neural tangent kernel theory \citep{jacot2018neural}. This generalization is facilitated by exploring the inherent structure of specific function classes to construct bonus terms and optimistic/pessimistic Q-functions using techniques proposed in \citet{zhou2020neural,yang2020provably,2021On}. Then, one can replace the function approximation steps in Algorithms \ref{algorithm:L1} and \ref{algorithm:one-step} with the ones tailored for these approximators to apply nonlinear function approximation. Such a direction of research warrants further studies in the future.

	\begin{remark}
		We emphasize that \texttt{VCG-LinMDP} (Algorithm~\ref{algorithm:LMVL}) is not a direct extension of reward-free RL algorithms with function approximation (e.g., \citet{jin2020reward, 2020On, 2021On}) which focus only on estimating the optimal value functions corresponding to different reward functions. Learning the dynamic mechanism requires achieving multiple desiderata as introduced in Section \ref{sec:prob} and minimizing the corresponding regrets, which introduces additional challenges with decomposing the regret terms not encountered in prior literature. In particular, we adopt reward-free exploration to address a specific challenge encountered when learning the dynamic VCG mechanism, namely, the need to learn the fictitious policy, i.e., the optimal policy in the absence of each agent $i$, yet reward-free exploration itself cannot ensure that the resulting mechanism is truthful or individually rational. 
		Particularly, to show that the final policy output by the exploitation phase enjoys the desired desiderata requires the particular structure of the VCG mechanism, which we exploit in our proofs.
		Besides, the exploitation phase (Algorithm~\ref{algorithm:L3} and Algorithm~\ref{algorithm:L4}) allows for optimism and pessimism in an online setting, inducing different price estimation strategies as discussed above. Moreover, Algorithm~\ref{algorithm:L1} differs from standard reward-free RL algorithms by recording the received rewards of different agents during exploration and utilizing these collected rewards to learn the welfare-maximizing policy and the agents' prices.
	\end{remark}

\section{Main Results}

In this section, we discuss our main theoretical results. We first state the results corresponding to the three desiderata in mechanism design when $\zeta_{1}=\texttt{ETC}, \texttt{EWC}$ respectively. Then we present the lower bound of our problem. In our algorithms and theoretical results, $\hat{c}$ is a universal absolute constant. We begin with the results for when $\zeta_{1}=\texttt{ETC}$, i.e., the proposed algorithms adopt the \emph{explore-then-commit} strategy, where the exploitation phase uses only the data generated during the exploration phase.

\begin{theorem}\label{theorem:ETC}
	When $\zeta_{1}=\texttt{ETC}$, setting $K=dH^{4/3}\iota^{1/3}T^{2/3}$ where $\iota:=\log(36 n d H T/\delta)$ for any $\delta\in (0,1]$, defining $n_R:=n+R_{\max}$, with probability at least $1-\delta$, for all $T > K$, the following results hold after executing Algorithm \ref{algorithm:LMVL} for $T$ rounds: 
	\begin{enumerate}[leftmargin=*,itemsep=0pt,parsep=1pt]
		\item Assuming all agents report truthfully, for all $\zeta_{2},\zeta_{3}\in\{\texttt{OPT},\texttt{PES}\}$, the welfare regret satisfies
		\begin{equation*}
			\mathrm{Reg}_{T}^{W}\leq(1+2\hat{c})n_Rd H^{7/3}\iota^{1/3}T^{2/3},
		\end{equation*}
		which indicates that the learned mechanism is $(1+2\hat{c})n_Rd H^{7/3}\iota^{1/3}T^{2/3}$-approximately efficient.
		\item Assuming all agents report truthfully, the regret of agent $i$ satisfies
		\begin{align*}
			\mathrm{Reg}_{i T}\leq \begin{cases}
				(1 + 2\hat{c}n_R)dH^{7/3}\iota^{1/3}T^{2/3} &\textrm{ if $(\zeta_{2},\zeta_{3})=(\texttt{PES},\texttt{OPT})$}\\
				(1 + 6\hat{c}n_R)dH^{7/3}\iota^{1/3}T^{2/3} &\textrm{ if $(\zeta_{2},\zeta_{3})=(\texttt{OPT},\texttt{PES})$}.
			\end{cases}
		\end{align*}
		
		\item Assuming all agents report truthfully, the regret of the seller satisfies
		\begin{equation*}
			\mathrm{Reg}_{0 T}\leq\begin{cases}
				(1+4\hat{c}n)n_Rd H^{7/3}\iota^{1/3} T^{2/3}& \text{if $(\zeta_{2},\zeta_{3})=(\texttt{PES},\texttt{OPT})$}\\
				n_Rd H^{7/3}\iota^{1/3} T^{2/3}& \text{if $(\zeta_{2},\zeta_{3})=(\texttt{OPT},\texttt{PES})$}.
			\end{cases}
		\end{equation*}
		\item 
		The learned mechanism is $6\hat{c}n_Rd H^{7/3}\iota^{1/3}T^{2/3}$-approximately individually rational.
		
		\item 
		The learned mechanism is $\big(1+ 4\hat{c}n_R\big)dH^{7/3}\iota^{1/3}T^{2/3}$-approximately truthful.
		
	\end{enumerate}
\end{theorem}
As the learning objective of our algorithm is to minimize the welfare regret together with the agent and seller regrets, we choose $K=dH^{4/3}\iota^{1/3}T^{2/3}$ that can lead to a small upper bound of $\max\{n\textrm{Reg}^W_T, \textrm{Reg}_T^\sharp, \textrm{Reg}_{0T}\}$, which is $\cO \big(n(n+R_{\max})d H^{7/3}\iota^{1/3}T^{2/3}\big)$. Here we ignore constant factors and emphasize $K$'s dependence on $d$, $H$, $\iota$, and $T$.
As discussed in Remark \ref{re:zeta}, we use $\zeta_{2}$ and $\zeta_{3}$ to control the charged price and the seller and agent utilities, which further affect the achieved regrets. When $(\zeta_{2},\zeta_{3})=(\texttt{OPT},\texttt{PES})$, the charged price will be large and favor the seller, which thus leads to a relatively low seller regret $(n+R_{\max})d H^{7/3}\iota^{1/3} T^{2/3}$ and a high agent regret $(1 + 6\hat{c}(n+R_{\max}))dH^{7/3}\iota^{1/3}T^{2/3}$. When $\zeta_{2}=\texttt{PES}$ and $\zeta_{3}=\texttt{OPT}$, there will be a lower price favoring the agent, such that the seller regret increases to $(1+4\hat{c}n)n_Rd H^{7/3}\iota^{1/3} T^{2/3}$ and agent $i$'s regret decreases to $(1 + 2\hat{c}(n+R_{\max}))dH^{7/3}\iota^{1/3}T^{2/3}$. The seller and agent regrets incurred by other options of $(\zeta_{2},\zeta_{3})$ will lie between the above regret bounds under such two settings. 
Since the welfare does not depend on the price as shown in Equation \eqref{equa:regret}, the choices of $(\zeta_{2},\zeta_{3})$ thus have no impact on the welfare regret.

We further present the results for  $\zeta_{1}=\texttt{EWC}$, i.e., the algorithm adopts the \emph{explore-while-commit} strategy, where the exploitation phase uses data collected during both the exploration and exploitation phases.

\begin{theorem}\label{theorem:OPT}
	When $\zeta_{1}=\texttt{EWC}$, setting $K=dH^{4/3}\iota^{1/3}T^{2/3}$ where $\iota:=\log(36 n d H T/\delta)$ for any $\delta\in (0,1]$, defining $n_R:=n+R_{\max}$, with probability at least $1-\delta$, for all $T > K$, the following results hold after executing Algorithm \ref{algorithm:LMVL} for $T$ rounds:
	\begin{enumerate}[leftmargin=*,itemsep=0pt,parsep=1pt]
		\item Assuming all agents report truthfully, for all $\zeta_{2},\zeta_{3}\in\{\texttt{OPT},\texttt{PES}\}$, the welfare regret satisfies
		\begin{equation*}
			\mathrm{Reg}_{T}^{W}\leq n_R d H^{7/3} \iota^{1/3} T^{2/3}+ 6\hat{c}n_R d^{3/2} H^2 \iota T^{1/2},
		\end{equation*}
		which indicates that the learned mechanism is $ (n_R d H^{7/3} \iota^{1/3} T^{2/3}+ 6\hat{c}n_R d^{3/2} H^2 \iota T^{1/2})$-approximately efficient.
		\item Assuming all agents report truthfully, the regret of agent $i$ satisfies
		\begin{equation*}
			\mathrm{Reg}_{i T}\leq \begin{cases}
				d H^{7/3} \iota^{1/3} T^{2/3}+ 6\hat{c}n_R d^{3/2} H^2 \iota T^{1/2}& \text{if $(\zeta_{2},\zeta_{3})=(\texttt{PES},\texttt{OPT})$}\\
				(1+4\hat{c}n_R)dH^{7/3} \iota^{1/3} T^{2/3}+ 6\hat{c}n_R d^{3/2} H^2 \iota T^{1/2}& \text{if $(\zeta_{2},\zeta_{3})=(\texttt{OPT},\texttt{PES})$},
			\end{cases}
		\end{equation*}
		\item Assuming all agents report truthfully, the regret of the seller satisfies
		\begin{equation*}
			\mathrm{Reg}_{0 T}\leq\begin{cases}
				(1+4\hat{c}n)n_R d H^{7/3}\iota^{1/3} T^{2/3}& \text{if $(\zeta_{2},\zeta_{3})=(\texttt{PES},\texttt{OPT})$}\\
				n_Rd H^{7/3}\iota^{1/3} T^{2/3}& \text{if $(\zeta_{2},\zeta_{3})=(\texttt{OPT},\texttt{PES})$}.
			\end{cases}
		\end{equation*}
		
		\item 
		The learned mechanism is $6\hat{c}n_R d H^{7/3}\iota^{1/3}T^{2/3}$-approximately individually rational.
		
		\item 
		The learned mechanism is $(1 +8\hat{c}n_R)d H^{7/3}\iota^{1/3}T^{2/3}$-approximately truthful.
	\end{enumerate}
\end{theorem}
Similar to Theorem \ref{theorem:ETC}, we choose a proper $K$ in Theorem \ref{theorem:OPT} that can lead to a small upper bound of $\max\{n\textrm{Reg}^W_T, \textrm{Reg}_T^\sharp, \textrm{Reg}_{0T}\}$ in terms of $d$, $H$, $\iota$, and $T$, which is $\cO \big(n(n+R_{\max})d H^{7/3}\iota^{1/3}T^{2/3}\big)$. Theorem \ref{theorem:OPT} also gives the seller and agent regret bounds for the two settings $(\zeta_{2},\zeta_{3})=(\texttt{PES}, \texttt{OPT})$ and $(\zeta_{2},\zeta_{3})=(\texttt{OPT},\texttt{PES})$, showing that the seller and agent regret bounds vary between the ones under these two extreme cases according to Remark \ref{re:zeta}. Note that when the problem reduces to the tabular setting, we have $d = |\cS| |\cA|$ in Theorems \ref{theorem:ETC} and \ref{theorem:OPT}. When $d\leq|\cS||\cA|$, we obtain a better rate than that under the tabular setting.

\begin{table*}[t] 
	\renewcommand{\arraystretch}{1.5}
	\centering
	\begin{tabular}{ | >{\centering\arraybackslash}m{2.2cm} | >{\centering\arraybackslash}m{4.6cm} | >{\centering\arraybackslash}m{6.8cm} | } 
		\hline
		Metrics & Theorem 4.1 ($\zeta_1=\texttt{ETC}$)  & Theorem 4.2  ($\zeta_1=\texttt{EWC}$)   \\ \hline
		$\mathrm{Reg}_T^W$ &  $(1+2\hat{c})n_Rd H^{\frac{7}{3}}\iota^{\frac{1}{3}}T^{\frac{2}{3}}$ & $n_R d H^{\frac{7}{3}} \iota^{\frac{1}{3}} T^{\frac{2}{3}} + {\color{red} 6\hat{c}n_R d^{\frac{3}{2}} H^2 \iota T^{\frac{1}{2}}}$    \\\hline
		\multirow{2}{*}[-0.05cm]{$\mathrm{Reg}_{i T}$}  & $(1 + 2\hat{c}n_R)d H^{\frac{7}{3}}\iota^{\frac{1}{3}}T^{\frac{2}{3}}$ $\blacklozenge$ & $ d H^{\frac{7}{3}} \iota^{\frac{1}{3}} T^{\frac{2}{3}}+ {\color{red}6\hat{c}n_R d^{\frac{3}{2}} H^2 \iota T^{\frac{1}{2}}}$ $\blacklozenge$ \\
		& $(1 + 6\hat{c}n_R)d H^{\frac{7}{3}}\iota^{\frac{1}{3}}T^{\frac{2}{3}}$ $\blacktriangle$ & $(1+4\hat{c}n_R)dH^{\frac{7}{3}} \iota^{\frac{1}{3}} T^{\frac{2}{3}}+ {\color{red}6\hat{c}n_R d^{\frac{3}{2}} H^2 \iota T^{\frac{1}{2}}}$ $\blacktriangle$  \\ \hline
		\multirow{2}{*}[-0.05cm]{$\mathrm{Reg}_{0 T}$}  & $(1+4\hat{c}n)n_R d H^{\frac{7}{3}}\iota^{\frac{1}{3}}T^{\frac{2}{3}}$ $\blacklozenge$ & $(1+4\hat{c}n)n_R d H^{\frac{7}{3}}\iota^{\frac{1}{3}}T^{\frac{2}{3}}$ $\blacklozenge$  \\
		& $n_Rd H^{\frac{7}{3}}\iota^{\frac{1}{3}}T^{\frac{2}{3}}$ $\blacktriangle$ & $n_Rd H^{\frac{7}{3}}\iota^{\frac{1}{3}}T^{\frac{2}{3}}$ $\blacktriangle$   \\ \hline
		Approx. I.R. & $6\hat{c}n_Rd H^{\frac{7}{3}}\iota^{\frac{1}{3}}T^{\frac{2}{3}}$ & $6\hat{c}n_Rd H^{\frac{7}{3}}\iota^{\frac{1}{3}}T^{\frac{2}{3}}$    \\ \hline
		Approx. Tr.  & $(1+ 4\hat{c}n_R)d H^{\frac{7}{3}}\iota^{\frac{1}{3}}T^{\frac{2}{3}}$ & $(1+ 8\hat{c}n_R)d H^{\frac{7}{3}}\iota^{\frac{1}{3}}T^{\frac{2}{3}}$   \\ \hline
	\end{tabular}
	\caption{Comparison of Theorem \ref{theorem:ETC} and Theorem \ref{theorem:OPT}. Here ``Approx. I.R.'' and ``Approx. Tr.'' are the abbreviations of ``Approximate Individual Rationality'' and ``Approximate Truthfulness''. The results in Theorem \ref{theorem:ETC} and Theorem \ref{theorem:OPT} hold with probability at least $1-\delta$ respectively for any $\delta\in (0,1]$. We let $n_R:=n+R_{\max}$ and $\iota:=\log(36 n d H T/\delta)$. We use 
		$\blacklozenge$ to represent the configuration $(\zeta_{2},\zeta_{3})=(\texttt{PES},\texttt{OPT})$ and $\blacktriangle$ to represent $(\zeta_{2},\zeta_{3})=(\texttt{OPT},\texttt{PES})$. We further highlight the improvements in the welfare and agent regrets in {\color{red}red}.}
	\label{tab:comparison}
	\vspace{-0.3cm}
\end{table*}

\paragraph{Further Discussion on Theorem \ref{theorem:ETC} and Theorem \ref{theorem:OPT}.} We summarize the results from the two theorems in Table \ref{tab:comparison}. As shown in our proof sketch in Section \ref{sec:proof-sk}, we obtain that $\mathrm{Reg}_{T}^{W}\leq (n+R_{\max})HK+2\hat{c}(n+R_{\max})\sqrt{d^{3}H^{6}\iota/K}(T-K)$ when $\zeta_1 = \texttt{ETC}$ in Theorem~\ref{theorem:ETC} and $\mathrm{Reg}_{T}^{W}\leq (n+R_{\max})HK+6\hat{c}(n+R_{\max})\sqrt{d^3 H^4  (T-K)\iota^2}$ when $\zeta_1 = \texttt{EWC}$ in Theorem~\ref{theorem:OPT}, where both bounds share the same term $H(n+R_{\max})K$ that results from the exploration phase. To compare the welfare regrets achieved in both theorems fairly, we in fact need the rounds of exploration $K$ to be the same, although a straightforward idea might be setting $K$ differently as $K=\tilde{\cO}(T^{2/3})$ for \texttt{ETC} and $K=0$ for \texttt{EWC} to minimize the two bounds respectively. However, we note that the setting $K=0$ for Theorem \ref{theorem:OPT} will lead to unboundedness in the seller and agent regrets as well as the individual rationality and truthfulness according to our proof sketch in Section \ref{sec:proof42}. Fortunately, our choice of $K$ depends on the metric of $\max\{n\textrm{Reg}^W_T, \textrm{Reg}_T^\sharp, \textrm{Reg}_{0T}\}$, where $\textrm{Reg}_T^\sharp:=\sum_{i=1}^n \textrm{Reg}_{iT}$, by taking all three types of regrets into consideration, which can naturally resolve the aforementioned issue. Moreover, under this metric, the choices of $K$ for both theorems all have the same dependence on $d$, $H$, $\iota$, and $T$ as justified in our proof sketch, and thus we set the same value of $K$ directly as $dH^{4/3}\iota^{1/3}T^{2/3}$.

From Table~\ref{tab:comparison}, it is seen that the same setting of $K$ leads to the same individual rationality guarantee and nearly the same truthfulness guarantee that differs only by an absolute constant scaling factor.
Again referencing~\citet{epasto2018incentive}, it is even challenging for real-world agents to capitalize on a slightly larger constant factor in the approximate truthfulness guarantees. Therefore, although a slight increase exists in the truthfulness guarantee for $\zeta_{1} = \texttt{EWC}$ compared to $\zeta_{1} = \texttt{ETC}$, the current setting of $K$ is justifiable and enables a fair comparison of regrets. Then, as shown in Table \ref{tab:comparison}, with $K=dH^{4/3}\iota^{1/3}T^{2/3}$, the algorithm under $\zeta_{1} = \texttt{ETC}$ can improve a part of the welfare regret from $\tilde{\cO}(T^{2/3})$ to $\tilde{\cO}(T^{1/2})$. This improvement results from the use of all the data gathered up to time step $t$ in the $\texttt{EWC}$ setting rather than the data collected only in the exploration phase in the $\texttt{ETC}$ setting. From Table \ref{tab:comparison}, we can also observe a similar improvement in the agent regret bound. The regret improvement also verifies the importance of using the explore-while-commit (\texttt{EWC}) strategy in the learning algorithm.

Furthermore, we remark that our regret guarantees rely on the assumption that agents report truthfully. Nevertheless, recalling our earlier discussion on our definition of $\delta$-approximate efficiency, we note that it is in general difficult to obtain regret bounds without assuming truthfulness, and thus obtaining performance guarantees under the truthfulness assumption is reasonable according to existing works~\citep{nazerzadeh2008dynamic,epasto2018incentive,2020Mechanism}.

Both Theorem \ref{theorem:ETC} and Theorem \ref{theorem:OPT} implies $\max\{n\textrm{Reg}^W_T, \textrm{Reg}_T^\sharp, \textrm{Reg}_{0T}\} = \cO \big(n(n+R_{\max})d H^{7/3}\iota^{1/3}T^{2/3}\big)$. We remark that the $\tilde{\cO}(T^{2/3})$ regret is necessary. If we were to focus only on welfare regret, then it is well-known that the lower bound would be $\Omega(\sqrt{T})$. However, the key challenge of learning the proposed Markov VCG mechanism lies in the interplay between the three kinds of regrets studied. Consider the extreme case where we set $K = 0$ in Theorem~\ref{theorem:OPT}. According to our proof sketch in Section \ref{sec:proof42}, while the welfare regret upper bound in Equation \eqref{eq2:w regret} improves to $\Tilde{\cO}(\sqrt{T})$, we can no longer control the agent nor the seller regrets in Equations \eqref{eq2:a regret} and \eqref{eq2:s regret}.

At last, we justify that the $\tilde{\cO}(T^{2/3})$ bound is tight by providing the lower bound of $\max\big\{ n \mathrm{Reg}_T^W,\mathrm{Reg}^{\sharp}_T,\mathrm{Reg}_{0T}\big\}$ when all agents are truthful.
Let $\Theta$ and $\mathsf{Alg}$ be the class of problems and the class of algorithms for this setting respectively, and we obtain the lower bound as follows:
\begin{theorem}
	\label{theorem:lowerbound}
	Let $\mathrm{Reg}_T^W,\mathrm{Reg}^{\sharp}_T,\mathrm{Reg}_{0T}$ be as defined in~\eqref{equa:regret}.
	Let all agents be truthful. Defining $n_R:=n+R_{\max}$,
	we have: 
	\begin{align*}
		&\inf_{\mathsf{Alg}} \sup_\Theta \;\EE\left[\max\big\{ n \mathrm{Reg}_T^W,\mathrm{Reg}^{\sharp}_T,\mathrm{Reg}_{0T}\big\}\right] \geq\Omega\left(n^{4/3} H^{2/3} T^{2/3} +n n_Rd\sqrt{HT}\right),
	\end{align*}
	for $T\geq \max\{16(n-1)/(H-1),64(d-3)^2H\},H\geq2,d\geq4$ and $n\geq3$.
\end{theorem}

At a high level, Theorem~\ref{theorem:lowerbound} indicates that the $\tilde{\cO}\big(T^{2/3}\big)$ upper bound of $\max\big\{ n \mathrm{Reg}_T^W,\allowbreak \mathrm{Reg}^{\sharp}_T,\mathrm{Reg}_{0T}\big\}$ obtained by the three regrets in Theorem \ref{theorem:ETC} and Theorem \ref{theorem:OPT} are tight. In other words, unlike typical single-agent RL, it is impossible to obtain $\Tilde{\cO}(\sqrt{T})$ regret when learning the Markov VCG mechanism. 
The intuition behind the hard case used for the lower bound is that we need to accurately learn the VCG prices to achieve a low regret. Setting the VCG prices too high harms the agents' utilities, whereas setting them too low harms the seller's.  Learning the VCG prices requires learning the welfare-maximizing policy when agent $i$ is absent, $\pi^{-i}_*$. Combined with our need to estimate the welfare-maximizing policy, any suitable learning algorithm needs to reduce the estimation error of the value functions for all policies. Our proposed algorithm resolves this challenge by reward-free exploration, and the procedure is crucial for efficiently learning the Markov VCG mechanism. There is still a gap between the upper and lower bounds in terms of the multiplicative factors $n$, $d$, and $H$, and we leave the derivation of exactly matching upper and lower bounds as an open question for future work.

Our work features several prominent contributions to the existing literature in mechanism design learning and online learning of linear MDPs. As shown in Theorem~\ref{theorem:ETC} and Theorem~\ref{theorem:OPT}, our work proposes the first algorithm capable of learning a dynamic mechanism with no prior knowledge. In particular, we further show that the mechanism learned by Algorithm~\ref{algorithm:LMVL} simultaneously satisfies approximate efficiency, approximate individual rationality, and approximate truthfulness. As we will demonstrate in the sequel, the satisfaction of the approximate versions of the three mechanism design desiderata is demonstrated through novel decomposition approaches. 
Moreover, Theorem~\ref{theorem:lowerbound} demonstrates that our achieved results are minimax optimal up to problem-dependent constants.

\section{Proof Sketch} \label{sec:proof-sk}
\label{sec:theoretical_analysis}
In this section, we outline the analysis of our theorems. The formal proof is deferred to Appendix  \ref{sec:proof-start} - \ref{sec:proof-end}. 
For a concise presentation, in the proof, we let $V_1^*(x_1;r):=\max_\pi V_1^\pi(x_1;r)$ for any reward function $r$. 
We further provide a table of notation in Appendix \ref{sec:tab_notation} summarizing all notations used here.

\subsection{Proof Sketch of Theorem \ref{theorem:ETC}}
\label{sec:5.1}
We assume that all agents report their rewards truthfully in the proof of the upper bounds of the welfare regret, the agent regret, and the seller regret. Since we use the explore-then-commit algorithm when $\zeta_{1}=\texttt{ETC}$, we decompose all the regrets into two components: the regret incurred in the exploration phase and the regret incurred in the exploitation phase. Additionally, for each of these regrets, we first show its dependence on both the rounds of exploration $K$ and the total rounds $T$. Then we determine $K$ that can lead to a tight upper bound of $\max\{ n \mathrm{Reg}_T^W,\mathrm{Reg}^{\sharp}_T,\mathrm{Reg}_{0T}\}$ in terms of $n,d,H,\iota,$ and $T$.

\paragraph{Welfare Regret.} 
We first decompose the welfare regret into two parts as follows: 
\begin{equation}\label{equa:welfare_regret_explore_exploit}
	\mathrm{Reg}_{T}^{W}=\textstyle\sum_{t=1}^{K}\mathrm{reg}_{t}^{W}+\textstyle\sum_{t=K+1}^{T}\mathrm{reg}_{t}^{W},
\end{equation}
where $\mathrm{reg}_{t}^{W}:=V_1^{\pi_*}(x_{1};R)-V_{1}^{\widehat{\pi}_{t}}(x_{1};R)$ is the instantaneous welfare regret. Here $\sum_{t=1}^{K}\mathrm{reg}_{t}^{W}$ is the welfare regret in the exploration phase and $\sum_{t=K+1}^{T}\mathrm{reg}_{t}^{W}$ is for the exploitation phase. For the regret incurred in the exploration phase in Equation \eqref{equa:welfare_regret_explore_exploit}, we bound the instantaneous regret $\mathrm{reg}_{t}^{W}$ at each time step by $H(n+R_{\max})$, which is the maximum of the instantaneous regret at each round. 
For the exploitation welfare regret in Equation \eqref{equa:welfare_regret_explore_exploit}, we can bound its instantaneous welfare regret $\mathrm{reg}_{t}^{W}$ by $2\hat{c}(n+R_{\max})\sqrt{d^{3}H^{6}\iota/K}$ with high probability, whose proof is inspired by the regret proof for learning linear MDPs, as the prices cancel out when calculating social welfare. Therefore, with high probability, the following welfare regret bound holds
\begin{equation} \label{eq:wreg-k-ps}
	\mathrm{Reg}_{T}^{W}\leq H(n+R_{\max})K+2\hat{c}(n+R_{\max})\sqrt{d^{3}H^{6}\iota/K}(T-K),
\end{equation} 
where the rounds of the exploration phase $K$ will be determined later.

\paragraph{Agent Regret.} 
We have the following regret decomposition in terms of the exploration phase and exploitation phase as follows,
\begin{equation}\label{equa:agent_regret_explore_exploit}
	\mathrm{Reg}_{i T}=\textstyle\sum_{t=1}^{K}\mathrm{reg}_{i t}+\textstyle\sum_{t=K+1}^{T}\mathrm{reg}_{i t},
\end{equation}
where $\mathrm{reg}_{i t}:=u_{i *}-u_{i t}$ is the instantaneous regret of agent $i$.
As shown in Algorithm \ref{algorithm:LMVL}, we do not charge the agents in the exploration phase. Thus, the instantaneous regret of agent $i$ in the exploration phase can be upper bounded as 
\begin{equation*}
	\textrm{reg}_{it} \leq u_{i *} - \min_{\pi}V_1^{\pi}(x_1; r_i) \leq u_{i *} = V_1^{\pi_*}(x_1; r_i) - p_{i *}\leq V_1^{\pi_*}(x_1; r_i)\leq H, \quad 1\leq t \leq K.
\end{equation*}
For the terms in the second summation in Equation \eqref{equa:agent_regret_explore_exploit}, i.e., the instantaneous regret of agent $i$ incurred in the exploitation phase, we first decompose it to several simple terms as follows, 
\begin{equation}\label{equation:decomp_agent_regret_EWC}
	\mathrm{reg}_{i t}=\underbrace{\big[V_{1}^{\pi_*}\big(x_{1};R\big)-V_{1}^{\widehat{\pi}^{t}}\big(x_{1};R\big)\big]}_{\rm(i.1)}+\underbrace{\big[F_{t}^{-i}-V_{1}^{\pi_*^{-i}}\big(x_{1};R^{-i}\big)\big]}_{\rm(i.2)}+\underbrace{\big[V_{1}^{\widehat{\pi}^{t}}\big(x_{1};R^{-i}\big)-G_{t}^{-i}\big]}_{\rm(i.3)},
\end{equation}
where $\rm(i.1)$ is the suboptimality of $\widehat{\pi}^t$, $\rm(i.2)$ is the estimation error of $V_{1}^{\pi_*^{-i}}\big(x_{1};R^{-i}\big)$ by $F_t^{-i}$, and $\rm(i.3)$ is the policy evaluation error.
To satisfy the desiderata of the mechanism design in Lemma \ref{lemma:Markov VCG mechanism}, we set $F$-function as 
the optimistic (when $\zeta_{2}=\texttt{OPT}$) or pessimistic (when $\zeta_{2}=\texttt{PES}$) estimate of $V_{1}^{\pi_*^{-i}}\big(x_{1};R^{-i}\big)$, while the $G$-function 
is the estimate of $V_{1}^{\widehat{\pi}^{t}}\big(x_{1};R^{-i}\big)$ w.r.t. the learned policy $\widehat{\pi}^{t}$. 
The different structures of $F$-function and $G$-function lead to different ways of bounding $\rm(i.2)$ and $\rm(i.3)$.
When we set $(\zeta_{2},\zeta_{3})=(\texttt{PES},\texttt{OPT})$, 
we have that $\rm(i.2)\leq 0$ and $\rm(i.3)\leq 0$ since $F^{-i}_{t}$ and $G^{-i}_{t}$ are the pessimistic and optimistic estimates respectively. Then, we can bound the instantaneous regret of agent $i$ in the exploitation phase as follows
\begin{equation*}
\mathrm{reg}_{i t}\leq V_{1}^{\pi_*}\big(x_{1};R\big)-V_{1}^{\widehat{\pi}^{t}}\big(x_{1};R\big)\leq 2\hat{c}(n+R_{\max})\sqrt{d^{3}H^{6}\iota/K}, \quad 1 \leq t\leq K.
\end{equation*} 
When we set $(\zeta_{2},\zeta_{3})=(\texttt{OPT},\texttt{PES})$,  we can bound $\rm(i.2)$ and $\rm(i.3)$ by $2\hat{c}(n+R_{\max})\sqrt{d^{3}H^{6}\iota/K}$ respectively with high probability.  
Thus, we bound the instantaneous regret of agent $i$ in the exploitation phase as 
\begin{equation*}
\mathrm{reg}_{i t}\leq 6\hat{c}(n+R_{\max})\sqrt{d^{3}H^{6}\iota/K}, \quad K < t\leq T.
\end{equation*} 
Combining the regrets incurred in both phases, we obtain with high probability,
\begin{equation}
\mathrm{Reg}_{i T}\leq H K+ 6\hat{c}(n+R_{\max})\sqrt{d^{3}H^{6}\iota/K}(T-K). \label{eq:areg-k-ps}
\end{equation}


\paragraph{Seller Regret.} 
We can decompose the seller regret into two parts as follows
\begin{equation}\label{equa:seller_regret_explore_exploit}
\mathrm{Reg}_{0 T}=\textstyle\sum_{t=1}^{K}\mathrm{reg}_{0 t}+\textstyle\sum_{t=K+1}^{T}\mathrm{reg}_{0 t},
\end{equation}
where $\mathrm{reg}_{0 t}:=u_{0 *}-u_{0 t}$ is the instantaneous regret of the seller. Since the seller charges a price of $0$ to all agents, the instantaneous seller regret in the exploration phase can be bounded as
\begin{equation*}
\mathrm{reg}_{0 t}\leq u_{0 *}-\min_{\pi}V^{\pi}(x_{1};r_{0})\leq u_{0 *}\leq H(n+R_{\max}), \quad 1 \leq t\leq K.
\end{equation*}
For the instantaneous seller regret in the exploitation phase ($K < t\leq T$), we have the following decomposition

\small
\begin{equation*}
\begin{aligned}
	\mathrm{reg}_{0 t}&=(n-1)\underbrace{\big[V_{1}^{\widehat{\pi}^{t}}\big(x_{1};R\big)-V_{1}^{*}\big(x_{1};R\big)\big]}_{\rm(ii.1)}+\sum_{i=1}^{n}\underbrace{\big[V_{1}^{*}\big(x_{1};R^{-i}\big)-F_{t}^{-i}\big]}_{\rm(ii.2)}+\sum_{i=1}^{n}\underbrace{\big[G_{t}^{-i}-V^{\widehat{\pi}^{t}}\big(x_{1};R^{-i}\big)\big]}_{\rm(ii.3)}.
\end{aligned}
\end{equation*}
\normalsize
Here we have $\rm(ii.1)=-\rm(i.3)$, $\rm(ii.2)=-\rm(i.1)$, and $\rm(ii.3)=-\rm(i.2)$  with $\rm(i.1),\rm(i.2),\rm(i.3)$ defined in Equation \eqref{equation:decomp_agent_regret_EWC}. Notice that $\rm(ii.1)\leq 0$ always holds regardless of the choice of $(\zeta_{2},\zeta_{3})$. We can upper bound $\rm(ii.2)$ and $\rm(ii.3)$ using the same method as bounding $\rm(i.1)$ and $\rm(i.2)$. Thus, with high probability, $\mathrm{reg}_{0 t}$ in the exploitation phase ($K < t\leq T$) is upper bounded as 
\begin{equation*}
\mathrm{reg}_{0 t}\leq\begin{cases}
	4\hat{c} n(n+R_{\max})\sqrt{d^{3}H^{6}\iota/K}& \text{if $(\zeta_{2},\zeta_{3})=(\texttt{PES},\texttt{OPT})$}\\
	0& \text{if $(\zeta_{2},\zeta_{3})=(\texttt{OPT},\texttt{PES})$}.
\end{cases}
\end{equation*} 
Combining the above results, the seller regret $\mathrm{Reg}_{0 T}$ is bounded by
\begin{equation}\label{eq:sreg-k-ps}
\begin{cases}
	H(n+R_{\max})K + 4\hat{c} n(n+R_{\max})\sqrt{d^{3}H^{6}\iota/K}(T-K)& \text{if $(\zeta_{2},\zeta_{3})=(\texttt{PES},\texttt{OPT})$}\\
	H(n+R_{\max})K& \text{if $(\zeta_{2},\zeta_{3})=(\texttt{OPT},\texttt{PES})$}.
\end{cases}
\end{equation}

\paragraph{Choice of $K$.} We determine the value of $K$ which can give a tight bound of $\max\{ n\mathrm{Reg}_T^W, \allowbreak \mathrm{Reg}^{\sharp}_T,\mathrm{Reg}_{0T}\}$ where $\mathrm{Reg}^{\sharp}_T = \sum_{i=1}^n \mathrm{Reg}_{i T}$. According to \eqref{eq:wreg-k-ps}, \eqref{eq:areg-k-ps}, and \eqref{eq:sreg-k-ps}, comparing the upper bounds of $n\mathrm{Reg}_T^W$, $\mathrm{Reg}^{\sharp}_T$, and $\mathrm{Reg}_{0T}$, we always have 
\begin{align*}
\max\{ n\mathrm{Reg}_T^W,\mathrm{Reg}^{\sharp}_T,\mathrm{Reg}_{0T}\} \leq H(n+R_{\max})nK+6\hat{c}(n+R_{\max})n\sqrt{d^{3}H^{6}\iota/K}(T-K).
\end{align*}
Focusing on the factors of $H$, $n$, $d$, $T$, and $\iota$, we set $K = dH^{4/3}\iota^{1/3}T^{2/3}$, which can minimize the order of these factors in the above inequality, and obtain the bounds in Theorem \ref{theorem:ETC}.


Next, we provide the proof sketches for the approximate individual rationality and truthfulness. Note that in the following analysis, we do not assume the agents are reporting truthfully. We denote the potentially untruthful reward function of agent $i$ at step $h$ by $\tilde{r}_{i h}$ and then $\tilde{r}_i = \{\tilde{r}_{i h}\}_{h=1}^H$. We further let $\tilde{R}^{-i}:=r_0 + \sum_{j=1,j\neq i}^{n}\tilde{r}_{j}$. 


\paragraph{Individual Rationality.} 
To prove the individual rationality, we assume that agent $i$ reports truthfully according to the reward function $r_i$ and other agents may report untruthfully according to the reward function $\tilde{r}_j$ for $j \neq i$. Under this reward setting, let $\tilde{\pi}_t^{\dag i}$ be the learned seller's policy substituting $\widehat{\pi}^t$ in Algorithm \ref{algorithm:LMVL}, which is generated by Algorithm \ref{algorithm:L3} in the current reward setting. We further denote the associated $F$ and $G$ functions as $F^{\dag,-i}_t$ and $G^{\dag,-i}_t$ generated by Algorithms \ref{algorithm:L3} and \ref{algorithm:L4} respectively. Note that we do not charge the agents in the exploration phase ($t\leq K$), and hence the utilities in this phase are always non-negative. Thus, we only need to consider the utilities in the exploitation phase ($t > K$). 
Then, according to the definition of $u_{it}$, under the current setting of the reward, the instantaneous utility $u_{it}$ of agent $i$ can be decomposed as 
\begin{equation}\label{eq2:ir 1}
u_{it}=V_{1}^{\tilde{\pi}_t^{\dag i}}(x_{1};r_{i})-p_{it}^\dag=\underbrace{\Big[V_{1}^{\tilde{\pi}_t^{\dag i}}\big(x_{1};r_i +\tilde{R}^{-i}\big)-F_{t}^{\dag, -i}\Big]}_{\rm(iii.1)}+\underbrace{\Big[G_{t}^{\dag, -i}-V_{1}^{\tilde{\pi}_t^{\dag i}}\big(x_{1};\tilde{R}^{-i}\big)\Big]}_{\rm(iii.2)},
\end{equation}
where $p_{it}^\dag = F_{t}^{\dag, -i}-G_{t}^{\dag, -i}$. 
To prove the individual rationality, we bound (iii.1) and (iii.2) from below. 
Here we denote the optimistic version of $F_{t}^{\dag, -i}$, when $\zeta_2 = \texttt{OPT}$, by $\widehat{V}_{1}^{t,\dag}\big(x_{1};\tilde{R}^{-i}\big)$ according to Algorithm \ref{algorithm:L3}, which implies $F_{t}^{\dag, -i} \leq \widehat{V}_{1}^{t,\dag}\big(x_{1};\tilde{R}^{-i}\big)$. Then, we have ${\rm(iii.1)} \geq V_{1}^{\tilde{\pi}_t^{\dag i}}\big(x_{1};r_{i }+\tilde{R}^{-i}\big)-\widehat{V}_{1}^{t,\dag}\big(x_{1};\tilde{R}^{-i}\big)$. This can be further decomposed as
\begin{equation*}
\begin{aligned}
	&V_{1}^{\tilde{\pi}_t^{\dag i}}\big(x_{1};r_{i }+\tilde{R}^{-i}\big)-\widehat{V}_{1}^{t,\dag}\big(x_{1};\tilde{R}^{-i}\big)=\underbrace{\Big[V_{1}^{*}(x_1;r_{i }+\tilde{R}^{-i})-V_{1}^{*}(x_1;\tilde{R}^{-i})\Big]}_{\text{(iii.1a)}}\\
	&+\underbrace{\Big[V_{1}^{\tilde{\pi}_t^{\dag i}}\big(x_{1};r_{i }+\tilde{R}^{-i}\big)-V_{1}^{*}(x_1;r_{i }+\tilde{R}^{-i})\Big]}_{\text{(iii.1b)}}+\underbrace{\Big[V_{1}^{*}(x_1;\tilde{R}^{-i})-\widehat{V}_{1}^{t,\dag}\big(x_{1};\tilde{R}^{-i}\big)\Big]}_{\text{(iii.1c)}}.
\end{aligned}
\end{equation*}
Note that $\text{(iii.1a)}\geq0$ always holds since both terms in $\text{(iii.1a)}$ are optimal value functions but $V_{1}^{*}(x_1;r_{i }+\tilde{R}^{-i})$ has larger reward function. Here $\text{(iii.1b)}$ is the suboptimality of policy $\tilde{\pi}_t^{\dag i}$ and $\text{(iii.1c)}$ is the estimation error of $V_{1}^{*}(x_1;\tilde{R}^{-i})$ by $\widehat{V}_{1}^{t,\dag}\big(x_{1};\tilde{R}^{-i}\big)$. We lower bound $\text{(iii.1b)}$ and $\text{(iii.1c)}$ by $-2\hat{c} (n+R_{\max})\sqrt{d^3H^6\iota / K}$ respectively with high probability. Then $\rm(iii.2)$ can be lower bounded by $-4\hat{c} (n+R_{\max})\sqrt{d^3H^6\iota / K}$. For $\rm(iii.2)$, the policy evaluation error for policy $\tilde{\pi}_t^{\dag i}$, we can lower bound it by $-2\hat{c} (n+R_{\max})\sqrt{d^3H^6\iota / K}$ invoking Lemma \ref{lemma:basic_lemma}. Recall that we set $K=dH^{4/3}\iota^{1/3}T^{2/3}$.
Then we lower bound the summation of $\rm(iii.1)$ and $\rm(iii.2)$ over $T$ episodes by $-4\hat{c}(n+R_{\max})d H^{7/3}\iota^{1/3} T^{2/3}$ and $-2\hat{c}(n+R_{\max})d H^{7/3}\iota^{1/3} T^{2/3}$ respectively. 
Combining these two parts, with high probability, we have 
\begin{align*}
U_{i T}\leq-6\hat{c}(n+R_{\max})d H^{7/3}\iota^{1/3}T^{2/3},     
\end{align*}
which indicates that the learned mechanism is $6\hat{c}(n+R_{\max})d H^{7/3}\iota^{1/3}T^{2/3}$-approximately individually rational. 

\paragraph{Truthfulness.}  
We consider two cases: (1) agent $i$ reports truthfully and others may report untruthfully (2) all agents may report untruthfully. Then we denote by $r_i$ the truthful reward and $\tilde{r}_i$ the potentially untruthful reward for all $i\in [n]$. For case (1), we adopt the same definitions of $F^{\dagger, -i}_t, G^{\dagger, -i}_t$, $\tilde{\pi}_t^{\dag i}$, and $u_{it}=V_{1}^{\tilde{\pi}_t^{\dag i}}(x_{1};r_{i})-p_{it}^\dag$ as in the above proof of individual rationality. For case (2), under the untruthful reporting of $\{\tilde{r}_i\}_{i\in[n]}$, we let $\tilde{\pi}_t^{\ddag}$ be the learned policy for the seller under the reward $\tilde{R} := r_0 + \sum_{i=1}^n \tilde{r}_i$ in Algorithm \ref{algorithm:LMVL}, $F^{\ddagger, -i}_t$ and $G^{\ddagger, -i}_t$ be the associated $F$ and $G$ functions generated by Algorithms \ref{algorithm:L3} and \ref{algorithm:L4} respectively, and $\tilde{u}_{it}=V_{1}^{\tilde{\pi}_t^{\ddag}}(x_{1};r_{i})-p_{it}^\ddag$ with $p_{it}^\ddag = F^{\ddagger, -i}_t-G^{\ddagger, -i}_t$. 
We then have the following decomposition
\begin{equation}\label{equa:truthful_decomp_0}
\tilde{U}_{i T}-U_{i T}=\textstyle\sum_{t=1}^{K}(\tilde{u}_{i t}-u_{i t})+\textstyle\sum_{t=K+1}^{T}(\tilde{u}_{i t}-u_{i t}).
\end{equation}
For the first summation, since the agents are not charged, we have
\begin{equation*}
\textstyle\sum_{t=1}^{K}(\tilde{u}_{i t}-u_{i t})\leq \textstyle\sum_{t=1}^{K}\tilde{u}_{i t}\leq\textstyle\sum_{t=1}^{K}\max_{\pi}V^{\pi}(x_{1};r_{i}) \leq H K.
\end{equation*}
We now turn to decomposing the second summation in Equation \eqref{equa:truthful_decomp_0}. 
We have for $t>K$,
\begin{equation*}
\tilde{u}_{i t}-u_{i t} =\big[V_{1}^{\tilde{\pi}_t^{\ddag}}(x_{1};r_{i})-F_{t}^{\ddagger,-i}+G_{t}^{\ddagger,-i}\big]-\big[ V_{1}^{\tilde{\pi}_t^{\dag i}}(x_{1};r_{i})-F_{t}^{\dagger,-i}+G_{t}^{\dagger,-i}\big].
\end{equation*}
Notice that when $\zeta_{1}=\texttt{ETC}$, we only use the data collected in the exploration phase to calculate the $F$ function. Thus, we have $F^{\dagger, -i}_t=F^{\ddagger, -i}_t$. Then, we can show that $\tilde{u}_{it} - u_{it}$ can be decomposed as 
\begin{equation*}
\begin{aligned}
&\tilde{u}_{it} - u_{it}=\underbrace{\big[{G}_{t}^{\ddagger,-i}-V_{1}^{\tilde{\pi}_t^{\ddag}}\big(x_{1};\tilde{R}^{-i}\big)\big]}_{\text{(iv.1)}}+\underbrace{\big[V_{1}^{\tilde{\pi}_t^{\dag i}}\big(x_{1};\tilde{R}^{-i}\big)-G_{t}^{\dagger,-i}\big]}_{\text{(iv.2)}}\\
&+\underbrace{\big[V_{1}^{\tilde{\pi}_t^{\ddag}}\big(x_{1};r_{i}+\tilde{R}^{-i}\big)-V_{1}^{\tilde{\pi}_*^{i}}\big(x_{1};r_{i}+\tilde{R}^{-i}\big)\big]}_{\text{(iv.3)}}+\underbrace{\big[V_{1}^{\tilde{\pi}_*^{i}}\big(x_{1};r_{i}+\tilde{R}^{-i}\big)-V_{1}^{\tilde{\pi}_t^{\dag i}}\big(x_{1};r_{i}+\tilde{R}^{-i}\big)\big]}_{\text{(iv.4)}}.\\      
\end{aligned}
\end{equation*}
We remark that different from the bandit setting in \citet{2020Mechanism}, 
the estimates of value functions are not linear
w.r.t. the reward functions, i.e., $\widehat{V}_{1}^{t,\pi}(x_{1};R_{1})+\widehat{V}_{1}^{t,\pi}(x_{1};R_{2})\neq\widehat{V}_{1}^{t,\pi}(x_{1};R_{1}+R_{2})$ or $\check{V}_{1}^{\pi}(x_{1};R_{1})+\check{V}_{1}^{t,\pi}(x_{1};R_{2})\neq\check{V}_{1}^{t,\pi}(x_{1};R_{1}+R_{2})$ for any reward functions $R_{1}$ and $R_{2}$, due to the truncation of $Q$-functions in Algorithm \ref{algorithm:L3} and Algorithm \ref{algorithm:L4}. 
However, the true value function, i.e., ${V}_{1}^{\pi}(x_{1};R_{1})+{V}_{1}^{\pi}(x_{1};R_{2})={V}_{1}^{\pi}(x_{1};R_{1}+R_{2})$, is linear w.r.t. the reward function. This leads to a novel and more complex decomposition in the above equation. 
Note that $\text{(iv.3)}\leq0$ since $V_1^*(x_1;r_{i}+\tilde{R}^{-i})= \max_{\pi}V_{1}^{\pi}(x_1;r_i+\tilde{R}^{-i})$. And (iv.4) is the suboptimality of policy $\tilde{\pi}_t^{\dag i}$. Then, with high probability, the term (iv.4) is upper bounded by $2\hat{c}(n+R_{\max})\sqrt{d^3H^6\iota/K}$. Here (iv.1) and (iv.2) are evaluation errors depending on the setting of $\zeta_3$ under different reward settings. When $\zeta_3 = \texttt{OPT}$, we have (iv.1) $\leq 2\hat{c} (n+R_{\max})\sqrt{d^3H^6\iota/K}$ while (iv.2) $\leq 0$. And when $\zeta_3 = \texttt{PES}$, we have $\textrm{(iv.1)} \leq 0$ and (iv.2) $\leq 2\hat{c}(n+R_{\max}) \sqrt{d^3H^6\iota/K}$. Thus, regardless of the choices for $\zeta_2, \zeta_3$, we always have 
\begin{equation*}
\tilde{u}_{it} - u_{it}\leq  4\hat{c}(n+R_{\max})\sqrt{d^3H^6\iota/K}, \qquad t > K.
\end{equation*}
Summing up the regret incurred in both the exploration and exploitation phases as in \eqref{equa:truthful_decomp_0}, and setting $K=dH^{4/3}\iota^{1/3}T^{2/3}$, with high probability, we have 
\begin{equation*}
\tilde{U}_{i T}-U_{i T}\leq \big(1+ 4\hat{c}(n+R_{\max})\big)d H^{7/3}\iota^{1/3}T^{2/3},
\end{equation*}
which implies that the mechanism learned by our algorithm is $\big(1+ 4\hat{c}(n+R_{\max})\big)d H^{7/3}\iota^{1/3}T^{2/3}$-approximately truthful.


\subsection{Proof Sketch of Theorem \ref{theorem:OPT}}\label{sec:proof42}
We assume that all agents report their rewards truthfully in the proof of the upper bounds of the welfare regret, the agent regret, and the seller regret. Although we use all the data generated in $T$ rounds to compute our mechanism when $\zeta_1=\texttt{EWC}$, we still need to perform reward-free exploration for individual rationality and truthfulness. Thus, we also decompose regrets into two components: the regret incurred in the exploration phase and the regret incurred in the exploitation phase. 
\paragraph{Welfare Regret.} 
We adopt the same decomposition as in Equation \eqref{equa:welfare_regret_explore_exploit} and decompose the welfare regret as $\mathrm{Reg}_{T}^{W}=\textstyle\sum_{t=1}^{K}\mathrm{reg}_{t}^{W}+\textstyle\sum_{t=K+1}^{T}\mathrm{reg}_{t}^{W}$.
The first summation $\sum_{t=1}^{K}\mathrm{reg}_{t}^{W}$, the welfare regret incurred in the exploration phase, can be bounded by $(n+R_{\max}) H K$ as in Section \ref{sec:5.1}. The key difference between the proofs of welfare regrets in Theorem \ref{theorem:OPT} and Theorem \ref{theorem:ETC} lies in the upper bound of $\sum_{t=K+1}^{T}\mathrm{reg}_{t}^{W}$, i.e., the regret incurred in the exploitation phase.  When $\zeta_1=\texttt{EWC}$, we use the information gathered up to round $t$ for planning in the exploitation phase, instead of just using the $K$ rounds' exploration data as we do when $\zeta_1=\texttt{ETC}$. Thus, we can bound the regret incurred in the exploitation phase by $6\hat{c}(n+R_{\max})\sqrt{d^3 H^4  (T-K)\iota^2}$ with high probability, whose proof takes inspiration from the regret proof for online linear MDPs with exploration, as the calculation of social welfare does not involve prices. 
Combining the regrets incurred in both phases,  with high probability, the following welfare regret bound holds
\begin{equation}\label{eq2:w regret}
\mathrm{Reg}_{T}^{W}\leq (n+R_{\max})HK+6\hat{c}(n+R_{\max})\sqrt{d^3 H^4  (T-K)\iota^2},
\end{equation}
where the rounds of the exploration phase $K$ will be determined later.
\paragraph{Agent Regret.} Following Equation \eqref{equa:seller_regret_explore_exploit}, we decompose the regret of agent $i$ in terms of the exploration phase and exploitation phase as $\mathrm{Reg}_{i T}=\sum_{t=1}^{K}\mathrm{reg}_{i t}+\sum_{t=K+1}^{T}\mathrm{reg}_{i t}$. For the first summation $\sum_{t=1}^{K}\mathrm{reg}_{i t}$, the agent $i$'s regret in the exploration phase, we can bound it by $H K$ as in Section \ref{sec:5.1}.
For the term $\sum_{t=K+1}^{T}\mathrm{reg}_{i t}$, recalling the decomposition in Equation \eqref{equation:decomp_agent_regret_EWC}, it can be decomposed as  
\begin{equation*}
\underbrace{\sum_{t=K+1}^{T}\big[V_{1}^{\pi_*}\big(x_{1};R\big)-V_{1}^{\widehat{\pi}^{t}}\big(x_{1};R\big)\big]}_{\rm(i.1)}+\underbrace{\sum_{t=K+1}^{T}\big[F_{t}^{-i}-V_{1}^{\pi_*^{-i}}\big(x_{1};R^{-i}\big)\big]+\big[V_{1}^{\widehat{\pi}^{t}}\big(x_{1};R^{-i}\big)-G_{t}^{-i}\big]}_{\rm(i.2)}.
\end{equation*}
For term $\rm(i.1)$, we can bound it by $6\hat{c}(n+R_{\max})\sqrt{d^3 H^4  (T-K)\iota^2}$ with high probability leveraging the information gathered up to round $t$ instead of $K$ in the exploitation phase, whose proof follows the proof for welfare regret when $\zeta_1=\texttt{EWC}$. For term $\rm(i.2)$, following the same proof in Section \ref{sec:5.1}, we get an upper bound $0$ when $(\zeta_{2},\zeta_{3})=(\texttt{PES},\texttt{OPT})$ and an upper bound $4\hat{c}(n+R_{\max})\sqrt{d^{3}H^{6}\iota/K}(T-K)$ when $(\zeta_{2},\zeta_{3})=(\texttt{OPT},\texttt{PES})$.
Combining the upper bounds of $\rm(i.1), \rm(i.2)$ for $\sum_{t=K+1}^{T}\mathrm{reg}_{i t}$ and the regret bound for the exploitation phase $\sum_{t=1}^{K}\mathrm{reg}_{i t}\leq H K$, with high probability, $\mathrm{Reg}_{i T}$ has the following upper bound,
\begin{small}
\begin{equation}
\begin{cases}HK+6\hat{c}(n+R_{\max})\sqrt{d^3 H^4  (T-K)\iota^2}
	& \text{if $(\zeta_{2},\zeta_{3})=(\texttt{PES},\texttt{OPT})$}\\
	HK+\hat{c}(n+R_{\max})\big(6\sqrt{d^3 H^4  (T-K)\iota^2}+4\sqrt{d^{3}H^{6}\iota/K}(T-K)\big)& \text{if $(\zeta_{2},\zeta_{3})=(\texttt{OPT},\texttt{PES})$}.
\end{cases}\label{eq2:a regret}
\end{equation}
\end{small}

\paragraph{Seller Regret.}
Since the trajectories we 
collected are according to the process where all the agents are engaged, we can not make a better estimation of the VCG prices even if we use the information gathered in the exploitation phase. Also, note that the seller regret comes from the estimation error of the VCG prices, we cannot improve the analysis of the seller regret. 
Thus, 
we reuse the proof in Section \ref{sec:5.1}, and can get the upper bound of seller regret $\mathrm{Reg}_{0 T}$ as
\begin{equation}\label{eq2:s regret}
\begin{cases}
H(n+R_{\max})K + 4\hat{c} n(n+R_{\max})\sqrt{d^{3}H^{6}\iota/K}(T-K)& \text{if $(\zeta_{2},\zeta_{3})=(\texttt{PES},\texttt{OPT})$}\\
H(n+R_{\max})K& \text{if $(\zeta_{2},\zeta_{3})=(\texttt{OPT},\texttt{PES})$}.
\end{cases}
\end{equation}

\paragraph{Choice of $K$.} We determine the value of $K$ which can give a tight bound of $\max\{ n\mathrm{Reg}_T^W, \allowbreak \mathrm{Reg}^{\sharp}_T,\mathrm{Reg}_{0T}\}$ where $\mathrm{Reg}^{\sharp}_T = \sum_{i=1}^n \mathrm{Reg}_{i T}$. According to \eqref{eq2:w regret}, \eqref{eq2:a regret}, and \eqref{eq2:s regret}, comparing the upper bounds of $n\mathrm{Reg}_T^W$, $\mathrm{Reg}^{\sharp}_T$, and $\mathrm{Reg}_{0T}$, we always have the upper bound of $\max\{ n\mathrm{Reg}_T^W,\mathrm{Reg}^{\sharp}_T,\mathrm{Reg}_{0T}\}$ as
\begin{align*}
n(n+R_{\max})\big(H K+ 6\hat{c}\sqrt{d^3 H^4  (T-K)\iota^2}+4\hat{c}\sqrt{d^{3}H^{6}\iota/K}(T-K)\big).
\end{align*}
Focusing on the factors of $H$, $n$, $d$, $T$, and $\iota$, we set $K = dH^{4/3}\iota^{1/3}T^{2/3}$, which can minimize the order of these factors in the above inequality, and obtain the bounds in Theorem \ref{theorem:OPT}.

\paragraph{Individual Rationality.} 
We assume that agent $i$ reports truthfully according to the reward function $r_i$ and other agents may report untruthfully according to the reward function $\tilde{r}_j$ for $j \neq i$. According to the above assumption, agent $i$ cannot manipulate the policy used during the exploitation phase, which implies that agent $i$ can not influence trajectories collected during the exploitation phase. Note that the only difference between the algorithm when $\zeta_1=\texttt{EWC}$ and $\zeta_1=\texttt{ETC}$ is the trajectories collected during exploitation are used for estimating policy and VCG prices. Thus, agent $i$ cannot affect policy and VCG price estimates obtained during exploration. Hence we can reuse the proof for individual rationality in Section \ref{sec:5.1} and get the conclusion that the mechanism we learned is $6\hat{c} (n+R_{\max})d H^{7/3}\iota^{1/3}T^{2/3}$-approximately individually rational.

\paragraph{Truthfulness.} The proof for truthfulness when $\zeta_1 = \texttt{EWC}$ significantly differs from the case when $\zeta_1 = \texttt{ETC}$. At a high level, when $\zeta_1 = \texttt{ETC}$, we use the fact that the data used to calculate $F$ is collected entirely during the exploration phase and is not affected by agent $i$ potentially reporting untruthfully, and hence $F_t^{\ddagger, -i}$ and $F_t^{\dagger, -i}$ cancel out. Unfortunately, 
when $\zeta_1 = \texttt{EWC}$, $F$ depends the untruthful behavior of agent $i$. The trajectories collected during exploitation affect $F$. The policy used for collecting these trajectories is affected by the agent $i$'s report. Because agent $i$'s untruthfulness impacts $F$, we need to bound the difference between $F_t^{\dagger, -i}$ and $F_t^{\ddagger, -i}$, which is different from the proof of truthfulness in Section~\ref{sec:5.1}.
Thus, we follow the decomposition in Equation~\eqref{equa:truthful_decomp_0}. For the first summation in Equation~\eqref{equa:truthful_decomp_0}, which corresponds to the exploration phase, we can upper bound it by $HK$. 
For the second summation that relates to the exploitation phase, regardless of other agents' truthfulness, the amount of utility an agent gains from untruthful reporting $\tilde{u}_{it} - u_{it}$ for $t>K$ can be decomposed as
\begin{align*}
&\tilde{u}_{it} - u_{it} = \underbrace{\big[V_{1}^{\tilde{\pi}_t^{\ddag}}\big(x_{1};r_{i}+\tilde{R}^{-i}\big)-V_{1}^*\big(x_{1};r_{i}+\tilde{R}^{-i}\big)\big]}_{\text{(i.1)}}+\underbrace{\big[V_{1}^*\big(x_{1};r_{i}+\tilde{R}^{-i}\big)-V_{1}^{\tilde{\pi}_t^{\dag i}}\big(x_{1};r_{i}+\tilde{R}^{-i}\big)\big]}_{\text{(i.2)}}\\
&\qquad \qquad \quad  +\underbrace{\big[{G}_{t}^{\ddagger,-i}-V_{1}^{\tilde{\pi}_t^{\ddag}}\big(x_{1};\tilde{R}^{-i}\big)\big]}_{\text{(i.3)}}+\underbrace{\big[V_{1}^{\tilde{\pi}_t^{\dag i}}\big(x_{1};\tilde{R}^{-i}\big)-G_{t}^{\dagger,-i}\big]}_{\text{(i.4)}} + \underbrace{\big[F_t^{\dagger, -i} - F_t^{\ddagger, -i}\big]}_{\textrm{(1.5)}}.
\end{align*}
Following Section~\ref{sec:5.1}, regardless of the choice of $\zeta_3$, with high probability, we have
\[
\textrm{(i.1)} + \textrm{(i.2)} + \textrm{(i.3)} + \textrm{(i.4)} \leq 4\hat{c} (n+R_{\max})\sqrt{d^3 H^6 \iota/K}.
\] 
We next focus on the upper bound of (i.5). When $\zeta_1 = \texttt{EWC}$, the trajectories collected during the exploitation phase may differ for the computations of $\widehat{V}_1^{t, \dag}(x_1; \tilde{R}^{-i})$ and 
$\check{V}_{1}^{t,\ddagger}(x_1; \tilde{R}^{-i})$, due to agent $i$'s untruthful reporting. Fortunately, the policy evaluation error can still be bounded. The reward-free exploration procedure in Algorithm~\ref{algorithm:L1} ensures that the data collected during exploitation cannot affect the estimated value functions too much. The estimation error surrounding estimated value functions is already small due to the exploration phase. As a result, adding more trajectories during exploitation cannot significantly alter our estimated values, thereby controlling the policy evaluation error. More formally, we have
\begin{align*}
\text{(i.5)}&\leq \underbrace{\Big(\widehat{V}_1^{t, \dag}(x_1; \tilde{R}^{-i}) - V_1^{*}(x_1; \tilde{R}^{-i})\Big)}_{\text{(ii.1)}} + \underbrace{\Big(V_1^{*}(x_1; \tilde{R}^{-i}) - \check{V}_{1}^{t,\ddagger}(x_1; \tilde{R}^{-i})\Big)}_{\text{(ii.2)}},
\end{align*}
where (ii.1) and (ii.2) can be upper bounded by $2\hat{c}\sqrt{d^3 H^6 \iota/K}$ with high probability respectively. In summary, we have that, with high probability, for all $t > K$,
\[
\tilde{u}_{it} - u_{it} \leq 8\hat{c} (n+R_{\max})\sqrt{d^3H^6\iota/K}.
\] 
Summing $\tilde{u}_{it} - u_{it}$ from $t = 1$ to $T$, recalling the bound for all $t \in [K]$, and setting $K=d H^{4/3}\iota^{1/3}T^{2/3}$, with high probability, we get
\begin{equation*}
\Tilde{U}_{iT}-U_{i T}
\leq (1 +8\hat{c}(n+R_{\max}))d H^{7/3}\iota^{1/3}T^{2/3},
\end{equation*}
which implies the mechanism we learned is $(1 +8\hat{c}(n+R_{\max}))d H^{7/3}\iota^{1/3}T^{2/3}$-approximately truthful.

\subsection{Proof Sketch of Theorem \ref{theorem:lowerbound}}

Although the previous work \citet{2020Mechanism} studies the lower bound for mechanism design in the bandit setting, we remark that deriving the lower bound for our problem is non-trivial which requires different constructions and proof techniques from that of this earlier work. Our lower bound takes into account the function approximation and the transition model within the finite horizon, which cannot be handled by \citet{2020Mechanism}. In addition, our work invalidates the Gaussian reward construction in  \citet{2020Mechanism} because of the bounded reward assumption in our work. We use a different construction with the Bernoulli reward and apply a different anti-concentration analysis.

Our lower bound is devised by considering two hard cases for the Markov VCG learning with linear function approximation. 
For the first hard case, we mimic the strategy of the lower bound design as in \citet{2020Mechanism} with constructing two problems $\theta_{0}$ and $\theta_{1}$ that are hard to distinguish. Then, the lower bound is obtained by further lower bounding specific quantities w.r.t. $\theta_{0}$ and $\theta_{1}$. Though we follow such a proving strategy, the model construction is specific to our MDP setting and different from the existing work as discussed above. Specifically, we consider constructing two linear MDPs for the two problems $\theta_{0}$ and $\theta_{1}$ that are hard to distinguish, i.e., they share the same linear feature mapping and deterministic transition kernel but have a small difference in the distribution of reward functions. In addition, we let the dimension of the linear space be $d = n+2$. Note that due to the bounded reward assumption in this work, we define Bernoulli reward functions which further leads to a different anti-concentration analysis.
By bounding the specific quantities associated with $\theta_0$ and $\theta_1$, we obtain a dimension-free lower bound in an order of  $\Omega(n^{4/3} H^{2/3} T^{2/3})$. 

Moreover, to further understand the dependence on any dimension $d$, our second hard case is constructed by the observation that $\max\big( n \mathrm{Reg}_{T}^{W},\mathrm{Reg}_{T}^{\sharp},\mathrm{Reg}_{0 T}\big)\geq n \mathrm{Reg}_{T}^{W}$ always holds. This further inspires us to connect the lower bound to the problem of learning a $d$ dimensional linear MDP with $n+1$ reward functions. We thus prove that the lower bound of $n \mathrm{Reg}_{T}^{W}$ is $\Omega\big(n(n+R_{\max})d\sqrt{H T}\big)$, where the factor $n+R_{\max}$ reflects the impact of the $n$ agent reward functions and the seller reward function 
on the lower bound. Combining the above two hard cases, we eventually obtain the lower bound for our mechanism design problem, which is $\Omega\big(n^{4/3} H^{2/3} T^{2/3} +n(n+R_{\max})d\sqrt{HT}\big)$. Please refer to Appendix \ref{sec:proof-lower} for the detailed proof.

\section{Conclusion}
In this paper, we consider the problem where the agents interact with the mechanism designer according to an unknown MDP. We focus on the online setting with linear function approximation and attempt to recover the dynamic VCG mechanism over multiple rounds of interaction. We propose novel algorithms to learn the mechanism and show that the regret of our proposed method is upper bounded by $\tilde{\cO}(T^{2/3})$, where $T$ is the total number of rounds. We further devise a lower bound, incurring the same $\Omega(T^{2 / 3})$ regret as the upper bound. Our work establishes the regret guarantee for online RL in solving dynamic mechanism design problems without prior knowledge of the underlying model.

\bibliographystyle{ims}
\bibliography{reference}

\newpage 
\appendix

\addtocontents{toc}{\protect\setcounter{tocdepth}{2}}

\centerline{ {\LARGE \textbf{Appendix}} }

\vspace{1cm}

{\hypersetup{
	      linkcolor=black,
	      }
\tableofcontents
}
\newpage 

\section{Table of Notation} \label{sec:tab_notation}
To summarize our notations, we present the following table of notation. 

\begin{table}[!h]
	\begin{small}
		\caption{Table of Notation}
		\centering
		\renewcommand*{\arraystretch}{1.3}
		\begin{tabular}{ >{\centering\arraybackslash}m{2.55cm} | >{\centering\arraybackslash}m{12.5cm} } 
			\hline\hline
			Notation & Meaning \\ 
			\hline
			
			$R$  & summation of the reward functions of the seller and the agents, i.e., $\sum_{i=0}^n r_i$ \\ 
			$R^{-i}$  & summation of the reward functions except that of agent $i$, i.e., $\sum_{j=0,j\neq i}^n r_j$\\ 
			$V^*( ;r)$ & $\max_\pi V^\pi( ;r)$ for any value function $r$ \\
			$\widehat{\pi}^{t}$ & seller's policy in Alg. \ref{algorithm:LMVL} w.r.t. the reward function $R$, generated by Alg. \ref{algorithm:L3} \\
			$\widehat{V}_{h}^{t,*}(x_{1};R)$ & optimistic value function generated by Alg. \ref{algorithm:L3} w.r.t. $R$  \\
			$\widehat{V}_{h}^{t,*}(x_{1};R^{-i})$ & optimistic value function generated by Alg. \ref{algorithm:L3} w.r.t. $R^{-i}$  \\
			$\check{V}_{h}^{t,*}(x_{1};R^{-i})$ & pessimistic value function generated by Alg. \ref{algorithm:L3} w.r.t. $R^{-i}$\\
			$\widehat{V}_{h}^{t,\widehat{\pi}^{t}}(x_{1};R^{-i})$ & optimistic value function generated by Alg. \ref{algorithm:L4} w.r.t. $R^{-i}$ and $\widehat{\pi}^{t}$\\
			$\check{V}_{h}^{t,\widehat{\pi}^{t}}(x_{1};R^{-i})$ & pessimistic value function generated by Alg. \ref{algorithm:L4} w.r.t. $R^{-i}$ and $\widehat{\pi}^{t}$\\
			$F_{t}^{-i}$ & $\widehat{V}_{1}^{t,*}(x_{1};R^{-i})$ if $\zeta_2=\texttt{OPT}$; $\check{V}_{1}^{t,*}(x_{1};R^{-i})$ if $\zeta_2=\texttt{PES}$ \\
			$G_{t}^{-i}$ & $\widehat{V}_{1}^{t,\widehat{\pi}^{t}}(x_{1};R^{-i})$ if $\zeta_3=\texttt{OPT}$; $\check{V}_{1}^{t,\widehat{\pi}^{t}}(x_{1};R^{-i})$ if $\zeta_3=\texttt{PES}$ \\
			$\iota$ & the logarithmic term $\log(36 n d H T/\delta)$\\
			\hline
			$\tilde{r}_i$  & potentially untruthful reward function for agent $i$, $i\in [n]$\\ 
			$\tilde{R}^{-i}$  & $r_0 + \sum_{j=1,j\neq i}^n \tilde{r}_j$\\ 
			$\tilde{\pi}_t^{\dag i}$ & seller's policy in Alg. \ref{algorithm:LMVL} w.r.t. the reward function $r_i + \tilde{R}^{-i}$, generated by Alg. \ref{algorithm:L3} \\
			$\widehat{V}_{h}^{t,\dag}(x_{1};r_i + \tilde{R}^{-i})$ & optimistic value by Alg. \ref{algorithm:L3} w.r.t. $r_i + \tilde{R}^{-i}$ if agents are untruthful except agent $i$  \\
			$\widehat{V}_{h}^{t,\dag}(x_{1};\tilde{R}^{-i})$ & optimistic value by Alg. \ref{algorithm:L3} w.r.t. $\tilde{R}^{-i}$ if agents are untruthful except agent $i$  \\
			$\check{V}_{h}^{t,\dag}(x_{1};\tilde{R}^{-i})$ & pessimistic value by Alg. \ref{algorithm:L3} w.r.t. $\tilde{R}^{-i}$ if agents are untruthful except agent $i$\\
			$\widehat{V}_{h}^{t,\tilde{\pi}_t^{\dag i}}(x_{1};\tilde{R}^{-i})$ & optimistic value by Alg. \ref{algorithm:L4} w.r.t. $\tilde{R}^{-i}$, $\tilde{\pi}_t^{\dag i}$ if agents are untruthful except agent $i$\\
			$\check{V}_{h}^{t,\tilde{\pi}_t^{\dag i}}(x_{1};\tilde{R}^{-i})$ & pessimistic value by Alg. \ref{algorithm:L4} w.r.t. $\tilde{R}^{-i}$, $\tilde{\pi}_t^{\dag i}$ if agents are untruthful except agent $i$\\
			$F_{t}^{\dag,-i}$ & $\widehat{V}_{1}^{t,\dag}(x_{1};\tilde{R}^{-i})$ if $\zeta_2=\texttt{OPT}$; $\check{V}_{1}^{t,\dag}(x_{1};\tilde{R}^{-i})$ if $\zeta_2=\texttt{PES}$ \\
			$G_{t}^{\dag,-i}$ & $\widehat{V}_{1}^{t,\tilde{\pi}_t^{\dag i}}(x_{1};\tilde{R}^{-i})$ if $\zeta_3=\texttt{OPT}$; $\check{V}_{1}^{t,\tilde{\pi}_t^{\dag i}}(x_{1};\tilde{R}^{-i})$ if $\zeta_3=\texttt{PES}$ \\
			\hline
			$\tilde{R}$ & $r_0 + \sum_{i=1}^n \tilde{r}_i$\\
			$\tilde{\pi}_t^{\ddag}$ & seller's policy in Alg. \ref{algorithm:LMVL} w.r.t. the reward function $\tilde{R}$, generated by Alg. \ref{algorithm:L3} \\
			$\widehat{V}_{h}^{t,\ddag}(x_{1};\tilde{R})$ & optimistic value by Alg. \ref{algorithm:L3} w.r.t. $\tilde{R}$ if all agents are untruthful\\
			$\widehat{V}_{h}^{t,\ddag}(x_{1};\tilde{R}^{-i})$ & optimistic value by Alg. \ref{algorithm:L3} w.r.t. $\tilde{R}^{-i}$ if all agents are untruthful\\
			$\check{V}_{h}^{t,\ddag}(x_{1};\tilde{R}^{-i})$ & pessimistic value by Alg. \ref{algorithm:L3} w.r.t. $\tilde{R}^{-i}$ if all agents are untruthful\\
			$\widehat{V}_{h}^{t,\tilde{\pi}_t^{\ddag}}(x_{1};\tilde{R}^{-i})$ & optimistic value by Alg. \ref{algorithm:L4} w.r.t. $\tilde{R}^{-i}$, $\tilde{\pi}_t^{\ddag}$ if all agents are untruthful\\
			$\check{V}_{h}^{t,\tilde{\pi}_t^{\ddag}}(x_{1};\tilde{R}^{-i})$ & pessimistic value by Alg. \ref{algorithm:L4} w.r.t. $\tilde{R}^{-i}$, $\tilde{\pi}_t^{\ddag}$ if all agents are untruthful\\
			$F_{t}^{\ddag,-i}$ & $\widehat{V}_{1}^{t,\ddag}(x_{1};\tilde{R}^{-i})$ if $\zeta_2=\texttt{OPT}$; $\check{V}_{1}^{t,\ddag}(x_{1};\tilde{R}^{-i})$ if $\zeta_2=\texttt{PES}$ \\
			$G_{t}^{\ddag,-i}$ & $\widehat{V}_{1}^{t,\tilde{\pi}_t^{\ddag}}(x_{1};\tilde{R}^{-i})$ if $\zeta_3=\texttt{OPT}$; $\check{V}_{1}^{t,\tilde{\pi}_t^{\ddag}}(x_{1};\tilde{R}^{-i})$ if $\zeta_3=\texttt{PES}$ \\
			\hline \hline
		\end{tabular}
		\label{tab:notation}
	\end{small}
\end{table}

\section{Proof of Lemma~\ref{lemma:Markov VCG mechanism}}
\label{sec:proof_sketch_of_lemma_markov_vcg_mechanism}


\begin{proof} 
	The detailed proof for these three properties can be found in Appendix B of~\citet{lyu2022pessimism}. We include a sketch of the proof here for completeness. The proof for the linear Markov VCG mechanism's properties is provided as follows:
	\begin{enumerate} 
		
		\item \textit{Truthfulness}: We begin by noting that when agent $i$ reports their rewards untruthfully, the untruthful reporting may change the optimal policy of $V_1^\pi(x_1; R)$ by altering only the reported value of $r_i$ and the associated value function $V_1^\pi( ; r_i)$. However, agent $i$ cannot affect the value of $V_1^\pi(x_1;R^{-i})$, as $R^{-i}$ is independent of $r_i$.
		
		With the previous observation in mind, let $\tR_i$ be the untruthful value function reported by agent $i$ and $\Tilde{\pi} = \argmax_{\pi \in \Pi} V_1^\pi(x_1;\tR_i + R^{-i})$. Under the linear Markov VCG mechanism, agent $i$ attains the following utility
		\[
		\Tilde{u}_i = V_1^{\Tilde{\pi}}(x_1; r_i) - V_1^{\pi_*^{-i}}(x_1; R^{-i}) + V_1^{\Tilde{\pi}}(x_1; R^{-i}) = V_1^{\Tilde{\pi}}(x_1; R) - V_1^{\pi_*^{-i}}(x_1; R^{-i}).
		\]
		Similarly, we know $u_i = V_{1}^{\pi_*}(x_1;R) - V_1^{\pi_*^{-i}}(x_1;R^{-i})$ when agent $i$ reports truthfully. Since $\pi_*$ is the maximizer of $V_1^\pi(x_1;R)$, we know $u_i \geq \Tilde{u}_i$, thus proving truthfulness.
		
		\item \textit{Individual Rationality}: For any agent $i$, their utility is given by
		\begin{equation}\label{equa:VCG IR}
			\begin{aligned}
				u_{i *} &= V_1^{\pi_*}(x_1; r_i) - p_{i *} = V_{1}^{\pi_*}(x_1; R) - V_1^{\pi_*^{-i}}(x_1; R^{-i})\\
				&\geq V_1^{\pi_*^{-i}}(x_1; R) - V_1^{\pi_*^{-i}}(x_1; R^{-i}) = V_1^{\pi_*^{-i}}(x_1; r_i) \geq 0,
			\end{aligned}
		\end{equation}
		where we use the fact that $r_{i, h}(s,a) \geq 0$ for all $(i,h,s,a) \in [n]\times[H]\times\cS\times\cA$. 
		\item \textit{Efficiency}: Under truthful reporting, the chosen policy $\pi_*$ is the maximizer of the value-function of welfare $V_{1}^{\pi}(x_{1};R)$ and hence is efficient.
	\end{enumerate}
	This completes the proof.
\end{proof}

\section{Proof of Theorems \ref{theorem:ETC} and \ref{theorem:OPT}} \label{sec:proof-start}

We begin by introducing a crucial result that will be used throughout the rest of the section. This lemma presents the estimation errors of certain value functions by their corresponding optimistic or pessimistic value estimates. 
We refer readers to the table of notation in Section \ref{sec:tab_notation} for detailed definitions of the policies, rewards, and value functions in this lemma.

\begin{lemma}\label{lemma:basic_lemma}
	For both when $\zeta_1 = {\texttt{ETC}}$ and when $\zeta_1 = \texttt{EWC}$, let $\iota=\log(36 n d H T/\delta)$. With probability at least $1 - \delta$, the following statements hold true jointly for all $t > K$ and some absolute constant $\hat{c}$.
	\begin{enumerate}
		\item Regardless of any agent's truthfulness, the policy used is sufficiently close to the one that maximizes the value functions of the \emph{reported} reward functions. More specifically, $V^{*}_1(x_1; \fR) - V^{\pi}_1(x_1; \fR) \leq 2\hat{c}\sqrt{d^3H^6\iota/K}$ for all $(\fR, \pi) \in \{(R, \widehat{\pi}^t), (\Tilde{R}, \Tilde{\pi}^{\ddagger}_t)\} \cup \{(r_i + \Tilde{R}^{-i}, \Tilde{\pi}_t^{\dagger i})\}_{i = 1}^n$.
		\item For all $i \in [n]$, Algorithm~\ref{algorithm:L3} returns a sufficiently good estimate regardless of agent $i$'s or other agents' truthfulness. More specifically, $0 \leq \widehat{V}_1^{t, \pi}(x_1; \fR) - V^*(x_1; \fR) \leq 2\hat{c}\sqrt{d^3H^6\iota/K}$ and $- 2\hat{c}\sqrt{d^3H^6\iota/K}\leq \check{V}_1^{t, \pi}(x_1; \fR) - V^*(x_1; \fR) \leq 0$, for all $(\fR, \pi) \in \{(R^{-i}, \star), (\Tilde{R}^{-i}, \dagger), (\Tilde{R}^{-i}, \ddagger)\}_{i = 1}^n$.
		\item For all $i \in [n]$, Algorithm~\ref{algorithm:L4} returns a sufficiently good estimate regardless of agent $i$'s or other agents' truthfulness. More specifically, $0\leq \widehat{V}_1^{t,\pi}(x_1; \fR) - V^{\pi}_1(x_1; \fR)\leq 2\hat{c} \sqrt{d^3H^6\iota/K}$ and $-2\hat{c} \sqrt{d^3H^6\iota/K} \leq \check{V}_1^{t,\pi}(x_1; \fR) - V^{\pi}_1(x_1; \fR)$, for all $(\fR, \pi) \in \allowbreak\{(R^{-i}, \widehat{\pi}^t),\allowbreak (\Tilde{R}^{-i}, \Tilde{\pi}_t^{\dagger i}), (\Tilde{R}^{-i}, \Tilde{\pi}_t^{\ddagger})\}_{i = 1}^n$.
	\end{enumerate}  
\end{lemma}
Please see Appendix \ref{sec:proof_basic_lemma} for the detailed proof. At the high level, the first clause ensures that the policy executed during exploitation is always sufficiently close to the one that maximizes the sum of the reported reward functions. The second and third clauses ensure that the price estimation is sufficiently good.
With Lemma~\ref{lemma:basic_lemma}, we can obtain the proofs of Theorems \ref{theorem:ETC} and \ref{theorem:OPT}. For a concise presentation, we ignore presenting the probability for a certain inequality holds when calling Lemma \ref{lemma:basic_lemma}. Overall, the results in Theorem \ref{theorem:ETC} and Theorem \ref{theorem:OPT} will hold with probability at least $1-\delta$ respectively, according to the above lemma.

\subsection{Proof of Theorem \ref{theorem:ETC}}
\label{subsec:proof_of_theorem_etc}
\begin{proof} We prove each bound in Theorem~\ref{theorem:ETC} separately. Overall, the inequalities in  Lemma \ref{lemma:basic_lemma} for the proof of Theorem~\ref{theorem:ETC} hold together with probability at least $1-\delta$. For conciseness, we ignore the detailed description of probabilities for each of these inequalities in our proof.
	
	\paragraph{Welfare Regret.} Recall that in Equation \eqref{equa:regret}, the social welfare regret is defined as $\mathrm{Reg}_{T}^{W}=\sum_{t=1}^{T}\mathrm{reg}_{t}^{W}$ where $\mathrm{reg}_{t}^{W}=V_{1}^{\pi_*}(x_1;R)-V_{1}^{\widehat{\pi}^{t}}(x_1;R)$. We begin by decomposing the regret into two parts, the regret suffered in the exploration phase and the regret suffered in the exploitation phase, as follows,
	\begin{equation}\label{decomp:welfare_two_phase}
		\mathrm{Reg}_{T}^{W}=\sum_{t=1}^{K}\mathrm{reg}_{t}^{W}+\sum_{t=K+1}^{T}\mathrm{reg}_{t}^{W}.
	\end{equation}
	For the first summation in Equation \eqref{decomp:welfare_two_phase}, we have
	\begin{equation}\label{equa:ETC wel explore}
		\sum_{t=1}^{K}\mathrm{reg}_{t}^{W}\leq K H(n+R_{\max})=H(n+R_{\max})K,
	\end{equation} recalling that 
	$\mathrm{reg}_{t}^{W}\leq H(n+R_{\max})$ due to the upper bound of the reward functions.
	
	We now turn to the second summation. By Lemma~\ref{lemma:basic_lemma}, for $t > K$ we have
	\begin{align}
		\mathrm{reg}_{t}^{W}&=V_{1}^{\pi_*}(x_1;R)-V_{1}^{\widehat{\pi}^{t}}(x_1;R) \leq 2\hat{c} (n+R_{\max} )\sqrt{d^3H^6\iota/K}. \label{eq:sowell_ind}
	\end{align}
	Summing the above equation form $t=K+1$ to $T$, we have
	\begin{equation}\label{equa:ETC_wel_exploit}
		\sum_{t=K+1}^{T}\mathrm{reg}_{t}^{W}\leq 2\hat{c}(n+R_{\max} ) \sqrt{d^{3}H^{6}\iota/K} (T-K).
	\end{equation}
	Combining Equations \eqref{decomp:welfare_two_phase}, \eqref{equa:ETC wel explore},  and \eqref{equa:ETC_wel_exploit}, we have
	\begin{equation}
		\mathrm{Reg}_{t}^{W}\leq H(n+R_{\max})K + 2\hat{c} (n+R_{\max} )\sqrt{d^3H^6\iota/K} (T-K), \label{eq:wreg-k-p}
	\end{equation}
	where the value of $K$ will be determined by jointly considering the upper bounds of $n\mathrm{Reg}_T^W$, $\mathrm{Reg}^{\sharp}_T$, and $\mathrm{Reg}_{0T}$.

	\paragraph{Agent Regret.} Recall that in Equation \eqref{equa:regret}, the agent regret is defined as $\mathrm{Reg}_{i T}=\sum_{t=1}^{T}\mathrm{reg}_{i t}$, where $\mathrm{reg}_{i t}=u_{i *}-u_{i t}$.
	Similar to our proof for welfare regret, we decompose the regret to that incurred during exploration and exploitation,
	\begin{equation}\label{decomp:agent_regret_two_phase}
		\mathrm{Reg}_{i T}=\sum_{t=1}^{K}\mathrm{reg}_{i t}+\sum_{t=K+1}^{T}\mathrm{reg}_{i t}.
	\end{equation}
	For the first summation in Equation \eqref{decomp:agent_regret_two_phase}, we begin by upper bounding the instantaneous regret of agent $i$ during the exploration phase. As the price charged to the agents is set to 0 during the exploration phase, for any $t \in [K]$, we have
	\begin{equation*}
		\textrm{reg}_{it} \leq u_{i *} - \min_\pi V_1^\pi(x_1; r_i) \leq u_{i *} = V_1^{\pi_*}(x_1; r_i) - p_{i *},
	\end{equation*} where we recall $p_{i *} = V_1^{\pi^{-i}_*}(x_1; R^{-i}) - V_1^{\pi_*}(x_1; R^{-i})$ and use the fact that $r_i \geq 0$. By definition of $\pi^{-i}_*$, we know that $p_{i *} \geq 0$ and 
	$V_1^{\pi_*}(x_1; r_i) \leq H$, using the fact that $r_i \leq 1$. We then have
	\begin{equation*}
		\sum_{t=1}^{K}\mathrm{reg}_{i t}\leq \sum_{t=1}^{K} V_1^{\pi_*}(x_1; r_i) \leq H K.
	\end{equation*}
	Bounding the instantaneous agent regret during the exploitation phase is more complicated, as it depends on not only the suboptimality of the learned policy $\widehat{\pi}^t$ itself, but also the suboptimality incurred by estimation of the VCG price, $p_{it} = F^{-i}_t - G^{-i}_t$. To handle this challenge, we propose the following decomposition for $t > K$,
	\begin{equation}\label{equa:r_it decomp}
		\begin{aligned}
			\mathrm{reg}_{i t}&=u_{i *}-u_{i t}\\
			&=\big[V_{1}^{\pi_{*}}\big(x_{1};r_{i}\big)-V_{1}^{\pi^{-i}_{*}}(x_{1};R^{-i})+V_{1}^{\pi_{*}}(x_{1};R^{-i})\big]-\big[V_{1}^{\widehat{\pi}^{t}}\big(x_{1};r_{i}\big)-F_{t}^{-i}+G_{t}^{-i}\big]\\
			&=\underbrace{\big[V_{1}^{\pi_{*}}(x_{1};R)-V_{1}^{\widehat{\pi}^{t}}(x_{1};R)\big]}_{\displaystyle\rm{(i)}}+\underbrace{\big[F_{t}^{-i}-V_{1}^{\pi_{*}^{-i}}(x_{1};R^{-i})\big]}_{\displaystyle \rm{(ii)}}+\underbrace{\big[V_{1}^{\widehat{\pi}^{t}}(x_{1};R^{-i})-G_{t}^{-i}\big]}_{\displaystyle \rm{(iii)}},
		\end{aligned}
	\end{equation} 
	where the second equation uses the fact that $V_{1}^{\pi}\big(x_{1};r_{i}\big)+V_{1}^{\pi}(x_{1};R^{-i}) = V_{1}^{\pi}(x_{1};R)$ for any $\pi$. The above decomposition allows us to bound the agent regret in terms of (i) suboptimality of $\widehat{\pi}^t$, (ii) estimation error of $F_t^{-i}$, and (iii) policy evaluation error of $G^{-i}_t$.
	
	For term (i), by the result already obtained in Equation \eqref{eq:sowell_ind} for the welfare regret, we have for all $ t > K $, 
	\begin{equation*}
		V_{1}^{\pi_{*}}(x_{1};R)-V_{1}^{\widehat{\pi}^{t}}(x_{1};R)\leq 2\hat{c}(n+R_{\max} )\sqrt{d^{3}H^{6}\iota/K}.
	\end{equation*} 
	We now bound term (ii). Let $\widehat{\pi}_t^{-i}$ be the fictitious policy generated by Algorithm~\ref{algorithm:L3} when calculating $F_t^{-i}$. For $t > K$, when $\zeta_{2}=\texttt{PES}$, $F_t^{-i}=\check{V}_{1}^{t,\widehat{\pi}_t^{-i}}(x_{1};R^{-i})$ , we have
	\begin{equation*}
		\displaystyle{ \rm{(ii)}}=\check{V}_{1}^{t,\widehat{\pi}_t^{-i}}(x_{1};R^{-i})-V_{1}^{\pi_{*}^{-i}}(x_{1};R^{-i}) \leq 0,
	\end{equation*} where the inequality is by Lemma~\ref{lemma:basic_lemma}. When $\zeta_2 = \texttt{OPT}$, $F_t^{-i}=\widehat{V}_{1}^{t,\widehat{\pi}_t^{-i}}(x_{1};R^{-i})$, we have
	\[
	\textrm{(ii)} = \widehat{V}_{1}^{t,\widehat{\pi}_t^{-i}}(x_{1};R^{-i})-V_{1}^{\pi_{*}^{-i}}(x_{1};R^{-i}) \leq 2\hat{c}(n+R_{\max} )\sqrt{d^3H^6\iota/K},
	\] where the inequality also stems from Lemma~\ref{lemma:basic_lemma}.
	
	Term (iii) is controlled in a similar way. By Lemma \ref{lemma:basic_lemma}, for $t\geq K$, when $\zeta_{3}=\texttt{OPT}$, $G_t^{-i}=\widehat{V}_{1}^{t,\widehat{\pi}^t}(x_{1};R^{-i})$, we have 
	\begin{align*}
		\displaystyle{ \rm{(iii)}}=V_{1}^{\widehat{\pi}^{t}}(x_{1};R^{-i})-\widehat{V}_{1}^{t,\widehat{\pi}^{t}}(x_{1};R^{-i})\leq 0,
	\end{align*} 
	and when $\zeta_{3}=\texttt{PES}$, $G_t^{-i}=\check{V}_{1}^{t,\widehat{\pi}^t}(x_{1};R^{-i})$, 
	we have
	\begin{equation*}
		\displaystyle{ \rm{(iii)}}=V_{1}^{\widehat{\pi}^{t}}(x_{1};R^{-i})-\check{V}_{1}^{t,\widehat{\pi}^{t}}(x_{1};R^{-i})\leq 2\hat{c}(n+R_{\max} )\sqrt{d^{3}H^{6}\iota/K}.
	\end{equation*}
	Combining the regrets incurred in both phases, by $\mathrm{Reg}_{i T}=\sum_{t=1}^{T}\mathrm{reg}_{i t}$, we obtain
	\begin{equation}\label{eq:areg-k-p}
		\mathrm{Reg}_{i T}\leq \begin{cases}
			HK + 2\hat{c}(n+R_{\max} )\sqrt{d^{3}H^{6}\iota/K}(T-K) &\textrm{ if $(\zeta_{2},\zeta_{3})=(\texttt{PES},\texttt{OPT})$}\\
			HK + 6\hat{c}(n+R_{\max} )\sqrt{d^{3}H^{6}\iota/K}(T-K) &\textrm{ if $(\zeta_{2},\zeta_{3})=(\texttt{OPT},\texttt{PES})$}.
		\end{cases}
	\end{equation}

	\paragraph{Seller Regret.} Recall that in Equation \eqref{equa:regret}, the seller regret is defined as $\mathrm{Reg}_{0 T}=\sum_{t=1}^{T}\mathrm{reg}_{0 t}$ where $\mathrm{reg}_{0 t}=u_{0 *}-u_{0 t}$. Thus, we have the following decomposition
	\begin{equation}\label{decomp:seller_regret_two_phase}
		\mathrm{Reg}_{0 T}=\sum_{t=1}^{K}\mathrm{reg}_{0 t}+\sum_{t=K+1}^{T}\mathrm{reg}_{0 t}.
	\end{equation}
	We begin with bounding the first summation. Recall that $\mathrm{reg}_{0 t} = u_{0 *} - u_{0 t}$. During exploration, as the seller charges a price of 0 to all agents, their utility is lower bounded by $u_{0 t} = \min_{\pi}V^{\pi}(x_1; r_0) + 0 = \min_{\pi}V^{\pi}(x_1; r_0)$. As $r_0 \geq 0$, we know that for all $t \in [K]$,
	\begin{align*}
		\sum_{t = 1}^K \mathrm{reg}_{0 t} \leq \sum_{t = 1}^K u_{0 *} \leq K u_{0 *}.
	\end{align*} 
	Recall that
	\begin{align*}
		u_{0 *} &= V_1^{\pi_*}(x_1; r_0) + \sum_{i = 1}^n p_{i *} = V_1^{\pi_*}(x_1; r_0) + \sum_{i = 1}^n \rbr{V^{\pi^{-i}_*}(x_1; R^{-i}) - V^{\pi_*}(x_1; R^{-i})}\\
		&= -(n - 1)V_1^{\pi_*}(x_1; R) + \sum_{i = 1}^n V^{\pi^{-i}_*}(x_1; R^{-i}).
	\end{align*} Since $r_i \geq 0$, $R = R^{-i} + r_i \geq R^{-i}$, we have
	\begin{align*}
		u_{0 *} &\leq -(n - 1)V_1^{\pi_*}(x_1; R) + \sum_{i = 1}^n V^{\pi^{-i}_*}(x_1; R) \leq -(n - 1)V_1^{\pi_*}(x_1; R) + \sum_{i = 1}^n V^{\pi_*}(x_1; R)\\
		&= V^{\pi_*}(x_1; R) \leq H(n + R_{\max}),
	\end{align*} 
	according to the definitions of $\pi_*$ and $\pi^{-i}_*$. We then have the following upper bound for the first summation in Equation~\eqref{decomp:seller_regret_two_phase} as
	\[
	\sum_{t=1}^{K}\mathrm{reg}_{0 t} \leq K u_{0*} \leq (n + R_{\max}) H K.
	\]
	We now bound the second summation in Equation \eqref{decomp:seller_regret_two_phase}. The seller's instantaneous regret during exploration can be decomposed as
	\small
	\begin{equation*}
		\begin{aligned}
			\mathrm{reg}_{0 t}&=u_{0 *}-u_{0 t}\\
			&=\bigg[V_{1}^{\pi_{*}}(x_{1};r_{0 })+\sum_{i=1}^{n}p_{i*}\bigg]-\bigg[V_{1}^{\widehat{\pi}^{t}}(x_{1};r_{0 })+\sum_{i=1}^{n}p_{i t}\bigg]\\
			&=\bigg[V_{1}^{\pi_{*}}(x_{1};r_{0 })+\sum_{i=1}^{n}\Big[V_{1}^{\pi^{-i}_{*}}(x_{1};R^{-i})-V_{1}^{\pi_{*}}(x_{1};R^{-i})\Big]\bigg]-\bigg[V_{1}^{\widehat{\pi}^{t}}(x_{1};r_{0 })+\sum_{i=1}^{n}\big(F_{t}^{-i}-G_{t}^{-i}\big)\bigg]\\
			&=\bigg[-(n-1)V_{1}^{\pi_{*}}(x_{1};R)+\sum_{i=1}^{n} V_{1}^{\pi^{-i}_{*}}(x_{1};R^{-i})\bigg]\\
			&\qquad -\bigg[-(n-1)V_{1}^{\widehat{\pi}^{t}}(x_{1};R)+\sum_{i=1}^{n}\big[F_{t}^{-i}-G_{t}^{-i}+V_{1}^{\widehat{\pi}^{t}}(x_{1};R^{-i})\big]\bigg]\\
			&=(n-1)\underbrace{\big[V_{1}^{\widehat{\pi}^{t}}(x_{1};R)-V_{1}^{\pi_*}(x_{1};R)\big]}_{\displaystyle{ \rm{(i)}}}+\sum_{i=1}^{n}\underbrace{\big[V_{1}^{\pi^{-i}_{*}}(x_{1};R^{-i})-F_{t}^{-i}\big]}_{\displaystyle{\rm{(ii)}}}+\sum_{i=1}^{n}\underbrace{\big[G_{t}^{-i}-V^{\widehat{\pi}^{t}}(x_{1};R^{-i})\big]}_{\displaystyle{ \rm{(iii)}}}.
		\end{aligned}
	\end{equation*}
	\normalsize
	For term (i), we have $\displaystyle{ \rm{(i)}}\leq0$ due to the optimality of $V_{1}^{*}$.
	For term (ii), when $\zeta_{2}=\texttt{OPT}$, by the construction of $F_t^{-i}$, we have
	\[
	\textrm{(ii)} = V_{1}^{\pi^{-i}_{*}}(x_{1};R^{-i}) - F_t^{-i} = V_{1}^{\pi^{-i}_{*}}(x_{1};R^{-i}) - \widehat{V}_1^{t, \widehat{\pi}_t^{-i}}(x_1; R^{-i}) \leq 0,
	\] where we invoke Lemma~\ref{lemma:basic_lemma} for the inequality. When $\zeta_2 = \texttt{PES}$, we obtain that
	\[
	\textrm{(ii)} = V_{1}^{\pi^{-i}_{*}}(x_{1};R^{-i}) - F_t^{-i} = V_{1}^{\pi^{-i}_{*}}(x_{1};R^{-i}) - \check{V}_1^{t, \widehat{\pi}_t^{-i}}(x_1; R^{-i}) \leq 2\hat{c} (n+R_{\max} )\sqrt{d^3H^6\iota /K},
	\] where the last inequality also uses Lemma~\ref{lemma:basic_lemma}.
	
	For term (iii), further invoking Lemma~\ref{lemma:basic_lemma}, when $\zeta_{3}=\texttt{PES}$, $\displaystyle{ \rm{(iii)}}\leq0$, and when $\zeta_{3}=\texttt{OPT}$, 
	we have \begin{equation*}
		\displaystyle{ \rm{(iii)}}=\widehat{V}_{1}^{t,\widehat{\pi}^{t}}(x_{1};R^{-i})-V_{1}^{\widehat{\pi}^{t}}\big(x_1;R^{-i}\big)\leq 2\hat{c}(n+R_{\max} )\sqrt{d^{3}H^{6}\iota/K}.
	\end{equation*}
	Combining the bounds for terms (i), (ii), and (iii) above, we have
	\begin{equation*}
		\sum_{t=K+1}^{T}\mathrm{reg}_{0 t}\leq\begin{cases}
			4\hat{c} n(n+R_{\max})\sqrt{d^{3}H^{6}\iota/K}(T-K)& \text{if $(\zeta_{2},\zeta_{3})=(\texttt{PES},\texttt{OPT})$}\\
			0& \text{if $(\zeta_{2},\zeta_{3})=(\texttt{OPT},\texttt{PES})$},
		\end{cases}
	\end{equation*}
	where $\hat{c}$ is some absolute constant. By adding the regret incurred in the exploration phase, this result further gives the upper bound of the seller regret $\mathrm{Reg}_{0 T}$ as
	\begin{equation}\label{eq:sreg-k-p}
		\begin{cases}
			H(n+R_{\max})K + 4\hat{c} n(n+R_{\max})\sqrt{d^{3}H^{6}\iota/K}(T-K)& \text{if $(\zeta_{2},\zeta_{3})=(\texttt{PES},\texttt{OPT})$}\\
			H(n+R_{\max})K& \text{if $(\zeta_{2},\zeta_{3})=(\texttt{OPT},\texttt{PES})$}.
		\end{cases}
	\end{equation}

	\paragraph{Choice of $K$.} Now we determine the value of $K$ that can lead to a tight bound of $\max\{ n\mathrm{Reg}_T^W, \allowbreak \mathrm{Reg}^{\sharp}_T,\mathrm{Reg}_{0T}\}$, where $\mathrm{Reg}^{\sharp}_T = \sum_{i=1}^n \mathrm{Reg}_{i T}$ as defined in Equation \eqref{equa:regret}. According to Equations \eqref{eq:wreg-k-p}, \eqref{eq:areg-k-p}, and \eqref{eq:sreg-k-p}, comparing the upper bounds of $n\mathrm{Reg}_T^W$, $\mathrm{Reg}^{\sharp}_T$, and $\mathrm{Reg}_{0T}$, we always have 
	\begin{align*}
		\max\{ n\mathrm{Reg}_T^W,\mathrm{Reg}^{\sharp}_T,\mathrm{Reg}_{0T}\} \leq H(n+R_{\max})nK+6\hat{c}(n+R_{\max})n\sqrt{d^{3}H^{6}\iota/K}(T-K).
	\end{align*}
	Focusing on the factors of $H$, $n$, $d$, $T$, and $\iota$, we set $K = dH^{4/3}\iota^{1/3}T^{2/3}$, which can minimize the order of these factors in the above inequality, and obtain the bound 
	\begin{align*}
		\max\{ n\mathrm{Reg}_T^W,\mathrm{Reg}^{\sharp}_T,\mathrm{Reg}_{0T}\} = \cO \big(n(n+R_{\max})d H^{7/3}\iota^{1/3}T^{2/3}\big).
	\end{align*}
	Thus, plugging $K = dH^{4/3}\iota^{1/3}T^{2/3}$ into \eqref{eq:wreg-k-p}, we have 
	\begin{equation*}
		\mathrm{Reg}_{T}^{W}\leq(1+2\hat{c})(n+R_{\max})d H^{7/3}\iota^{1/3}T^{2/3}.
	\end{equation*}
	Plugging the value of $K$ into \eqref{eq:areg-k-p}, we have 
	\begin{align*}
		\mathrm{Reg}_{i T}\leq \begin{cases}
			\big(1 + 2\hat{c}(n+R_{\max} )\big)dH^{7/3}\iota^{1/3}T^{2/3} &\textrm{ if $(\zeta_{2},\zeta_{3})=(\texttt{PES},\texttt{OPT})$}\\
			\big(1 + 6\hat{c}(n+R_{\max} )\big)dH^{7/3}\iota^{1/3}T^{2/3} &\textrm{ if $(\zeta_{2},\zeta_{3})=(\texttt{OPT},\texttt{PES})$}.
		\end{cases}
	\end{align*}
	Plugging the value of $K$ into \eqref{eq:sreg-k-p}, we obtain 
	\begin{equation*}
		\mathrm{Reg}_{0 T}\leq\begin{cases}
			(1+4\hat{c}n)(n+R_{\max})d H^{7/3}\iota^{1/3} T^{2/3}& \text{if $(\zeta_{2},\zeta_{3})=(\texttt{PES},\texttt{OPT})$}\\
			(n+R_{\max})d H^{7/3}\iota^{1/3} T^{2/3}& \text{if $(\zeta_{2},\zeta_{3})=(\texttt{OPT},\texttt{PES})$}.
		\end{cases}
	\end{equation*}
	This completes the proof of the upper bounds of the welfare regret, the agent regret, and the seller regret.

	\paragraph{Individual Rationality.}  We note that for the proof of individual rationality, we do not require the truthfulness of agents other than agent $i$. Recall that if we do not charge the agents in the exploration phase, for any agent $i$, we always have utility $u_{i t} \geq 0$ during exploration because $r_{i} \geq 0$. Thus, we only need to bound from below agent $i$'s utility during the exploitation phase. When agent $i$ reports according to the reward function $r_i$ but other agents report rewards potentially untruthfully according to $\tilde{r}_j$ for $j\neq i$, we define $\tilde{R}^{-i}:=r_0 + \sum_{j\in [n], i\neq j} \tilde{r}_j$ and let $\tilde{\pi}_t^{\dag i}$ substitute $\widehat{\pi}^t$ in Algorithm \ref{algorithm:LMVL}, which is generated by Algorithm \ref{algorithm:L3} in the current reward setting. We further define the associated $F$ and $G$ generated by Algorithms \ref{algorithm:L3} and \ref{algorithm:L4} respectively as follows
	\begin{equation}\label{def:F,G}
		F_{t}^{\dag,-i}=\begin{cases}\widehat{V}_{1}^{t,\dag}\big(x_{1};\tilde{R}^{-i}\big) &\text{if $\zeta_2=\texttt{OPT}$}\\
			\check{V}_{1}^{t,\dag}\big(x_{1};\tilde{R}^{-i}\big) &\text{if $\zeta_2=\texttt{PES}$},
		\end{cases}   
		\qquad 
		G_{t}^{\dag,-i}=\begin{cases}\widehat{V}_{1}^{t,\tilde{\pi}_t^{\dag i}}\big(x_{1};\tilde{R}^{-i}\big) &\text{if $\zeta_3=\texttt{OPT}$}\\
			\check{V}_{1}^{t,\tilde{\pi}_t^{\dag i}}\big(x_{1};\tilde{R}^{-i}\big) &\text{if $\zeta_3=\texttt{PES}$}.
		\end{cases}   
	\end{equation}
	For all $t > K$, according to the definition of $u_{i t}$, under the current reward setting, we have
	\begin{equation}\label{equa:IR ETC decomp}
		\begin{aligned}
			u_{i t} &=V_{1}^{\tilde{\pi}_t^{\dag i}}(x_{1};r_{i})- p_{it}^{\dagger} \\
			&=V_{1}^{\tilde{\pi}_t^{\dag i}}\big(x_1;r_{i }\big)-F_{t}^{\dag, -i}+G_{t}^{\dag, -i}\\
			&=\big[V_{1}^{\tilde{\pi}_t^{\dag i}}\big(x_1;r_{i }\big)+V_{1}^{\tilde{\pi}_t^{\dag i}}\big(x_1;\tilde{R }^{-i}\big)-F_{t}^{\dag, -i}]+\big[G_{t}^{\dag, -i}-V_{1}^{\tilde{\pi}_t^{\dag i}}\big(x_1;\tilde{R }^{-i}\big)\big]\\
			&=\underbrace{\big[V_{1}^{\tilde{\pi}_t^{\dag i}}\big(x_{1};r_{i }+\tilde{R}^{-i}\big)-F_{t}^{\dag, -i}\big]}_{\displaystyle{ \rm{(i)}}}+\underbrace{\big[G_{t}^{\dag, -i}-V_{1}^{\tilde{\pi}_t^{\dag i}}\big(x_{1};\tilde{R}^{-i}\big)\big]}_{\displaystyle{ \rm{(ii)}}},
		\end{aligned}
	\end{equation}
	where $p_{it}^\dag = F_{t}^{\dag, -i}-G_{t}^{\dag, -i}$. For term (i) in Equation \eqref{equa:IR ETC decomp}, by the definition of $V_1^*(x_1,r):=\max_\pi V_1^\pi(x_1,r)$ for any $r$, we have
	\begin{equation*}
		\begin{aligned}
			\displaystyle{ \rm{(i)}}&\geq V_{1}^{\tilde{\pi}_t^{\dag i}}\big(x_{1};r_{i }+\tilde{R}^{-i}\big)-\widehat{V}_{1}^{t,\dag}\big(x_{1};\tilde{R}^{-i}\big)\\
			&=\underbrace{\big[V_{1}^{*}(x_1;r_{i }+\tilde{R}^{-i})-V_{1}^{*}(x_1;\tilde{R}^{-i})\big]}_{\text{(i.a)}}+\underbrace{\big[V_{1}^{\tilde{\pi}_t^{\dag i}}\big(x_{1};r_{i }+\tilde{R}^{-i}\big)-V_{1}^{*}(x_1;r_{i }+\tilde{R}^{-i})\big]}_{\text{(i.b)}}\\
			&\quad +\underbrace{\big[{V_{1}^{*}}(x_1;\tilde{R}^{-i})-\widehat{V}_{1}^{t,\dag}\big(x_{1};\tilde{R}^{-i}\big)\big]}_{\text{(i.c)}},
		\end{aligned}
	\end{equation*}
	where the first inequality stems from the fact that $F_{t}^{\dagger,-i}$ is always at most $\widehat{V}_{1}^{t,\dag}\big(x_{1};\tilde{R}^{-i}\big)$ regardless of the choice of $\zeta_2$ shown in Equation \eqref{def:F,G}. 
	For $\text{(i.a)}$, 
	we have that
	\begin{equation*}
		\text{(i.a)}= \max_{\pi}V_{1}^{\pi}(x_1;r_{i }+\tilde{R}^{-i}) - \max_{\pi}V_{1}^{\pi}(x_1;\tilde{R}^{-i}). 
	\end{equation*}
	Note that for any $\pi$, we have $V_{1}^{\pi}(x_1;r_{i }+\tilde{R}^{-i}) \geq V_{1}^{\pi}(x_1;\tilde{R}^{-i})$ since $r_{i } \geq 0$, which implies that $\max_\pi V_{1}^{\pi}(x_1;r_{i }+\tilde{R}^{-i}) \geq V_{1}^{\pi}(x_1;\tilde{R}^{-i})$ holds for any $\pi$. Taking maximum on the right-hand side further gives $\max_\pi V_{1}^{\pi}(x_1;r_{i }+\tilde{R}^{-i}) \geq \max_\pi V_{1}^{\pi}(x_1;\tilde{R}^{-i})$.
	We then have that $\text{(i.a)} \geq 0$. 
	Moreover, $\text{(i.b)}$ is the suboptimality of policy $\tilde{\pi}_t^{\dag i}$ and $\text{(i.c)}$ is the estimation error of $V_{1}^{*}(x_1;\tilde{R}^{-i})$ by $\widehat{V}_{1}^{t,\dag}\big(x_{1};\tilde{R}^{-i}\big)$, which can be bounded below by $-2\hat{c}(n+R_{\max}) \sqrt{d^3H^6\iota /K}$ respectively invoking Lemma \ref{lemma:basic_lemma}. We can then bound term $\text{(i)}$ from below by $-4\hat{c}(n+R_{\max})\sqrt{d^3H^6\iota/K}$. 
	
	For term (ii) in Equation \eqref{equa:IR ETC decomp}, observe that $G_t^{\dagger,-i}$ is always at least $\check{V}_{1}^{t,\tilde{\pi}_t^{\dag i}}\big(x_{1};\tilde{R}^{-i}\big)$ regardless of the choice of $\zeta_3$ shown in Equation \eqref{def:F,G} and thus we have by Lemma~\ref{lemma:basic_lemma} that
	\begin{equation*}
		\begin{aligned}
			\displaystyle{ \rm{(ii)}}&\geq \check{V}_{1}^{t,\tilde{\pi}_t^{\dag i}}\big(x_{1};\tilde{R}^{-i}\big)-V_{1}^{\tilde{\pi}_t^{\dag i}}\big(x_{1};\tilde{R}^{-i}\big)\geq - 2\hat{c}(n+R_{\max})\sqrt{d^{3}H^{6}\iota/K},
		\end{aligned}
	\end{equation*}
	for some absolute constant $\hat{c}$. Summing (i) and (ii) from $t = 1$ to $T$, we get 
	\begin{equation*}
		U_{i T}\geq\sum_{t=K+1}^{T}u_{i t}\geq-6\hat{c} (n+R_{\max})\sqrt{d^3H^6\iota/K} ,
	\end{equation*}
	Setting $K=d H^{4/3}\iota^{1/3}T^{2/3}$ in the above inequality, we further get,
	\begin{equation*}
		U_{i T}\geq-6\hat{c} (n+R_{\max})d H^{7/3}\iota^{1/3}T^{2/3}
	\end{equation*}
	which implies the mechanism we learned is $6\hat{c} (n+R_{\max})d H^{7/3}\iota^{1/3}T^{2/3}$-approximately individually rational.

\paragraph{Truthfulness.} 
We consider two cases for our proof of truthfulness: (1) agent $i$ reports truthfully, and others may report untruthfully (2) all agents may report untruthfully. Then we denote by $r_i$ the truthful reward and $\tilde{r}_i$ the potentially untruthful reward. For case (1), we adopt the same notations $F^{\dagger, -i}_t, G^{\dagger, -i}_t$, $\tilde{\pi}_t^{\dag i}$, and $u_{it}=V_{1}^{\tilde{\pi}_t^{\dag i}}(x_{1};r_{i})-p_{it}^\dag$ as in the above proof of individual rationality. For case (2), we let $\tilde{\pi}_t^{\ddag}$ be the learned policy for the seller under the reward $\tilde{R} := r_0 + \sum_{i=1}^n \tilde{r}_i$ in Algorithm \ref{algorithm:LMVL}, $F^{\ddagger, -i}_t$ and $G^{\ddagger, -i}_t$ be the associated $F$ and $G$ functions, and $\tilde{u}_{it}=V_{1}^{\tilde{\pi}_t^{\ddag}}(x_{1};r_{i})-p_{it}^\ddag$ with $p_{it}^\ddag = F^{\ddagger, -i}_t-G^{\ddagger, -i}_t$ generated by Algorithms \ref{algorithm:L3} and \ref{algorithm:L4} respectively.
Let $\Tilde{U}_{iT} = \sum_{t = 1}^T \tilde{u}_{it}$ and $U_{iT} = \sum_{t = 1}^T u_{it}$. The surplus in utility the agent gains from untruthful reporting is then
\begin{equation}
	\label{equa:difference sum utilities}
	\tilde{U}_{i T}-U_{i T}=\sum_{t=1}^{T}\rbr{\tilde{u}_{i t}-u_{i t}}.
\end{equation}
We decompose the summation in terms of the exploration and exploitation phases. When $t \leq K$, the agents are not charged any price, and then $r_i \geq 0$ ensures $u_{it} \geq 0$. We then have
\[
\tilde{u}_{it} - u_{it} \leq \tilde{u}_{it} \leq \max_{\pi}V^\pi_1(x_1; r_i) \leq H,
\] 
where the second inequality uses the fact that the price is 0.

We now consider the case when $t > K$. We explicitly define $F^{\ddag,-i}_t$ and $G^{\ddag,-i}_t$ as follows   
\begin{equation}\label{def:F,G,2}
	F_{t}^{\ddag,-i}=\begin{cases}\widehat{V}_{1}^{t,\ddag}\big(x_{1};\tilde{R}^{-i}\big) &\text{if $\zeta_2=\texttt{OPT}$}\\
		\check{V}_{1}^{t,\ddag}\big(x_{1};\tilde{R}^{-i}\big) &\text{if $\zeta_2=\texttt{PES}$},
	\end{cases}   
	\qquad 
	G_{t}^{\dag,-i}=\begin{cases}\widehat{V}_{1}^{t,\tilde{\pi}_t^{\ddag}}\big(x_{1};\tilde{R}^{-i}\big) &\text{if $\zeta_3=\texttt{OPT}$}\\
		\check{V}_{1}^{t,\tilde{\pi}_t^{\ddag}}\big(x_{1};\tilde{R}^{-i}\big) &\text{if $\zeta_3=\texttt{PES}$},
	\end{cases}   
\end{equation}
where the value functions are generated by Algorithms \ref{algorithm:L3} and \ref{algorithm:L4} respectively based on the untruthfully reported rewards by all agents.

For any $t > K$, we have
\begin{equation*}
	\tilde{u}_{i t}-u_{i t} =\big[V_{1}^{\tilde{\pi}_t^{\ddag}}(x_{1};r_{i})-F_{t}^{\ddag,-i}+G_{t}^{\ddag,-i}\big]-\big[V_{1}^{\tilde{\pi}_t^{\dag i}}(x_{1};r_{i})-F_{t}^{\dag,-i}+G_{t}^{\dag,-i}\big].
\end{equation*}
We first show that $F^{\dagger, -i}_t = F^{\ddagger, -i}_t$. Recall that when $\zeta_1 = \texttt{ETC}$, both $F^{\dagger, -i}_t$ and $F^{\ddagger, -i}_t$ are calculated using only data collected during the exploration phase. As the data collection policy is given by a reward-free exploration algorithm, namely Algorithm~\ref{algorithm:L1}, the trajectories collected remain the same whether agent $i$ is truthful or not. Additionally, both $F^{\dagger, -i}_t$ and $F^{\ddagger, -i}_t$ are given by Algorithm~\ref{algorithm:L3}, which only uses the rewards reported by other agents. In other words, the input data used to calculate $F^{\dagger, -i}_t$ and $F^{\ddagger, -i}_t$ are exactly the same, irregardless of the truthfulness of agent $i$. Conditionally on the $K$ trajectories collected during the exploration phase, the two functions $F^{\dagger, -i}_t$ and $F^{\ddagger, -i}_t$ equal to each other and cancel out. 
We then obtain that for all $t > K$,
\begin{align*}
	&\tilde{u}_{it} - u_{it}\\
	&=V_{1}^{\tilde{\pi}_t^{\ddag}}(x_{1};r_{i})+G_{t}^{\ddagger,-i}-V_{1}^{\tilde{\pi}_t^{\dag i}}(x_{1};r_{i})-G_{t}^{\dagger,-i}\\
	&=\underbrace{\big[V_{1}^{\tilde{\pi}_t^{\ddag}}\big(x_{1};r_{i}+\tilde{R}^{-i}\big)-V_{1}^*\big(x_{1};r_{i}+\tilde{R}^{-i}\big)\big]}_{\text{(i)}}+\underbrace{\big[V_{1}^*\big(x_{1};r_{i}+\tilde{R}^{-i}\big)-V_{1}^{\tilde{\pi}_t^{\dag i}}\big(x_{1};r_{i}+\tilde{R}^{-i}\big)\big]}_{\text{(ii)}}\\
	&\quad +\underbrace{\big[{G}_{t}^{\ddagger,-i}-V_{1}^{\tilde{\pi}_t^{\ddag}}\big(x_{1};\tilde{R}^{-i}\big)\big]}_{\text{(iii)}}+\underbrace{\big[V_{1}^{\tilde{\pi}_t^{\dag i}}\big(x_{1};\tilde{R}^{-i}\big)-G_{t}^{\dagger,-i}\big]}_{\text{(iv)}}.
\end{align*} 
Here, term (i) $\leq 0$ is due to the 
definition of $V_1^*(x_1;r_{i}+\tilde{R}^{-i}) = \max_{\pi} V_1^{\pi}(x_1;r_{i}+\tilde{R}^{-i})$. Term (ii) is the suboptimality of policy $\tilde{\pi}_t^{\dag i}$, term (iii) and term (iv) are policy evaluation errors for policy $\tilde{\pi}_t^{\ddag}$ and $\tilde{\pi}_t^{\dag i}$. Using Lemma~\ref{lemma:basic_lemma}, term (ii) is upper bounded by $2\hat{c}(n+R_{\max})\sqrt{d^3H^6\iota/K}$. We then consider terms (iii) and (iv). When $\zeta_3 = \texttt{OPT}$, we have (iii) $\leq 2\hat{c} (n+R_{\max})\sqrt{d^3H^6\iota/K}$ while (iv) $\leq 0$. Similarly, we have $\textrm{(iii)} \leq 0$ and $\textrm{(iv)} \leq 2\hat{c} \sqrt{d^3H^6\iota/K}$ when $\zeta_3 = \texttt{PES}$. In summary, regardless of the choices for $\zeta_2, \zeta_3$, we always have for all $i, t$
\begin{align*}
	\tilde{u}_{i t}-u_{i t}
	\leq \begin{cases}
		H &\textrm{ if $t \in [K]$}\\
		4\hat{c}(n+R_{\max})\sqrt{d^3H^6\iota/K} &\textrm{ if $t > K$}.
	\end{cases}
\end{align*}
Now we have obtained the upper bounds of $\tilde{u}_{it} - u_{it}^{t - 1}$ for both when $t \in [K]$ and when $t > K$. Summing $\tilde{u}_{it} - u_{it}$ from $t = 1$ to $T$, we get
\begin{equation*}
	\Tilde{U}_{iT}-U_{i T}\leq HK+4\hat{c}(n+R_{\max})\sqrt{d^3H^6\iota/K}(T-K). 
\end{equation*}
Setting $K=d H^{4/3}\iota^{1/3}T^{2/3}$ in the above inequality, we further obtain
\begin{equation*}
	\Tilde{U}_{iT}-U_{i T}\leq d H^{7/3}\iota^{1/3}T^{2/3}+ 4\hat{c}(n+R_{\max})d H^{7/3}\iota^{1/3}T^{2/3},
\end{equation*}
which implies that the learned mechanism is $\big(1+ 4\hat{c}(n+R_{\max})\big)d H^{7/3}\iota^{1/3}T^{2/3}$-approximately truthful. This completes the proof.
\end{proof}

\subsection{Proof of Theorem \ref{theorem:OPT}}\label{subsec:proof_of_theorem_ewc}

\begin{proof} We now prove each result separately in Theorem \ref{theorem:OPT}. The concentration inequalities for the proof of Theorem~\ref{theorem:OPT} jointly hold with probability at least $1-\delta$. We ignore the detailed description of probabilities in our proof for conciseness.

\paragraph{Welfare Regret.} When setting $\zeta_{1}=\texttt{EWC}$, we can decompose the regret into two parts, the regret incurred in the exploration phase and the regret incurred in the exploitation phase as 
\begin{equation*}
\mathrm{Reg}_{T}^{W}=\sum_{t=1}^{K}\mathrm{reg}_{t}^{W}+\sum_{t=K+1}^{T}\mathrm{reg}_{t}^{W}.
\end{equation*}
Then we can bound the first summation as $\sum_{t=1}^{K}\mathrm{reg}_{t}^{W}\leq H(n + R_{\max})K$ using the same technique for obtaining Equation~\eqref{equa:ETC wel explore}. For the second part, 
we have
\begin{equation*}
\begin{aligned}
	\sum_{t=K+1}^{T}\mathrm{reg}_{t}^{W}&=\sum_{t=K+1}^{T}\big[V_{1}^{\pi_*}(x_1;R)-V_{1}^{\widehat{\pi}^{t}}(x_1;R)\big].
\end{aligned}
\end{equation*} Notice that during the exploitation phase, the welfare regret of Algorithm~\ref{algorithm:LMVL} when $\zeta_1 = \texttt{EWC}$ is the well-studied regret bound for LSVI-UCB, derived in~\citet{2019Provably}. 
For integrity, we sketch out the proof below and refer interested readers to the detailed proofs in~\citet{2019Provably}.

Following standard decomposition (see Lemmas B.5 and B.6 in \citet{2019Provably}, for instance), we have
\begin{equation}\label{equa:EWC_decompose}
\begin{aligned}
	\sum_{t=K+1}^{T}\mathrm{reg}_{t}^{W}&=\sum_{t=K+1}^{T}\big[V_{1}^{\pi_*}(x_1;R)-V_{1}^{\widehat{\pi}^{t}}(x_1;R)\big] \leq\sum_{t=K+1}^{T}\big[V_{1}^{t}(x_1;R)-V_{1}^{\widehat{\pi}^{t}}(x_1;R)\big]\\
	&\leq\underbrace{\sum_{t=K+1}^{T}\sum_{h=1}^{H}\big(\EE\big[\xi_{h}^{t}\biggiven x_{h-1}^{t},a_{h-1}^{t}\big]-\xi_{h}^{t}\big)}_{\text{(i)}}\\
	&\qquad +2\beta\underbrace{\sum_{t=K+1}^{T}\sum_{h=1}^{H}\sqrt{\big(\phi(x_h^t, a_h^t)\big)^{\top}\big(\Lambda_{h}^{t}\big)^{-1}\big(\phi(x_h^t, a_h^t)\big)}}_{\text{(ii)}},
\end{aligned}
\end{equation}
where $\xi_{h}^{t}=V_{h}^{t}\big(x_{h}^{t};R\big)-V_{h}^{\widehat{\pi}_{t}}\big(x_{h}^{t};R\big)$ .
Then, we bound terms (i) and (ii) in Equation \eqref{equa:EWC_decompose} respectively. 
For term (i), since the computation of $\widehat{V}_{h}^{t}$ does not use the new observation $x_{h}^{t}$ at rounds $t$, the terms in term (i) is a martingale difference sequence bounded by $2(n+R_{max})H$. Then we can bound it by Azuma-Hoeffding inequality and get an $\cO\big((n+R_{\max})H \iota T^{1/2}\big)$ upper bound for term (i) in Equation \eqref{equa:EWC_decompose}. We provide the details as follows: for any $\nu>0$, we have
\begin{equation*}
\PP\bigg(\sum_{t=K+1}^{T}\sum_{h=1}^{H}\big(\EE\big[\xi_{h}^{t}\biggiven x_{h-1}^{t},a_{h-1}^{t}\big]-\xi_{h}^{t}\big) \geq \nu\bigg)\leq \exp{\bigg\{\frac{-\nu^2}{2(n+R_{\max})^2 H^2(T-K)}\bigg\}}.
\end{equation*}
Hence, with high probability, we have
\begin{equation}\label{equation:EWC T1}
\begin{aligned}
	\sum_{t=K+1}^{T}\sum_{h=1}^{H}\big(\EE\big[\xi_{h}^{t}\biggiven x_{h-1}^{t},a_{h-1}^{t}\big]-\xi_{h}^{t}\big)&\leq\sqrt{2(n+R_{\max})^2 H^2 (T-K)\log(2/\delta)}\\
	&\leq 2 (n+R_{\max})\sqrt{  H^2   (T-K)\iota},
\end{aligned}
\end{equation}
where $\iota=\log(36n d H T/\delta)$.
For term (ii), we can bound it using Lemma \ref{lemma:telescope} and Cauchy-Schwarz inequality,
\begin{align}
\sum_{t=K+1}^{T}\sum_{h=1}^{H}\sqrt{\big(\phi(x_h^t, a_h^t)\big)^{\top}\big(\Lambda_{h}^{t}\big)\big(\phi(x_h^t, a_h^t)\big)}&\leq\sum_{t=K+1}^{T}\sum_{h=1}^{H}\sqrt{\big(\phi(x_h^t, a_h^t)\big)^{\top}\big(\tilde{\Lambda}_{h}^{t}\big)^{-1}\big(\phi(x_h^t, a_h^t)\big)}\nonumber\\
&\leq\sum_{h=1}^{H}\bigg[\sum_{t=K+1}^{T}\phi(x_h^t,a_h^t)^\top   (\tilde{\Lambda}_h^t) ^{-1}
\phi(x_h^{t},a_h^t)\bigg]^{1/2}\nonumber\\
&\leq 2\sqrt{2d H^2 (T-K)\iota}, \label{equation:EWC T2}
\end{align}
where $\tilde{\Lambda}_{h}^{t}=\sum_{\tau=K+1}^{t-1}\phi(x_{h}^{\tau},a_{h}^{\tau})\phi(x_{h}^{\tau},a_{h}^{\tau})^{\top}+\lambda I$ is the design matrix only using the data in the exploitation phase. The first step is due to $\tilde{\Lambda}_{h}^{t}\preceq\Lambda_{h}^{t}$, the second step is by Cauchy-Schwartz inequality, and the last step uses the elliptical potential lemma in \citet{abbasi2011improved}. 
Combining Equations \eqref{equa:EWC_decompose}, \eqref{equation:EWC T1} and \eqref{equation:EWC T2}, with the setting of $\beta=\hat{c} (n+R_{\max} )  d H\sqrt{\iota} \text{ where } \iota=\log(36n d H T/\delta)$, we have the following upper bound
\begin{align*}  \sum_{t=K+1}^{T}\mathrm{reg}_{t}^{W}&\leq 2 (n+R_{\max})\sqrt{  H^2   (T-K)\iota}+2\beta\sqrt{2d H^2 (T-K)\iota}\\
&\leq 2 (n+R_{\max})\sqrt{  H^2   (T-K)\iota}+4\hat{c}(n+R_{\max})\sqrt{d^3 H^4  (T-K)\iota^2}\\
&\leq 6\hat{c}(n+R_{\max})\sqrt{d^3 H^4  (T-K)\iota^2}.
\end{align*}
Combining the above inequality with the upper bound for $\sum_{t=1}^{K}\mathrm{reg}_{t}^{W}$, we have the upper bound of the welfare regret as
\begin{equation}\label{Equation:EWC regret1}
\mathrm{Reg}_{T}^{W}\leq (n+R_{\max})H K+6\hat{c}(n+R_{\max})\sqrt{d^3 H^4  (T-K)\iota^2},
\end{equation}
where the value of $K$ will be determined by jointly considering the upper bounds of $n\mathrm{Reg}_T^W$, $\mathrm{Reg}^{\sharp}_T$, and $\mathrm{Reg}_{0T}$.

\paragraph{Agent Regret.} For agent $i$'s regret incurred during the exploration phase, we know from Section~\ref{subsec:proof_of_theorem_etc} that it is bounded as $\sum_{t = 1}^K \textrm{reg}_{it} \leq H K$. We now focus on when $t > K$. According Equation~\eqref{equa:r_it decomp}, we have that
\begin{equation}
\begin{aligned}
	\mathrm{reg}_{i t}&=u_{i *}-u_{i t}\\
	&=\underbrace{\big[V_{1}^{\pi_*}(x_{1};R)-V_{1}^{\widehat{\pi}^{t}}(x_{1};R)\big]}_{\displaystyle\rm{(i)}}+\underbrace{\big[F_{t}^{-i}-V_{1}^{\pi^{-i}_{*}}(x_{1};R^{-i})\big]}_{\displaystyle \rm{(ii)}}+\underbrace{\big[V_{1}^{\widehat{\pi}^{t}}(x_{1};R^{-i})-G_{t}^{-i}\big]}_{\displaystyle \rm{(iii)}}.
\end{aligned} 
\end{equation} 
Term (i) is the welfare regret, term (ii) is the function evaluation and policy estimation errors for $F_t^{-i}$, and term (iii) is the function evaluation error for $G_t^{-i}$. Recalling that the welfare regret bound above, we know that the summation of (i) from $t = K + 1$ to $T$ can be bounded as 
\[\sum_{t = K + 1}^T\big[V_{1}^{\pi_*}(x_{1};R)-V_{1}^{\widehat{\pi}^{t}}(x_{1};R)\big]\leq 6\hat{c}(n+R_{\max})\sqrt{d^3 H^4  (T-K)\iota^2}.\]
Our bounds for terms (ii) and (iii) use similar techniques for the case when $\zeta_1 = \texttt{ETC}$. Let $\widehat{\pi}_t^{-i}$ be the fictitious policy returned by Algorithm~\ref{algorithm:L3} when we compute $F_t^{-i}$. We obtain that
\[
\textrm{(ii)} = \widehat{V}_1^{t, \widehat{\pi}_t^{-i}} - V_{1}^{\pi^{-i}_{*}}(x_{1};R^{-i}) \leq 2\hat{c} (n+R_{\max})\sqrt{d^3 H^6 \iota/K},
\] 
when $\zeta_2 = \texttt{OPT}$, using Lemma~\ref{lemma:basic_lemma}. Similarly, we know $\textrm{(ii)} \leq 0$ when $\zeta_2 = \texttt{PES}$. 

Finally, by Lemma~\ref{lemma:basic_lemma}, we know $\textrm{(iii)} \leq 2\hat{c}(n+R_{\max})\sqrt{d^3H^6\iota/K}$ when $\zeta_3 = \texttt{PES}$ and $\textrm{(iii)} \leq 0$ when $\zeta_3 = \texttt{OPT}$. Combining the bounds for terms (i), (ii), and (iii) in both phases, we have the upper bound of the agent regret $\mathrm{Reg}_{i T}$ as follows:

If $(\zeta_{2},\zeta_{3})=(\texttt{PES},\texttt{OPT})$, then
\begin{align}
\mathrm{Reg}_{i T} \leq H K+6\hat{c}(n+R_{\max})\sqrt{d^3 H^4  (T-K)\iota^2}. \label{Equation:EWC regret2-1}
\end{align}

If $(\zeta_{2},\zeta_{3})=(\texttt{OPT},\texttt{PES})$, then
\begin{align}
\hspace{-0.36cm}\mathrm{Reg}_{i T} \leq H K+6\hat{c}(n+R_{\max})\sqrt{d^3 H^4  (T-K)\iota^2}+4\hat{c}(n+R_{\max})\sqrt{d^3H^6\iota/K}(T-K).  \label{Equation:EWC regret2-2}
\end{align}



\paragraph{Seller Regret.} Similar to our proof of agent regret, from Section~\ref{subsec:proof_of_theorem_etc}, we first have  
\[
\sum_{t = 1}^K \textrm{reg}_{0t} \leq H(n + R_{\max})K.
\] 
In addition, the exploration regret can be decomposed as
\begin{align*}
\mathrm{reg}_{0 t}&=u_{0 *}-u_{0 t}\\
&=(n-1)\big[{V_{1}^{\widehat{\pi}^{t}}(x_{1};R)-V_{1}^{\pi_*}(x_{1};R)}\big]+\sum_{i=1}^{n}\big[{V_{1}^{\pi^{-i}_{*}}(x_{1};R^{-i})-F_{t}^{-i}}\big]\\
&\quad +\sum_{i=1}^{n}\big[{G_{t}^{-i}-V^{\widehat{\pi}^{t}}(x_{1};R^{-i})}\big]\\
&\leq \sum_{i=1}^{n}\underbrace{\big[V_{1}^{\pi^{-i}_{*}}(x_{1};R^{-i})-F_{t}^{-i}\big]}_{\displaystyle{\rm{(i)}}}+\sum_{i=1}^{n}\underbrace{\big[G_{t}^{-i}-V^{\widehat{\pi}^{t}}(x_{1};R^{-i})\big]}_{\displaystyle{ \rm{(ii)}}},
\end{align*} 
where 
the second equation directly follows the decomposition proven in Section~\ref{subsec:proof_of_theorem_etc} and 
the inequality comes from the definition of $\pi_*$, which is then used to eliminate the first term.

Similar to our proof for agent regret, invoking Lemma~\ref{lemma:basic_lemma} we immediately know that $\textrm{(i)} \leq 2\hat{c} (n+R_{\max})\sqrt{d^3H^6\iota/K}$ when $\zeta_2 = \texttt{PES}$, and $\textrm{(i)} \leq 0$ when $\zeta_2 = \texttt{OPT}$. Also by Lemma~\ref{lemma:basic_lemma}, we have $\textrm{(ii)} \leq 0$ when $\zeta_3 = \texttt{PES}$, and $\textrm{(ii)} \leq 2\hat{c} (n+R_{\max})\sqrt{d^3H^6\iota/K}$ when $\zeta_3 = \texttt{OPT}$.
Summing both (i) and (ii) over $i \in [n]$ and then summing the regrets incurred in both exploration and exploitation phases, we have the upper bound of the seller regret $\mathrm{Reg}_{0 T}$ as
\begin{equation}\label{Equation:EWC regret3}
\begin{cases}
(n+R_{\max})H K+4\hat{c} n(n+R_{\max})\sqrt{d^3H^6\iota/K}(T-K)
&\textrm{ if $(\zeta_{2},\zeta_{3})=(\texttt{PES},\texttt{OPT})$}\\
(n+R_{\max})H K
&\textrm{ if $(\zeta_{2},\zeta_{3})=(\texttt{OPT},\texttt{PES})$}.
\end{cases}
\end{equation}
\paragraph{Choice of $K$.} We determine the value of $K$ that can lead to a tight bound of $\max\{ n\mathrm{Reg}_T^W, \allowbreak \mathrm{Reg}^{\sharp}_T,\mathrm{Reg}_{0T}\}$, where $\mathrm{Reg}^{\sharp}_T = \sum_{i=1}^n \mathrm{Reg}_{i T}$. According to Equations \eqref{Equation:EWC regret1}, \eqref{Equation:EWC regret2-1}, \eqref{Equation:EWC regret2-2}, and \eqref{Equation:EWC regret3}, comparing the upper bounds of $n\mathrm{Reg}_T^W$, $\mathrm{Reg}^{\sharp}_T$, and $\mathrm{Reg}_{0T}$, we always have 
\begin{align*}
&\max\{ n\mathrm{Reg}_T^W,\mathrm{Reg}^{\sharp}_T,\mathrm{Reg}_{0T}\} \leq n(n+R_{\max})HK+6\hat{c}n(n+R_{\max})\sqrt{d^3 H^4  (T-K)\iota^2}\\
&\qquad\quad +4\hat{c}n(n+R_{\max})\sqrt{d^{3}H^{6}\iota/K}(T-K)\\
&\qquad\leq n(n+R_{\max})HK+ 6\hat{c}n(n+R_{\max})\sqrt{d^3 H^4  T\iota^2}+4\hat{c}n(n+R_{\max})\sqrt{d^{3}H^{6}\iota/K}T.
\end{align*}
Focusing on the factors of $H$, $n$, $d$, $T$, and $\iota$, we set $K = dH^{4/3}\iota^{1/3}T^{2/3}$, which can minimize the order of these factors in the above inequality, and obtain the bound 
\begin{align*}
\max\{ n\mathrm{Reg}_T^W,\mathrm{Reg}^{\sharp}_T,\mathrm{Reg}_{0T}\} = \cO \big(n(n+R_{\max})d H^{7/3}\iota^{1/3}T^{2/3}\big).
\end{align*}
Thus, plugging $K = dH^{4/3}\iota^{1/3}T^{2/3}$ into \eqref{Equation:EWC regret1}, we have 
\begin{equation*}
\mathrm{Reg}_{T}^{W}\leq(n+R_{\max}) ( d H^{7/3} \iota^{1/3} T^{2/3}+ 6\hat{c} d^{3/2} H^2 \iota T^{1/2}).
\end{equation*}
Plugging the value of $K$ into \eqref{Equation:EWC regret2-1} and \eqref{Equation:EWC regret2-2}, we have that $\mathrm{Reg}_{i T}$ can be bounded by
\begin{equation*}
\begin{cases}
d H^{7/3} \iota^{1/3} T^{2/3}+ 6\hat{c}(n+R_{\max}) d^{3/2} H^2 \iota T^{1/2}& \text{if $(\zeta_{2},\zeta_{3})=(\texttt{PES},\texttt{OPT})$}\\
(1+4\hat{c}(n+R_{\max}))dH^{7/3} \iota^{1/3} T^{2/3}+ 6\hat{c}(n+R_{\max}) d^{3/2} H^2 \iota T^{1/2}& \text{if $(\zeta_{2},\zeta_{3})=(\texttt{OPT},\texttt{PES})$}.
\end{cases}
\end{equation*}
Plugging the value of $K$ into \eqref{Equation:EWC regret3}, we obtain 
\begin{equation*}
\mathrm{Reg}_{0 T}\leq\begin{cases}
(1+4\hat{c}n)(n+R_{\max})d H^{7/3}\iota^{1/3} T^{2/3}& \text{if $(\zeta_{2},\zeta_{3})=(\texttt{PES},\texttt{OPT})$}\\
(n+R_{\max})d H^{7/3}\iota^{1/3} T^{2/3}& \text{if $(\zeta_{2},\zeta_{3})=(\texttt{OPT},\texttt{PES})$}.
\end{cases}
\end{equation*}
This completes the proof of the upper bounds of the welfare regret, the agent regret, and the seller regret.

\paragraph{Individual Rationality.}
We assume that agent $i$ reports truthfully according to the reward function $\tilde{r}_i$ and other agents may report untruthfully according to the reward function $\tilde{r}_j$ for $j \neq i$. Then, we adopt the same definitions of $\tilde{\pi}_t^{\dag i}$, $\tilde{R}^{-i}$, $F_t^{\dag,-i}$, and $G_t^{\dag,-i}$ as in the proof of individual rationality in Section \ref{subsec:proof_of_theorem_etc}.

Here the agents are not charged during the exploration phase, and $r_i \geq 0$ ensures that $u_{it} \geq 0$ for all $t \in [K]$. Recalling Equation~\eqref{equa:IR ETC decomp}, we have the following decomposition for $t > K$,
\[
u_{it}=\underbrace{\big[V_{1}^{\tilde{\pi}_t^{\dag i}}\big(x_{1};r_{i }+\tilde{R}^{-i}\big)-F_{t}^{\dag, -i}\big]}_{\displaystyle{ \rm{(i)}}}+\underbrace{\big[G_{t}^{\dag, -i}-V_{1}^{\tilde{\pi}_t^{\dag i}}\big(x_{1};\tilde{R}^{-i}\big)\big]}_{\displaystyle{ \rm{(ii)}}}.
\] Moreover, in the proof of individual rationality in Section~\ref{subsec:proof_of_theorem_etc}, we have shown that
\begin{equation*}
\begin{aligned}
\displaystyle{ \rm{(i)}}&\geq \big[V_{1}^{\tilde{\pi}_t^{\dag i}}\big(x_{1};r_{i }+\tilde{R}^{-i}\big)-V_{1}^{*}(x_1;r_{i }+\tilde{R}^{-i})\big]+\big[V_{1}^{*}(x_1;\tilde{R}^{-i})-\widehat{V}_{1}^{t,\dag}\big(x_{1};\tilde{R}^{-i}\big)\big]
\end{aligned}
\end{equation*} and
\[
\displaystyle{ \rm{(ii)}}\geq \check{V}_{1}^{\tilde{\pi}_t^{\dag i}}\big(x_{1};\tilde{R}^{-i}\big)-V_{1}^{\tilde{\pi}_t^{\dag i}}\big(x_{1};\tilde{R}^{-i}\big),
\]
according to the definitions of $F_{t}^{\dag, -i}$ and $G_{t}^{\dag, -i}$. 
Applying Lemma~\ref{lemma:basic_lemma}, we have that 
\[
\textrm{(i)} \geq -4\hat{c} (n+R_{\max})\sqrt{d^3H^6\iota/K}, \quad \textrm{(ii)} \geq -2\hat{c} (n+R_{\max})\sqrt{d^3H^6\iota/K}.
\]
Summing (i) and (ii) from $t = 1$ to $T$, we get 
\begin{equation*}
U_{i T}\geq\sum_{t=K+1}^{T}u_{i t}\geq-6\hat{c} (n+R_{\max})\sqrt{d^3H^6\iota/K} ,
\end{equation*}
Setting $K=d H^{4/3}\iota^{1/3}T^{2/3}$ in the above inequality, we further get,
\begin{equation*}
U_{i T}\geq-6\hat{c} (n+R_{\max})d H^{7/3}\iota^{1/3}T^{2/3}
\end{equation*}
which implies the mechanism we learned is $6\hat{c} (n+R_{\max})d H^{7/3}\iota^{1/3}T^{2/3}$-approximately individually rational.

\paragraph{Truthfulness:} The proof for truthfulness when $\zeta_1 = \texttt{EWC}$ significantly differs from the case when $\zeta_1 = \texttt{ETC}$. At a high level, when $\zeta_1 = \texttt{ETC}$, we use the fact that the data used to calculate $F$ is collected entirely during the exploration phase and is not affected by agent $i$ potentially reporting untruthfully, and hence $F_t^{\ddagger, -i}$ and $F_t^{\dagger, -i}$ cancel out. Unfortunately, 
when $\zeta_1 = \texttt{EWC}$, $F$'s computation is dependent on the untruthful behavior of agent $i$. The trajectories collected during exploitation are used for computing $F$. The policy used for collecting these trajectories is learned using the agent $i$'s report and thus is affected by the agent's untruthfulness. In this way, different from the proof of truthfulness in Section~\ref{subsec:proof_of_theorem_etc} where $F_t^{\dagger, -i}=F_t^{\ddagger, -i}$, the following proof also bounds the difference between $F_t^{\dagger, -i}$ and $F_t^{\ddagger, -i}$. 
We adopt the same notations as in the proof of truthfulness in Section \ref{subsec:proof_of_theorem_etc}.

We first decompose Equation \eqref{equa:difference sum utilities} in terms of the exploration and exploitation phases. When $t \leq K$, the agents are not charged any price, and then $r_i \geq 0$ ensures $u_{it} \geq 0$. We thus have
\[
\tilde{u}_{it} - u_{it} \leq \tilde{u}_{it} \leq \max_{\pi}V^\pi_1(x_1; r_i) \leq H,
\] 
where the second inequality uses the fact that the price is 0. 

For $t > K$, the utility an agent gains from untruthful reporting, regardless of other agents' truthfulness, can be decomposed as follows
\begin{align*}
&\tilde{u}_{it} - u_{it}\\
&\quad =V_{1}^{\tilde{\pi}_t^{\ddag}}(x_{1};r_{i}) - F_t^{\ddagger, -i}+G_{t}^{\ddagger,-i}-V_{1}^{\tilde{\pi}_t^{\dag i}}(x_{1};r_{i}) + F_t^{\dagger, -i}-G_{t}^{\dagger,-i}\\
&\quad =\underbrace{\big[V_{1}^{\tilde{\pi}_t^{\ddag}}\big(x_{1};r_{i}+\tilde{R}^{-i}\big)-V_{1}^*\big(x_{1};r_{i}+\tilde{R}^{-i}\big)\big]}_{\text{(i)}}+\underbrace{\big[V_{1}^*\big(x_{1};r_{i}+\tilde{R}^{-i}\big)-V_{1}^{\tilde{\pi}_t^{\dag i}}\big(x_{1};r_{i}+\tilde{R}^{-i}\big)\big]}_{\text{(ii)}}\\
&\quad \quad  +\underbrace{\big[{G}_{t}^{\ddagger,-i}-V_{1}^{\tilde{\pi}_t^{\ddag}}\big(x_{1};\tilde{R}^{-i}\big)\big]}_{\text{(iii)}}+\underbrace{\big[V_{1}^{\tilde{\pi}_t^{\dag i}}\big(x_{1};\tilde{R}^{-i}\big)-G_{t}^{\dagger,-i}\big]}_{\text{(iv)}} + \underbrace{\big[F_t^{\dagger, -i} - F_t^{\ddagger, -i}\big]}_{\textrm{(v)}}.
\end{align*}
By Lemma~\ref{lemma:basic_lemma}, we know that regardless of the choice of $\zeta_3$, we have
\[
\textrm{(i)} + \textrm{(ii)} + \textrm{(iii)} + \textrm{(iv)} \leq 4\hat{c} (n+R_{\max})\sqrt{d^3 H^6 \iota/K}.
\] 
We focus on studying the upper bound of (v). By the definitions of $F$ function in Equations \eqref{def:F,G} and \eqref{def:F,G,2}, we know
\begin{equation*}
F_{t}^{\dag,-i}=\begin{cases}\widehat{V}_{1}^{t,\dag}\big(x_{1};\tilde{R}^{-i}\big) &\text{if $\zeta_2=\texttt{OPT}$}\\
\check{V}_{1}^{t,\dag}\big(x_{1};\tilde{R}^{-i}\big) &\text{if $\zeta_2=\texttt{PES}$},
\end{cases}   \qquad  F_{t}^{\ddagger,-i}=\begin{cases}\widehat{V}_{1}^{t,\ddagger}\big(x_{1};\tilde{R}^{-i}\big) &\text{if $\zeta_2=\texttt{OPT}$}\\
\check{V}_{1}^{t,\ddagger}\big(x_{1};\tilde{R}^{-i}\big) &\text{if $\zeta_2=\texttt{PES}$}.
\end{cases}  
\end{equation*}
Recall that $F_{t}^{\dag,-i}$ are generated by Algorithm \ref{algorithm:L3} using dataset $\cD$ collected with untruthful report from all the agents except agent $i$. On the other hand, $F_{t}^{\ddagger,-i}$ are generated by Algorithm \ref{algorithm:L3} with dataset $\cD$ collected with untruthful report from all the agents. Then, regardless of the choice of $\zeta_2$ when generating $F$ function, we have
\begin{equation*}
F_t^{\dagger, -i} - F_t^{\ddagger, -i}\leq \widehat{V}_1^{t, \dag}(x_1; \tilde{R}^{-i}) - \check{V}_{1}^{t,\ddagger,}(x_1; \tilde{R}^{-i}),
\end{equation*}
since it can be easily verify that $\widehat{V}_{1}^{t,\dag}\big(x_{1};\tilde{R}^{-i}\big)\geq \check{V}_{1}^{t,\dag}\big(x_{1};\tilde{R}^{-i}\big)$ and $\widehat{V}_{1}^{t,\ddagger}\big(x_{1};\tilde{R}^{-i}\big)\geq \check{V}_{1}^{t,\ddagger}\big(x_{1};\tilde{R}^{-i}\big)$, which thus implies that $F_t^{\dagger, -i}$ is at most $\widehat{V}_{1}^{t,\dag}\big(x_{1};\tilde{R}^{-i}\big)$ and $F_t^{\ddagger, -i}$ is at least $\check{V}_{1}^{t,\ddagger}\big(x_{1};\tilde{R}^{-i}\big)$ regardless of the choices of $\zeta_2,\zeta_3$.

When $\zeta_3 = \texttt{EWC}$, the trajectories collected during the exploitation phase may differ for the computations of $\widehat{V}_1^{t, \dag}(x_1; \tilde{R}^{-i})$ and 
$\check{V}_{1}^{t,\ddagger}(x_1; \tilde{R}^{-i})$, due to agent $i$'s untruthful reporting. Fortunately, as we can see from Lemma~\ref{lemma:basic_lemma}, the policy evaluation error can still be bounded: the reward-free exploration procedure in Algorithm~\ref{algorithm:L1} ensures that even when agent $i$ is not truthful and $\zeta_3 = \texttt{EWC}$, data collected during the exploration phase ensures a sufficient value function estimation.
With adding and subtracting $V_{1}^{*}(x_{1};\tilde{R}^{-i})$, we have
\begin{align*}
F_t^{\dagger, -i} - F_t^{\ddagger, -i}&\leq \underbrace{\Big(\widehat{V}_1^{t, \dag}(x_1; \tilde{R}^{-i}) - V_1^{*}(x_1; \tilde{R}^{-i})\Big)}_{\text{(i)}} + \underbrace{\Big(V_1^{*}(x_1; \tilde{R}^{-i}) - \check{V}_{1}^{t,\ddagger}(x_1; \tilde{R}^{-i})\Big)}_{\text{(ii)}}\\
&\leq 2\hat{c}(n+R_{\max})\sqrt{d^3 H^6 \iota/K}+2\hat{c}(n+R_{\max})\sqrt{d^3 H^6 \iota/K}\\
&=4\hat{c}(n+R_{\max})\sqrt{d^3 H^6 \iota/K}, 
\end{align*}
by apply Lemma~\ref{lemma:basic_lemma} to term $\text{(i)}$ and $\text{(ii)}$ and get $2\hat{c}\sqrt{d^3 H^6 \iota/K}$ upper bounds on both terms respectively. 
In summary, we have that for all $t > K$,
\[
\tilde{u}_{it} - u_{it} \leq 8\hat{c} (n+R_{\max})\sqrt{d^3H^6\iota/K}.
\] 
Summing $\tilde{u}_{it} - u_{it}$ from $t = 1$ to $T$, recalling the bound for all $t \in [K]$, we get
\begin{equation*}
\Tilde{U}_{iT}-U_{i T}\leq HK+8\hat{c} (n+R_{\max})\sqrt{d^3H^6\iota/K}.
\end{equation*}
Setting $K=d H^{4/3}\iota^{1/3}T^{2/3}$ in the above inequality, we further get 
\begin{equation*}
\Tilde{U}_{iT}-U_{i T}\leq (1 +8\hat{c}(n+R_{\max}))d H^{7/3}\iota^{1/3}T^{2/3},
\end{equation*}
implying the learned mechanism is $(1 +8\hat{c}(n+R_{\max}))d H^{7/3}\iota^{1/3}T^{2/3}$-approximately truthful.
\end{proof}

\section{Proof of Lemma \ref{lemma:basic_lemma}}\label{sec:proof_basic_lemma}
In this section, we present the detailed proof of Lemma \ref{lemma:basic_lemma}.
We first introduce several important notions e.g., Bellman operator and model evaluation error, and a supporting lemma with its proof in \ref{subsec:Preliminaries for the Lemmas}.
Then we provide the proof of Lemma \ref{lemma:basic_lemma} in Section \ref{subsec:basic_lemma}.

We note that bounding the errors in our setting is significantly different from the results in earlier works on reward-free exploration. Note that the planning subroutines described in Algorithms~\ref{algorithm:L3} and~\ref{algorithm:L4} use the collected rewards, rather than an arbitrary given reward function, to calculate the functions $F$ and $G$. As a result, the concentration analysis required to prove Lemma~\ref{lemma:basic_lemma}, as well as the decomposition used for the lemma, are all designed to cater to the dynamic mechanism design regime.

\subsection{Preliminaries for Proofs}\label{subsec:Preliminaries for the Lemmas}
We first define two operators to help characterize the estimation errors. For any function $f( ;\fR):\cS\rightarrow\RR$ with reward function $\fR$,

\begin{equation}
\label{equa:trans operator}
(\PP_{h}f)(x,a;\fR)=\EE[f(x_{h+1})|x_{h}=x,a_{h}=a],
\end{equation}
and the Bellman operator at step $h\in[H]$ as
\begin{equation}
\label{equa:Bellman operator}
\begin{aligned}
(\BB_{h}f)(x,a;\fR)&=\EE[\fR_{h}(x,a)+f(x_{h+1})|x_{h}=x,a_{h}=a]\\
&=\EE[\fR_{h}(x,a)|x_{h}=x,a_{h}=a]+(\PP_{h}f)(x,a).
\end{aligned}
\end{equation}
For estimated value functions $V_{h}^{t,\pi}$ and corresponding action-value functions $Q_{h}^{t,\pi}$. We define the model evaluation error with policy $\pi$ in episode  $t$ at each step $h\in[H]$ as
\begin{equation}
\label{equa:evaluation error}
\begin{aligned}
&\hat{\Delta}_{h}^{t,\pi}(x,a; )=(\BB_{h}\widehat{V}_{h+1}^{t,\pi})(x,a; )-\hat{Q}_{h}^{t,\pi}(x,a; ),\\
&\check{\Delta}_{h}^{t,\pi}(x,a; )=(\BB_{h}\check{V}_{h+1}^{t,\pi})(x,a; )-\check{Q}_{h}^{t,\pi}(x,a; ),
\end{aligned}
\end{equation}
for $\zeta_{3}=\texttt{OPT}$ and $\texttt{PES}$ respectively. In other words, ${\Delta}_{h}$ is the error in estimating the Bellman operator defined in Equation $\eqref{equa:Bellman operator}$, based on the dataset $\cD$ collected in Algorithm \ref{algorithm:L1}.

For clarity, we define the following events to quantify the uncertainty of the estimation of the Bellman operator $\BB_{h}$ in Algorithm \ref{algorithm:L3} and Algorithm \ref{algorithm:L4} with different hyperparameters.

\begin{definition}\label{def:event}  We define for all $t>K$ the event $\cE_t$ 
by requiring the following inequalities hold for all $(x,a) \in \cS\times\cA, h\in[H]$, and $(\fR, \pi) \in \{(R, \widehat{\pi}), (\Tilde{R}, \Tilde{\pi}^{\ddagger}_t)\} \cup \{(r_i + \Tilde{R}^{-i}, \Tilde{\pi}_t^{\dagger i}), (R^{-i}, *), (\Tilde{R}^{-i}, \dagger), (\Tilde{R}^{-i}, \ddagger), (R^{-i}, \widehat{\pi}^t), (\Tilde{R}^{-i}, \Tilde{\pi}^{\dagger i}_t), (\Tilde{R}^{-i}, \Tilde{\pi}_t^{\ddagger})\}_{i =1 }^n$, for each pair's associated $w$'s
\begin{align*}
&\big|\phi(x,a)^{\top}\hat{w}_{h}^{t,\pi}\big(\fR\big)-\BB_{h}\widehat{V}_{h+1}^{t,\pi}\big(x,a;\fR\big)\big|\leq u_{h}^{t}(x,a), \\
&\big|\phi(x,a)^{\top}\check{w}_{h}^{t,\pi}\big(\fR\big)-\BB_{h}\check{V}_{h+1}^{t,\pi}\big(x,a;\fR\big)\big|\leq u_{h}^{t}(x,a), 
\end{align*}
where the associated $w$'s are the learned parameters generated by Algorithm~\ref{algorithm:L3} if $(\fR, \pi) \in  \{(R, \widehat{\pi}^t), (\Tilde{R}, \Tilde{\pi}^{\ddagger}_t)\} \cup \{(r_i + \Tilde{R}^{-i}, \Tilde{\pi}_t^{\dagger i}), (R^{-i}, *), (\Tilde{R}^{-i}, \dagger), (\Tilde{R}^{-i}, \ddagger)\}_{i = 1}^n$, and the associated  $w$'s are learned parameters generated by Algorithm~\ref{algorithm:L4} if $(\fR, \pi) \in \{ (R^{-i}, \widehat{\pi}^t), (\Tilde{R}^{-i}, \Tilde{\pi}^{\dagger i}_t), (\Tilde{R}^{-i}, \Tilde{\pi}_t^{\ddagger})\}_{i =1 }^n$.
\end{definition}
Intuitively, the event defined here ensures that we attain sufficiently good policy estimates and sufficiently good value function estimates for these policies. Moreover, we highlight that the event allows for untruthfulness in the agents' behavior, thanks to our choices of $\fR$, and the ``good'' properties remain valid even when agents are untruthful. Examining the pairs of $(\fR, \pi)$ included in $\cE$, we can see that the good event $\cE_t$ directly implies that the clauses in Lemma~\ref{lemma:basic_lemma} hold for a specific value of $t > K$.

Across this paper, we let $\cE$ denote the intersection of all the event $\{\cE_t\}_{t=K+1}^{T}$ defined in \ref{def:event}, which is
\begin{equation}\label{def:cE}
\cE:=\cap_{t=k+1}^{T}\cE_t
\end{equation}
The following lemma shows that under the appropriate choice of regularization parameter $\lambda$ and scaling parameter $\beta$, event $\cE$ is guaranteed to happen with high probability. 
\begin{lemma}[Adaptation of Lemma 5.2 from~\citet{2020Is}]\label{lemma:event prob}
Under the setting in Section \ref{sec:prob}, we set
\begin{equation*}
\lambda=1,\quad \beta=\hat{c} (n+R_{\max} )  d H\sqrt{\iota},\quad \text{where } \iota=\log(36n d H T/\delta).
\end{equation*}
Here $\delta\in(0,1)$ is the confidence parameter. It holds that 
\begin{equation*}
\Pr_{\cD}(\cE)\geq 1-\delta/2.
\end{equation*}
where $\Pr_{\cD}$ denotes the probability under the data-generating distribution.
\end{lemma}
\begin{proof}
Note that by union bound, we only need to show that for an arbitrary and fixed $t > K$, the event $\cE_t$ holds with probability at least $1 - \delta/T$. We note that we can obtain a tighter bound for \texttt{ETC}, as the value functions and the policies do not change during exploitation. Here we slightly loosen our bound (by a multiplicative factor of $\log T$) for brevity of the proof.

Additionally, let us examine the possible choices of $(\fR, \pi)$ and $w$ for any $t > K$. We know that for any $t$, the concentration bound needs to hold for $2\cdot(2 + 7n )\leq 18n$ distinct reward-policy pairs. As such, we only need to show that for an arbitrary and fixed pair of $(\fR, \pi)$, the concentration bounds on $\hat{w}_h^{t, \pi}$ and $\check{w}_{h}^{t, \pi}$ hold simultaneously for all $h$ with probability at least $1 - \delta/36nT$. Without loss of generality, we consider only the pair $(R, \widehat{\pi}^t)$ and the associated optimistic linear weight, as the proof for all other pairs of $(\fR, \pi)$ and choices of weight $w$ remain largely the same. 

Moreover, note that $\widehat{\pi}^t$ is simply the policy outputted by Algorithm~\ref{algorithm:L3} with respect to $R$ when all agents are truthful. For simplicity, we then let $\hat{w}_h^{t, *}$ denote the weight associated with the pair $(R, \widehat{\pi}^t)$.
As we focus on the pair $(R, \pi)$ and the weight $\hat{w}_h^{t, *}$, for the rest of the proof, we let $f_h^t$ and $u_h^t$ denote the terms used by Algorithm~\ref{algorithm:L3}.

Recall the definition of the transition operator $\PP_{h + 1}$ and the Bellman operator $\BB_{h + 1}$ in Equation \eqref{equa:trans operator} and Equation \eqref{equa:Bellman operator}. We first show that for any function $f$,  $(\PP_{h}f)( ,  ; R)$ and $(\BB_{h + 1}f)( ,  ; R)$ are linear in the feature map $\phi$. By Equation \eqref{equa:Linear MDP},
\begin{align*}
&(\PP_{h} f)(x,a;R) = \Big\langle \phi(x,a) , \int f(x') \mu_h (x') \ud x' \Big \rangle\\
&(\BB_{h} f)(x,a;R) = \sum_{i = 0}^n\langle \phi(x,a), \btheta_{ih}\rangle + \Big\langle \phi(x,a) , \int f(x') \mu_h (   x') \ud x' \Big \rangle
\end{align*}
where we recall $\btheta_{ih}$ parameterizes $r_{ih}$. Crucially, the fact that both equations hold for a generic $f$ shows that $(\PP_{h}\widehat{V}^{t,*}_{h + 1})( ,  ; R)$ and $(\BB_{h}\widehat{V}^{t,*}_{h + 1})( ,  ; R)$ are both linear.

The objective is then to obtain a high probability bound over $|(\BB_h \widehat{V}_{h+1}^{t,*})(,; R)-\phi^{\top}\hat{w}_{h}^{t,*}|$ for all $h\in[H], (x,a)\in\cS\times\cA$. Let $w_h$ be the vector such that $(\BB_h\widehat{V}^{t,*}_{h+1})(,; R) = \phi(,)^{\top}w_h$, which is guaranteed to exist by the term's linearity. When $\zeta_1 = \texttt{EWC}$,  for all $t > K$, we have
\begin{equation}\label{equa:Bellman estimation diff}
\begin{aligned}
&\hspace{-0.2cm}(\BB_h \widehat{V}_{h+1}^{t,*})(x,a;R)-\phi(x,a)^{\top}\hat{w}_{h}^{t,*}=\phi(x,a)^{\top}(w_h-\hat{w}_{h}^{t,*})\\
&\hspace{-0.2cm}\qquad=\phi(x,a)^{\top}w_h-\phi(x,a)^{\top}(\Lambda_{h}^{t})^{-1}\Big(\sum_{\tau=1}^{K}\phi(x_{h}^{\tau},a_{h}^{\tau})\big(R_{h}^{\tau}+\widehat{V}_{h+1}^{t,*}(x_{h+1}^{\tau};R)\big)\Big)\\
&\hspace{-0.2cm}\qquad  = \underbrace{\phi(x,a)^\top w_h - \phi(x,a)^\top (\Lambda_{h}^{t})^{-1}\Big( \sum_{\tau=1}^{K} \phi(x_h^\tau,a_h^\tau) (\BB_h\widehat{V}_{h+1}^{t,*}) (x_{h}^\tau,a_h^\tau;R)  \Big)}_{\displaystyle \text{(i)}}\\
&\hspace{-0.2cm}\quad \qquad  -  \underbrace{\phi (x,a)^\top (\Lambda_{h}^{t})^{-1}\Big( \sum_{\tau=1}^{K} \phi(x_h^\tau,a_h^\tau)  \big(R_{h}^{\tau}+\widehat{V}_{h+1}^{t,*}(x_{h+1}^\tau;R) - (\BB_h\widehat{V}_{h+1}^{t,*}) (x_{h}^\tau,a_h^\tau;R) \big ) \Big)}_{\displaystyle\text{(ii)}},
\end{aligned}
\end{equation}
where the second equality follows from the construction of $\hat{w}_{h}^{t,*}$. Therefore we have
\begin{equation*}
\bigl | (\BB_h \widehat{V}_{h+1}^{t,*})(x,a)-\phi(x,a)^{\top}\hat{w}_{h}^{t,*} \bigr | \leq  |\text{(i)}| + |\text{(ii)}|.  
\end{equation*}
We now bound the two terms separately. Note  that $\widehat{V}_{h+1}^{t,*}( ; R)\in [0,(n+R_{\max})(H - h)]$ by truncation and $\|\theta_{h}\|=\|\sum_{i=0}^{n}\theta_{i h}\|\leq (n+1)\sqrt{d}$. Applying Lemma \ref{lemma:bound_weight_of_bellman}, we then know that $\|w_{h}\|\leq (n+R_{\max})(H - h) \sqrt{d} < (n + R_{max})H\sqrt{d}$ for all $h$. Hence, term (i) in Equation \eqref{equa:Bellman estimation diff} satisfies
\begin{equation}\label{equa:Bellman diff term1}
\begin{aligned}
|\text{(i)} | &= \bigg| \phi(x,a)^\top w_h - \phi(x,a)^\top (\Lambda_h^{t})^{-1} \Big( \sum_{\tau=1}^{K} \phi(x_h^\tau,a_h^\tau) \phi(x_h^\tau,a_h^\tau)^\top w_h  \Big)\bigg| \\
&= \big| \phi(x,a)^\top w_h - \phi (x,a)^\top (\Lambda_h^{t})^{-1}(\Lambda_h^{t} -  \lambda  I)w_h \big|=\lambda \big| \phi(x,a)^\top(\Lambda_h^{t})^{-1} w_h   \big|  \\
& \leq\lambda \|w_h \|_{ (\Lambda_h^{t})^{-1}} \|\phi(x,a) \|_{(\Lambda_h^{t})^{-1}} \leq (n+R_{\max})H\sqrt{d /\lambda } \sqrt{\phi(x,a)^\top(\Lambda_h^{t})^{-1}\phi(x,a)},
\end{aligned}
\end{equation}
where the second equality is by definition of $\Lambda_{h}^{t}$ and the last by the fact that $\Lambda_h^t \succeq \lambda I$.

It remains to upper bound term (ii) in Equation \eqref{equa:Bellman estimation diff} . For simplicity, we defined the random variable
\begin{equation}\label{equa:define eps V}
\epsilon_{h}^{\tau}(V; R)=R_{h}^{\tau}+V(x_{h+1}^\tau;R) - (\BB_h V) (x_{h}^\tau,a_h^\tau;R).
\end{equation}
We then have
\begin{equation}\label{equa:define_term3} 
\begin{aligned}
|\text{(ii)}| &= \bigg|  \phi (x,a)^\top (\Lambda_h^{t})^{-1} \Big( \sum_{\tau=1}^{K} \phi(x_h^\tau,a_h^\tau) \epsilon_h^\tau(\widehat{V}_{h+1}^{t,*};R) \Big)    \bigg|\\
& \leq \Big\|  \sum_{\tau=1}^{K} \phi(x_h^\tau,a_h^\tau) \epsilon_h^\tau(\widehat{V}_{h+1}^{t,*}) \Big\|_{(\Lambda_h^{t})^{-1}}   \|\phi(x,a) \|_{ (\Lambda_h^{t})^{-1}}\\
&= \underbrace{\Big\|  \sum_{\tau=1}^{K} \phi(x_h^\tau,a_h^\tau)  \epsilon_h^\tau(\widehat{V}_{h+1}^{t,*};R) \Big\|_{(\Lambda_h^{t})^{-1}}}_{\displaystyle \text{(iii)} } \sqrt{\phi(x,a)^\top (\Lambda_h^{t})^{-1}\phi(x,a)}.
\end{aligned}
\end{equation} 
Define the function class for any $L > 0, B > 0$, $h \in [H$
\begin{equation}\label{eq:uniform_concentration_function_class}
\begin{aligned}
&\cV_{h}(L,B,\lambda) =\big\{V_{h}(x;\theta,\beta,\Sigma)\colon \cS\to [0,(n+R_{\max})  H]~\text{with}~\|\theta\|\leq L, \beta\in [0,B], \Sigma \succeq \lambda  I   \big\},\\
&\text{where~~}V_h(x;\theta,\beta,\Sigma) = \max_{a\in \cA}  \Bigl\{ \min\bigl \{ \phi(x,a)^\top \theta + \beta  \sqrt{ \phi(x,a)^\top \Sigma^{-1}\phi(x,a) },(n+R_{\max})  H\bigr \} \Bigr\}.
\end{aligned}
\end{equation}
and let $ \cN_{h} (\varepsilon; L, B, \lambda)$ be the 
$\varepsilon$-cover of $\cV_h(L, B, \lambda)$ with respect to the distance $\dist(V,V^{\prime})=\sup_{x\in\cS}\big\|V(x)-V^{\prime}(x)\big\|$. By Lemma \ref{lemma:bound_weight_of_bellman}, we have $\big\|\hat{w}_{h + 1}^{t,*}\big\|\leq (n+R_{\max})H\sqrt{K d/\lambda}$, and therefore
\begin{equation*}
\widehat{V}_{h+1}^{t,*} \in \cV_{h+1} (L_0, B_0 , \lambda),\qquad \text{where~~}L_0=(n+R_{\max})H\sqrt{K d/\lambda},~B_0=2\beta.
\end{equation*}
Here $\lambda>0$ is the regularization parameter, and $\beta>0$ is the scaling parameter specified in Algorithm \ref{algorithm:L3}.
For simplicity, we use $\cV_{h+1} $ and $\cN_{h+1} (\varepsilon)$ to denote $\cV_{h+1} (L_0, B_0 , \lambda)$
and $\cN_{h+1} (\varepsilon; L_0, B_0, \lambda)$, respectively. There then exists a function $V^{\dagger}_{h+1}(x;R)\in\cN(\varepsilon)$ where
\begin{equation}\label{equa:bound_sup_norm_diff}
\sup_{x\in \cS} \big| \widehat{V}_{h+1}^{t,*}(x;R) - V^{0}_{h+1} (x;R)\big| \leq \varepsilon,
\end{equation}
By definition of the transition operator $\PP_h$ and Jensen's inequality,
\begin{equation*}
\begin{aligned}
&\big|(\PP_h V^{0}_{h+1} )(x,a;R) - (\PP_h \widehat{V}_{h+1}^{t,*})(x,a;R)\big| =  \Big|\EE \big[ V^{0}_{h+1} (x_{h+1};R)- \widehat{V}_{h+1}^{t,*}(x_{h+1};R) \Biggiven s_h = x,a_h = a  \big]\Big|\\
&\qquad \leq  \EE \Big[ \big|V^{0}_{h+1} (x_{h+1};R)- \widehat{V}_{h+1}^{t,*}(x_{h+1};R)\big| \Biggiven s_h = x,a_h = a  \Big]   \leq \varepsilon.
\end{aligned}
\end{equation*}
We then know that $    \big|(\BB_h V^{0}_{h+1} )(x,a; R) - (\BB_h \widehat{V}_{h+1}^{t,*})(x,a;R)\big|\leq\varepsilon$, and by triangle inequality,
\begin{equation}\label{equa:bound sup+cond}
\begin{aligned}
&\Big| \bigl (R_{h}^{t}(x,a)+\widehat{V}_{h+1}^{t,*} (x';R) - (\BB_h  \widehat{V}_{h+1}^{t,*} )(x,a;R) \bigr )\\
&\hspace{8em} -   \bigl (R_{h}^{t}(x,a)+ V^{0}_{h+1} (x';R) - (\BB_h V^{0}_{h+1}  )(x,a;R)\big )  \Big| \leq 2\varepsilon   
\end{aligned}
\end{equation}
for all $h\in [H]$ and all $(x,a,x')\in\cS\times\cA\times \cS$. Setting $(x,a,x')=(x_h^\tau,a_h^\tau,x_{h+1}^\tau)$ in Equation \eqref{equa:bound sup+cond} ensures
\begin{equation*}
\bigl| \epsilon_h^\tau(\widehat{V}_{h+1}^{t,*};R)-\epsilon_h^\tau(V^{0}_{h+1};R )\bigr| \leq 2\varepsilon, \qquad \forall \tau\in [K],~\forall h\in [H].
\end{equation*}
We then have the following bound for term (iii) in Equation \eqref{equa:define_term3}.
\begin{equation}\label{equa:term3_bound1}
\begin{aligned}
|\textrm{(iii)}|^2 \leq  &2  \Big\| \sum_{\tau=1}^{K} \phi(x_h^\tau,a_h^\tau)  \epsilon_h^\tau(V^{0}_{h+1};R) \Big\|_{(\Lambda_h^{t})^{-1}}^2\\
&+ 2 \Big\| \sum_{\tau=1}^{K} \phi(x_h^\tau,a_h^\tau)  \big(\epsilon_h^\tau(\widehat{V}_{h+1}^{t,*}; R)-\epsilon_h^\tau(V^{0}_{h+1};R)\big) \Big\|_{(\Lambda_h^{t})^{-1}}^2. 
\end{aligned}
\end{equation}
By direct expansion, the second term on the right-hand side of Equation \eqref{equa:term3_bound1} can be controlled as follows. 
\begin{equation}\label{equa:term3_bound2}
\begin{aligned}
&2 \Big\| \sum_{\tau=1}^{K} \phi(x_h^\tau,a_h^\tau)   \big(\epsilon_h^\tau(\widehat{V}_{h+1}^{t,*};R)-\epsilon_h^\tau(V^{0}_{h+1};R)\big) \Big\|_{(\Lambda_h^{t})^{-1}}^2 \\
&\qquad  =  2 \sum_{\tau=1}^{K} \sum_{\tau'=1}^K \phi(x_h^\tau,a_h^\tau)^\top   (\Lambda_h^\tau) ^{-1} \phi(x_h^{\tau'},a_h^{\tau'})\\
&\hspace{10em} \times \big(\epsilon_h^\tau(\widehat{V}_{h+1}^{t,*} ;R)-\epsilon_h^\tau(V^{0}_{h+1};R)\big)  \big(\epsilon_h^{\tau'}(\widehat{V}_{h+1}^{t,*};R)-\epsilon_h^{\tau'}(V^{0}_{h+1};R)\big)\\
&\qquad \leq 8\varepsilon^2   \sum_{\tau=1}^{K} \sum_{\tau'=1}^K \bigl | \phi(x_h^\tau,a_h^\tau)^\top (\Lambda_h^{t})^{-1} \phi(x_h^{\tau'},a_h^{\tau'}) \big |\leq 8\varepsilon^2 K^2/\lambda,
\end{aligned}
\end{equation}
where the last step follows from the fact that $\| \phi(x,a) \| \leq 1$ and $\Lambda_h^t \succeq \lambda I$. Combining Equations \eqref{equa:term3_bound1} and \eqref{equa:term3_bound2} shows
\begin{equation}\label{equa:term3_bound3}
|\textrm{(iii)}|^2 \leq  2    \sup_{V \in \cN_{h+1} (\varepsilon)}    \Big\| \sum_{\tau=1}^{K} \phi(x_h^\tau,a_h^\tau)  \epsilon_h^\tau(V; R)\Big\|_{(\Lambda_h^{t})^{-1}}^2   +  8 \varepsilon^2  K^2 / \lambda. 
\end{equation}
We then upperbound the term $ \sup_{V \in \cN_{h+1} (\varepsilon)}    \Big\| \sum_{\tau=1}^{K} \phi(x_h^\tau,a_h^\tau)  \epsilon_h^\tau(V; R)\Big\|_{(\Lambda_h^{t})^{-1}}^2$ by uniform concentration over the covering $\cN_{h + 1}(\varepsilon)$. Applying Lemma \ref{lemma:concentration_of_self_norm_pro} and taking union bound over $\cN_{h + 1}(\varepsilon)$, for any fixed $h\in[H]$, with probability at least $1-p |\cN_{h+1}(\varepsilon)|$,
\begin{equation*}
\sup_{V\in\cN_{h+1}(\varepsilon)}\Big\| \sum_{\tau=1}^{K} \phi(x_h^\tau,a_h^\tau)  \epsilon_h^\tau(V; R)\Big\|_{(\Lambda_h^{t})^{-1}}^2 \leq (n+R_{\max})^2 H^2  \bigl (  2    \log(1/ p ) + d  \log(1+K/\lambda)\big).
\end{equation*}
For all $\delta\in(0,1)$ and all $\varepsilon>0$, we set $p=\delta/[(36n)  H |\cN_{h+1}(\varepsilon)|]$. Hence, for all fixed $h\in[H]$, it holds that 
\begin{equation}\label{equa:concentratation_fixed_h}
\begin{aligned}
&\sup_{V\in\cN_{h+1}(\varepsilon)}\Big\|   \sum_{\tau=1}^{K} \phi(x_h^\tau,a_h^\tau)  \epsilon_h^\tau(V;R)\Big\|_{(\Lambda_h^{t})^{-1}}^2\\ &\qquad \leq (n+R_{\max})^2 H^2  \bigl (  2    \log((36n)  H |\cN_{h+1}(\varepsilon)|/ \delta ) + d  \log(1+K/\lambda)\big)
\end{aligned}
\end{equation}
with probability at least $1-\delta/(36 n H)$, taken with respect to process that generates the dataset $\cD$. Then, combining Equations \eqref{equa:term3_bound3} and \eqref{equa:concentratation_fixed_h}, for all $h\in[H]$, with probability at least $1 - \delta/(18 n H)$,
\begin{align*}
&\Big\|\sum_{\tau=1}^{K} \phi(x_h^\tau,a_h^\tau)   \epsilon_h^\tau(\hat V_{h+1}^{t,*};R) \Big\|_{(\Lambda_h^{t})^{-1}}^2 = |\textrm{(iii)}|^2 \\ 
&\qquad \leq (n+R_{\max})^2 H^2  \bigl (  2    \log((36n)  H |\cN_{h+1}(\varepsilon)|/ \delta ) + d  \log(1+K/\lambda)\big)+8\varepsilon^2 K^2 / \lambda.
\end{align*}
Since $\widehat{V}_{h+1}^{t,*} \in \cV_{h+1} ((n+R_{\max})H\sqrt{T d/\lambda},2\beta , \lambda)$  we can upperbound $|\cN_{h+1}(\varepsilon)|$ via Lemma \ref{lem:covering_num}. As term (iii) is controlled, we can then ensure that term (ii) of Equation~\eqref{equa:Bellman estimation diff} can be bounded, which when combined with Equation~\eqref{equa:Bellman diff term1} yields a bound for $\bigl | (\BB_h \widehat{V}_{h+1}^{t,*})(x,a)-\phi(x,a)^{\top}\hat{w}_{h}^{t,*} \bigr |$ under a specific choice of $\varepsilon$, $\beta$, and $\lambda$.

All that remains is then to set the hyperparameters to ensure that the error can be bounded. Letting $\iota=\log(36n d H T/\delta)$, we set
\[
\beta=\hat{c}(n+R_{\max} ) d H\sqrt{\iota},\quad  \varepsilon = dH / K, \quad  \lambda = 1,
\]
where $\hat{c}>0$ is an absolute constant that ensures
\begin{equation}\label{equa:term2_bound}
|\text{(ii)}| \leq (\hat{c}/2 ) n  d H \sqrt{ \iota }   \sqrt{ \phi(x,a) ^\top  (\Lambda_{h}^{t})  ^{-1} \phi(x,a) }  =  \beta /2   \sqrt{ \phi(x,a) ^\top  (\Lambda_{h}^{t})  ^{-1} \phi(x,a) }
\end{equation}
with probability at least $1-\delta/(36nT)$.
By Equations \eqref{equa:Bellman estimation diff}, \eqref{equa:Bellman diff term1} and \eqref{equa:term2_bound}, for all $h \in [H]$ and all $(x,a) \in \cS\times \cA$, it holds that 
\begin{equation*}
\bigl |  (\BB_h \hat V_{h+1} ) (x,a) -\phi(x,a)^{\top}\hat{w}_{h}^{t,*} \bigr |  \leq ((n+R_{\max})  H \sqrt{d} +  \beta /2 )   \sqrt{ \phi(x,a) ^\top (\Lambda_h^t)  ^{-1} \phi(x,a) },
\end{equation*}
with probability at least $1-\delta/(36n)$, taking the union bound over $h \in [H]$. 

Extending the result to when $\zeta_1 = \texttt{EWC}$ is straightforward. Observe that Equation~\eqref{equa:Bellman estimation diff} consists of bounding $K$ random variables whose randomness is due to only the stochasticity inherent in the transition kernel. Moving from the $\texttt{ETC}$ to $\texttt{EWC}$ setting simply requires bounding $T$, rather than $K$, such variables. However, as our choice for $\beta$ and $\iota$ accommodates the move from $K$ to $T$, the bound in Equations~\eqref{equa:term2_bound} and~\eqref{equa:Bellman diff term1} remain valid.

Then combining Equation \eqref{equa:term2_bound} and~\eqref{equa:Bellman diff term1}, we obtain
\begin{equation*}
\big|\BB_h \widehat{V}_{h+1}^{t,*}(x,a)-\phi(x,a)^{\top}\hat{w}_{h}^{t,*}\big|\leq \beta   \sqrt{ \phi(x,a) ^\top (\Lambda_h^t)  ^{-1} \phi(x,a) }.
\end{equation*}

As there are only $18n$ such combinations of $\fR, \pi$ and $w$, obtaining the individual upper bound with probability at least $1-\delta/(36n)$ ensures that the union bound over all these triplets is satisfied with probability at least $1 - \delta / 2$.
Therefore, we conclude the proof of Lemma \ref{lemma:event prob}.
\end{proof}

\subsection{Proof of Lemma \ref{lemma:basic_lemma}}\label{subsec:basic_lemma}
With event $\cE$ defined, 
we proceed with the proof of Lemma~\ref{lemma:basic_lemma}. The proof is organized as follows. We first directly control the model evaluation errors conditioned on the event $\cE$, then relate these model evaluation errors to uncertainty bonuses $u_h^t$, followed by a reward-free style analysis that ensures sufficiently small model evaluation error across all policies. Combining these three ingredients yields Lemma~\ref{lemma:basic_lemma} directly.

In the first step of the proof, we upper and lower bound the model evaluation error $\Delta$, defined in Equation~\eqref{equa:evaluation error}, in the following lemma.
\begin{lemma}[Adaptation of Lemma 5.1 from~\citet{2020Is}]\label{lemma:width ETC}
With $\lambda, \beta$ set according to Lemma \ref{lemma:event prob}, which ensures $\Pr_{\cD}(\cE)\geq 1-\delta/2$,
we have
\begin{equation}
\begin{aligned}
&0\geq \hat{\Delta}_{h}^{t,\pi}(x,a;\fR)\geq -2 u_{h}^{t}(x,a), \qquad 0\leq \check{\Delta}_{h}^{t,\pi}(x,a;\fR)\leq 2 u_{h}^{t}(x,a)
\end{aligned}
\end{equation}
for all $t > K$, $(x,a)\in\cS\times\cA,$ $h\in[H]$, and $(\fR, \pi) \in \{(R, \widehat{\pi}), (\Tilde{R}, \Tilde{\pi}^{\ddagger}_t)\} \cup \{(r_i + \Tilde{R}^{-i}, \Tilde{\pi}_t^{\dagger i}), (R^{-i}, *),\\ (\Tilde{R}^{-i}, \dagger), (\Tilde{R}^{-i}, \ddagger), (R^{-i}, \widehat{\pi}^t), (\Tilde{R}^{-i}, \Tilde{\pi}^{\dagger i}_t), (\Tilde{R}^{-i}, \Tilde{\pi}_t^{\ddagger})\}_{i =1 }^n$, regardless of the choice of $\zeta_1$.
\end{lemma}
\begin{proof}
The results in Lemma \ref{lemma:width ETC} can be split into two parts: the upper and lower bounds of $\{\check \Delta_{h}\}$ and $\{\hat \Delta_{h}\}$. For brevity, we take $\hat{\Delta}_{h}^{t,*}(x,a;R^{-i})$ and $\check{\Delta}_{h}^{t,\widehat{\pi}^{t}}(x,a;R^{-i})$ as examples for optimistic and pessimistic versions for an arbitrary $i$, because the techniques used are largely the same.

\paragraph{Bounding $\hat{\Delta}_h^{t, *}(x, a; R^{-i})$.}
We first show that conditioned on the event $\cE$, as defined in Definition~\ref{def:event} and Equation \eqref{def:cE}, 
the model evaluation errors $ \hat{\Delta}_{h}^{t,*}(x,a;R^{-i}) \leq 0$ for all $h \in [H]$. We assume that $\cE$ holds for the rest of the proof. Recalling the construction of $\hat Q_h^{t,*} $ from Algorithm~\ref{algorithm:L3},
for all $h \in [H]$ and all $(x,a) \in \cS \times \cA$, we have 
\begin{equation*}
\hat Q_h^{t,*}(x,a;R^{-i})  = \min \{  (  f_{h}^{t}+ u_{h}^{t} )(x,a) , (H - h + 1)(n-1+R_{\max})\}.
\end{equation*}
Throughout the rest of the paragraph, we use $f_h^t$ and $u_h^t$ to denote the components that Algorithm~\ref{algorithm:L3} uses in order to construct $\hat{Q}_h^{t, *}(x, a; R^{-i})$.
We first focus on when $f_{h}^{t}+u_{h}^{t}(x,a)\leq (H - h + 1)(n-1+R_{\max})$. Here we have $\hat Q_h^{t,*}(x,a;R^{-i})=f_{h}^{t}+u_{h}^{t}(x,a)$. By definition of $\hat{\Delta}_{h}^{t,*}(x,a;R^{-i})$ in Equation \eqref{equa:evaluation error}, 
\begin{equation*}
\hat{\Delta}_{h}^{t,*}(x,a;R^{-i})= (\BB_h\widehat{V}_{h+1}^{t,*})(x,a;R^{-i}) - \hat{Q}_h^{t,*}(x,a;R^{-i})= (\BB_h\widehat{V}_{h+1}^{t,*})(x,a; R^{-i})-f_{h}^{t}-u_{h}^{t}\leq 0,
\end{equation*}
and the desired bound on $\hat{\Delta}_{h}^{t,*}(x,a;R^{-i})$ inequality follows from Lemma \ref{lemma:event prob}.

If $f_{h}^{t}+u_{h}^{t}(x,a)\geq (H - h + 1)(n+R_{\max})$, we have $\hat Q_h^{t,*}(x,a;R^{-i})=(H - h + 1)(n+R_{\max})$, which implies
\begin{equation*}
\hat{\Delta}_{h}^{t,*}(x,a;R^{-i})= (\BB_h\widehat{V}_{h+1}^{t,*})(x,a;R^{-i}) - ((H - h + 1)(n+R_{\max}))\leq 0,
\end{equation*}
where the inequality follows from the definition of the Bellman operator in Equation \eqref{equa:Bellman operator} and the construction of $\hat V_{h+1}^{t,*}$ in Algorithm \ref{algorithm:L3}.

It remains to establish the lower bound of $\hat{\Delta}_{h}^{t,*}(x,a;R^{-i})$.
Combining the definition of $\hat{\Delta}_{h}^{t,*}(x,a;R^{-i})$ and $\hat Q_h^{t,*}(x,a;R^{-i})$, we have 
\begin{align*}
\hat{\Delta}_{h}^{t,*}(x,a;R^{-i})&= (\BB_h\widehat{V}_{h+1}^{t,*})(x,a; R^{-i}) - \hat{Q}_h^{t,*}(x,a; R^{-i})\\
&\geq (\BB_h\widehat{V}_{h+1}^{t,*})(x,a; R^{-i})-f_{h}^{t}-u_{h}^{t}\geq -2u_{h}^{t},
\end{align*}
where the first inequality follows from the definition of $\hat Q_h^{t,*}(x,a;R^{-i})$ and the second inequality follows from Lemma \ref{lemma:event prob}.
In summary, we conclude that when conditioned on $\cE$,
\begin{equation*}
0\geq \hat{\Delta}_{h}^{t,*}(x,a;R^{-i}) \geq -2u_{h}^{t}(x,a),\qquad \forall (x,a)\in \cS\times \cA, ~\forall h \in [H]. 
\end{equation*}

\paragraph{Bounding $\check{\Delta}_h^{t, *}(x, a; R^{-i})$.} We now show that the model evaluation errors for the pessimistic version is also bounded. Recalling the construction of $\check{Q}_h^{t, *}$, we have
\[
\check{Q}_h^{t, *}(x, a; R^{-i}) = \Pi_{[0, (n - 1 + R_{\max})(H - h + 1)]}[(f_h^t - u_h^t)(x, a)].
\]
For the rest of the paragraph, we instead let $f_h^t$ and $u_h^t$ denote the components Algorithm~\ref{algorithm:L3} uses to construct $\check{Q}_h^{t, *}(x, a; R^{-i})$ instead.
We first show that the term is bounded below by zero.
When $(f_h^t - u_h^t)(x, a) \leq 0$, we trivially have
\[
\check{\Delta}_h^{t,*}(x, a; R^{-i}) = (\BB_h\check{V}_{h + 1}^{t, *})(x, a; R^{-i}) - 0 \geq 0.
\]
When $(f_h^t - u_h^t)(x, a) \in (0, (n - 1 + R_{\max})(H - h + 1))$, we have
\begin{align*}
\check{\Delta}_h^{t,*}(x, a; R^{-i}) = (\BB_h\check{V}_{h + 1}^{t, *})(x, a; R^{-i}) - f_h^t + u_h^t \geq 0,
\end{align*}
where the inequality direct follows from Lemma~\ref{lemma:event prob}. Finally, when $(f_h^t - u_h^t)(x, a) \geq (n - 1 + R_{\max})(H - h + 1)$, we have
\begin{align*}
\check{\Delta}_h^{t,*}(x,a; R^{-i}) \geq (\BB_h\check{V}_{h + 1}^{t, *})(x, a; R^{-i}) - f_h^t + u_h^t \geq 0,
\end{align*}
where the inequality is again by Lemma~\ref{lemma:event prob}.

We then bound the term from above. When $(f_h^t - u_h^t)(x, a) \in (0, (n - 1 + R_{\max})(H - h + 1))$
\[
\check{\Delta}_h^{t,*}(x, a; R^{-i}) \leq (\BB_h\check{V}_{h + 1}^{t, *})(x, a; R^{-i}) - f_h^t + u_h^t \leq 2u_h^t 
\]
by Lemma~\ref{lemma:event prob}. When $(f_h^t - u_h^t)(x, a) \in (0, (n - 1 + R_{\max})(H - h + 1))$, the same bound holds as well for the same reason. We then focus on when  $(f_h^t - u_h^t)(x, a) \geq (n - 1 + R_{\max})(H - h + 1)$, in which case
\begin{align*}
\check{\Delta}_h^{t,*}(x, a; R^{-i}) &= (\BB_h\check{V}_{h + 1}^{t, *})(x, a; R^{-i})  -  (n - 1 + R_{\max})(H - h + 1)\\
&\leq (n - 1 + R_{\max})(H - h + 1) - (n - 1 + R_{\max})(H - h + 1) = 0.
\end{align*}
The last inequality comes from the fact that $\check{V}_{h + 1}^{t, *}( ; R^{-i})$ and $R^{-i}$ are bounded above. 

As the proofs for the remaining reward functions remain largely the same, we can apply the same analysis, only changing the reward function being used, thus completing the proof.
\end{proof}

With Lemma~\ref{lemma:width ETC} in mind, we relate the value function estimation errors to the uncertainty bonus $u_h^t$.
\begin{lemma}\label{lemma:width bound}
With $\lambda, \beta$ set according to Lemma \ref{lemma:event prob}, {which ensures $\Pr_{\cD}(\cE)\geq 1-\delta/2$,}
regardless of the choice of $\zeta_1$,
the following statements hold true jointly for all $t > K$ and some absolute constant $\hat{c}$.
\begin{enumerate}
\item $0 \leq \widehat{V}^{\pi}_1(x_1; \fR) - V^{*}_1(x_1; \fR) \leq 2\sum_{h = 1}^H \EE_{\pi}[u_h^t]$ for all $(\fR, \pi) \in \{(R, \widehat{\pi}^t), (\Tilde{R}, \Tilde{\pi}^{\ddagger}_t)\} \cup \{(r_i + \Tilde{R}^{-i}, \Tilde{\pi}_t^{\dagger i})\}_{i = 1}^n$.
\item For all $i \in [n]$, $0 \leq \widehat{V}_1^{t, \pi}(x_1; \fR) - V^*(x_1; \fR) \leq 2\sum_{h = 1}^H \EE_{\pi}[u_h^t]$ and $0 \leq V^*(x_1; \fR) - \check{V}_1^{t, \pi}(x_1; \fR) \leq 2\max_{\pi'}\{\sum_{h = 1}^H \EE_{\pi'}[u_h^t]\}$, for all $(\fR, \pi) \in \{(R^{-i}, \star), (\Tilde{R}^{-i}, \dagger), (\Tilde{R}^{-i}, \ddagger)\}_{i = 1}^n$.
\item For all $i \in [n]$,  $0\leq \widehat{V}_1^{t,\pi}(x_1; \fR) - V^{\pi}_1(x_1; \fR)\leq 2\sum_{h = 1}^H \EE_{\pi}[u_h^t]$ and $0 \leq V^{\pi}_1(x_1; \fR) - \check{V}_1^{t,\pi}(x_1; \fR) \leq 2\sum_{h = 1}^H \EE_{\pi}[u_h^t]$, for all $(\fR, \pi) \in \{(R^{-i}, \widehat{\pi}^t), (\Tilde{R}^{-i}, \Tilde{\pi}_t^{\dagger i}), (\Tilde{R}^{-i}, \Tilde{\pi}_t^{\ddagger})\}_{i = 1}^n$.
\end{enumerate} 
where the bonuses $\{u_{h}^{t}\}$ are the exploration bonuses calculated by either Algorithm~\ref{algorithm:L3} or Algorithm~\ref{algorithm:L4}.
\end{lemma}
\begin{proof}
For brevity, we only upper bound $V_{1}^{*}(x_1;R)-V_{1}^{\widehat{\pi}^{t}}(x_1;R)$ in this section, as the proof of the remaining terms is similar. 

Adding and subtracting $\hat V_1^{t,*}$ into the difference, we can decompose the difference into two terms
\begin{equation}\label{equa:decompoce1}
V_{1}^{*}(x_1;R)-V_{1}^{\widehat{\pi}^{t}}(x_1;R)=\underbrace{\Big(V_{1}^{*}(x_{1};R)-\hat V_1^{t,*}(x_1;R)\Big)}_{\text{(i)}}+\underbrace{\Big(\hat V_1^{t,*}(x_1;R)-V_{1}^{\widehat{\pi}^t}(x_1;R)\Big)}_{\text{(ii)}}.
\end{equation}
where we recall $\hat V_1^{t,*}(x_1;R)$ is the value function estimates constructed by Algorithm \ref{algorithm:L3}.  
Term (i) in Equation \eqref{equa:decompoce1} is the difference between the estimated value function    $\hat V_1^{t,*}( ; R)$  and the optimal value function $V_1^{*}( ; R)$, while  term (ii) is the difference between  $\hat V_1^{t,*}( ; R)$ and the value function of $\hat\pi^t$, $V_1^{\hat\pi^t}( ; R)$.

For term (i), we invoke Lemma \ref{lemma:ext_val_diff} with $\pi=\widehat{\pi}^t$ and  $\pi^\prime = \pi_*$ and have
\begin{align*}
\hat V_1^{t,*}(x_1;R) - V_1^{*}(x_1;R) &= \sum_{h=1}^H \EE_{\pi_*}\big[ \langle \hat{Q}_h^{t,*}(x_h, ;R) , \hat\pi_h^t ( \given x_h) - \pi_{*,h}( \given x_h) \rangle_{\cA} \biggiven x_1=x\big] \notag \\
& \quad + \sum_{h=1}^H   \EE_{\pi_*}\big[      \hat{Q}_h^{t,*}(x_h,a_h;R)-  ( \BB_h \widehat{V}_{h+1}^{t,*}) (x_h,a_h;R) \biggiven x_1=x\big],
\end{align*}
where $\EE_{\pi_*} $  is taken with respect to the trajectory generated by $\pi_*$.
By the definition of the model evaluation error 
$\Delta _h $  in Equation \eqref{equa:evaluation error}, we have
\begin{align}\label{equa:subopt_term1}
\begin{aligned}
V_1^{*}(x_1;R)-\hat V_1^{t,*}(x_1;R)  &= \sum_{h=1}^H \EE_{\pi_*}\big[ \langle \hat{Q}_h^{t,*}(x_h,
a_h;R) ,\pi_{*,h}( \given x_h)-\hat\pi_h^t ( \given x_h) \rangle_{\cA} \biggiven x_1\big] \\
& \quad + \sum_{h=1}^H   \EE_{\pi_*}\big[      \hat{\Delta}_h^{t,*}(x_h,a_h;R)\biggiven x_1\big].
\end{aligned}
\end{align}
Similarly, invoking Lemma \ref{lemma:ext_val_diff} with $\pi=\pi^\prime = \widehat{\pi}^t$, for term (ii), we have
\begin{align}\label{equa:subopt_term2}
\begin{aligned}
\hat V_1^{t,*}(x_1;R)-V_{1}^{\widehat{\pi}^{t}}(x_1;R)&=\sum_{h=1}^H   \EE_{\widehat{\pi}^t}\big[      \hat{Q}_h^{t,*}(x_h,a_h;R)-  ( \BB_h \widehat{V}_{h+1}^{t,*}) (x_h,a_h;R) \biggiven x_1\big]\\
&=-\sum_{h=1}^H   \EE_{\widehat{\pi}^t}\big[      \hat{\Delta}_h^{t,*}(x_h,a_h;R)\biggiven x_1\big],
\end{aligned}
\end{align}
where $\EE_{\widehat{\pi}^t} $ is taken with respect to the trajectory generated by $\widehat{\pi}^t$.

Combining Equations \eqref{equa:decompoce1}, \eqref{equa:subopt_term1} and \eqref{equa:subopt_term2}, we have
\begin{align}
V_{1}^{*}(x_1;R)-V_{1}^{\widehat{\pi}^{t}}(x_1;R)&=\sum_{h=1}^H \EE_{\pi_*}\big[ \langle \hat{Q}_h^{t,*}(x_h, ;R) ,\pi_{*,h}( \given x_h)-\hat\pi_h^t ( \given x_h) \rangle_{\cA} \biggiven x_1\big]\label{equa:subopt_total_1}\\
&\quad+\sum_{h=1}^H   \EE_{\pi_*}\big[      \hat{\Delta}_h^{t,*}(x_h,a_h;R)\biggiven x_1\big]-\sum_{h=1}^H   \EE_{\widehat{\pi}^t}\big[      \hat{\Delta}_h^{t,*}(x_h,a_h;R)\biggiven x_1\big].\notag
\end{align}
It remains to upper bound the three terms in the right-hand side of Equation \eqref{equa:subopt_total_1}. For the first term, we can upper bound it by $0$ following the definition of $\widehat{\pi}^t$ in  Algorithm \ref{algorithm:L3}. To bound the last two terms, we invoke Lemma \ref{lemma:width ETC}, which implies 
\begin{align*}
&\sum_{h=1}^H   \EE_{\pi_*}\big[      \hat{\Delta}_h^{t,*}(x_h,a_h;R)\biggiven x_1=x\big]\leq 0,\\
&-\EE_{\widehat{\pi}^t}\big[      \hat{\Delta}_h^{t,*}(x_h,a_h;R)\biggiven x_1=x\big]\leq\EE_{\widehat{\pi}^t}\big[     2 u_{h}^{t}(x_{h},a_{h})\biggiven x_1=x\big],
\end{align*}
for all $(x, a) \in \cS \times \cA$ under event $\cE$. We then know that
\begin{equation*}
V_{1}^{*}(x_1;R)-V_{1}^{\widehat{\pi}^{t}}(x_1;R)\leq\sum_{h=1}^{H}\EE_{\widehat{\pi}^t}\big[   2   u_{h}^{t}(x_{h},a_{h})\biggiven x_1=x\big].
\end{equation*}
The remaining terms can be controlled with a similar technique, with only minor differences between optimistic and pessimistic value function estimates. The differences only affect the signs of the resulting terms but do not change the proof itself.
We conclude the proof.
\end{proof}

As we can see from Lemma \ref{lemma:width bound}, all that remains is to control the term $\EE_{\pi}[\sum_{h=1}^{H}u_{h}^{t} \given x_1]$. For convenience, we begin with a more general bound that holds for all $\pi$ and $\fR$, and then discuss a specialized bound for when $\zeta_1 = \texttt{EWC}.$ Recalling Algorithm \ref{algorithm:L3}, bounding $V^{*}(x_{1};u_h^t)$ suffices, as the definition of $V^*$ ensures that it is the maximum of $\EE_{\pi}[\sum_{h=1}^{H}u_{h}^{t} \given {x_{1}}]$ taken over $\pi$. 
{We detail the steps in the following Lemma.}

\begin{lemma}\label{lemma:reward_free_explore}
With probability at least $1-\delta/(36nT)$, for the function $u_{h}^{t}$ defined in Algorithm \ref{algorithm:L3}, we have for all $t > K$ that
\begin{equation*}
V_{1}^{*}(x_{1};u^{t})\leq 2\hat{c} (n+R_{\max} ) \sqrt{d^{3}H^{6}\iota/K},
\end{equation*}
{where $\iota=\log\big(36 n d H T/\delta\big)$, and $\hat{c}$ is an absolute constant.} The claim holds regardless of the choice of $\zeta_1$.
\end{lemma}

\begin{proof}
Using the similar technique in the proof of Lemma  \ref{lemma:event prob} and Lemma \ref{lemma:width bound},  with probability at least $1-\delta/8$, we have for possible pairs of $(\fR, \pi)$, 
\begin{align}
\begin{aligned}
\label{equa:reward_free_concentration}
&\bigl |  (\PP_h  V_{h+1}^{k} ) (x,a; \fR) -\Pi_{[0,B]}[\phi(x,a)^{\top}w_{h}^{k}] \bigr |  \\
&\qquad \leq \min\Big\{\beta \sqrt{ \phi(x,a) ^\top (\Lambda_h^k)  ^{-1} \phi(x,a) },B \Big\}= u_h^k (x,a),
\end{aligned}
\end{align}
for all $h \in [H]$ and all $(x,a) \in \cS\times \cA$ with $B=H(n+R_{\max})$, where $w_h^k$ is the linear weight constructed in Algorithm~\ref{algorithm:LMVL} during the exploration phase. 
For simplicity, for the remaining proof we let $V^k(\cdot) = V(\cdot; u^k), Q^k(\cdot, \cdot) = Q(\cdot, \cdot; u^k),$ and $(\PP_hV^k)(\cdot, \cdot) = (\PP_hV)(\cdot, \cdot; u^k)$.
Based on the above inequality, we have the following intermediate results for the functions $V_1^*(\cdot ; l^k)$ and $V^k_1(\cdot)$ defined in Algorithm~\ref{algorithm:L1}
\begin{equation}\label{equa:reward_free_1}
V_{1}^{*}(x_{1};l^{k})\leq V_{1}^{k}(x_{1}) \quad\text{for all $k\in[K]$},
\end{equation}
and 
\begin{equation}\label{equa:reward_free_2}
\sum_{k=1}^{K}V_{1}^{k}(x_{1})\leq \hat{c} (n+R_{\max} ) \sqrt{ d^{3}H^{4}K\iota},
\end{equation}
for some absolute constant $\hat{c}$ with probability at least $1-\delta/4$.

Equation \eqref{equa:reward_free_1} and Equation \eqref{equa:reward_free_2} show that the estimated value function in the exploration phase is optimistic and the sum of $V_{1}^{k}(x_1)$ should be small with high probability. 

Equation \eqref{equa:reward_free_1} can be proved by induction.
When $h=H+1$, for all $k\in[K]$ and $s\in\cS$, we know $V_{H+1}^{*}(x;l^{k})=0$ and $V_{H+1}^{k}(x)=0$ such that $V_{H+1}^{*}(x;l^{k})=V_{H+1}^{k}(x)$. Assume that for some $h\in[H]$ and all $x\in\cS$,
\begin{equation*}
V_{h+1}^{*}(x;l^{k})\leq V_{h+1}^{k}(x).
\end{equation*} 
Then based on Equation \eqref{equa:reward_free_concentration}, for all $(x,h,k)\in\cS\times[H]\times[K]$, we further have
\begin{equation*}
\begin{aligned}
&Q_{h}^{*}(x,a;l^{k})- Q_{h}^{k}(x,a)\\
&=l_{h}^{k}(x,a)+(\PP_{h}V_{h+1}^{*})(x, a;l^{k})-\min\{\Pi_{[0,B]}[(w_{h}^{k})^{\top}\phi(x,a)]+l_{h}^{k}(x,a)+u_{h}^{k}(x,a), B\}\\
&\leq \max\{(\PP_{h}V_{h+1}^{*})(x, a;l^{k})-\Pi_{[0,B]}[(w_{h}^{k})^{\top}\phi(x,a)]-u_{h}^{k}(x,a),0\}\\
&\leq \max\{(\PP_{h}V_{h+1}^{k})(x, a)-\Pi_{[0,B]}[(w_{h}^{k})^{\top}\phi(x,a)]-u_{h}^{k}(x,a),0\}\\
&\leq 0,
\end{aligned}
\end{equation*}
where the first inequality is due to $0\leq l_{h}^{k}(x,a)+(\PP_{h}V_{h}^{*})(x, a;l^{k})\leq B$, the second inequality is by the assumption that $l_{h}^{k}(x,a)+\PP_{h}V_{h}^{*}(x;l^{k})$, and the last inequality by Equation \eqref{equa:reward_free_concentration}.
The above inequality further leads to 
\begin{equation*}
V_{h}^{*}(x;l^{k}) =\max_{a\in\cA}Q_{h}^{*}(x,a;l^{k})\leq \max_{a\in\cA}Q_{h}^{k}(x,a)=V_{h}^{k}(x).
\end{equation*}
We can then complete the proof of Equation \eqref{equa:reward_free_1} by induction.

Next, we detail the proof of  Equation \eqref{equa:reward_free_2}, namely the upper bound of $\sum_{k=1}^{K}V_{1}^{k}(x_1)$.
Specifically, based on Equation \eqref{equa:reward_free_concentration}, we have
\begin{equation}\label{equa:reward_free_22}
\begin{aligned}
V_{h}^{k}(x_{h}^{k})&\leq\Pi_{[0,B]}[(w_{h}^{k})^{\top}\phi(x_{h}^{k},a_{h}^{k})]+l_{h}^{k}(x_{h}^{k},a_{h}^{k})+u_{h}^{k}(x_{h}^{k},a_{h}^{k})\\
&\leq \PP_{h}V_{h+1}^{k}(x_{h}^{k},a_{h}^{k})+l_{h}^{k}(x_{h}^{k},a_{h}^{k})+2u_{h}^{k}(x_{h}^{k},a_{h}^{k})\\
&= \PP_{h}V_{h+1}^{k}(x_{h}^{k},a_{h}^{k})-V_{h+1}(x_{h+1}^{k})+V_{h+1}(x_{h+1}^{k})+(2+1/H)u_{h}^{k}(x_{h}^{k},a_{h}^{k}),
\end{aligned}
\end{equation}
where the first inequality is due to the definition of $V_{h}^{k}$ and the second by Equation \eqref{equa:reward_free_concentration}. For brevity, we let $\xi_{h}^{k}=\PP_{h}V_{h+1}^{k}(x_{h}^{k},a_{h}^{k})-V_{h+1}(x_{h+1}^{k})$ in the following.
Recursively applying Equation \eqref{equa:reward_free_22}, we have
\begin{equation*}
V_{1}^{k}(x_{1})\leq\sum_{h=1}^{H-1}\xi_{h}^{k}+(2+1/H)\sum_{h=1}^{H}u_{h}^{k}(x_{h}^{k},a_{h}^{k}).
\end{equation*}
Taking summation on both sides of the above inequality with $k$ from $1$ to $K$, we have
\begin{equation*}
\sum_{k=1}^{K}V_{1}^{k}(x_{1})\leq\sum_{k=1}^{K}\sum_{h=1}^{H-1}\xi_{h}^{k}+(2+1/H)\sum_{k=1}^{K}\sum_{h=1}^{H}u_{h}^{k}(x_{h}^{k},a_{h}^{k}).
\end{equation*}
For the first summation on the right side of the above inequality, we can bound it with Azuma-Hoeffding inequality and have
\begin{equation*}
\sum_{k=1}^{K}\sum_{h=1}^{H-1}\xi_{h}^{k}\leq \cO\Big(\sqrt{H^3  K\log(1/\delta)}\Big),
\end{equation*}
with probability at least $1-\delta/8$.
On the other hand, by Lemma \ref{lemma:telescope}, we have
\begin{equation*}
\sum_{k=1}^{K}\sum_{h=1}^{H}u_{h}^{k}(x_{h}^{k},a_{h}^{k})\leq\cO\Big(\sqrt{d K H^2 \log{K}}\Big),
\end{equation*}
with probability at least $1-\delta/8$. Then, combining the above two inequalities, we obtain that with probability at least $1-\delta/4$, there is 
\begin{equation*}
\sum_{k=1}^{K}V_{1}^{k}(x_{1})\leq \hat{c}(n+R_{\max} )\sqrt{ d^{3}H^{4}K\iota},
\end{equation*}
which completes the proof of Equation \eqref{equa:reward_free_2}.

At last, we prove the conclusion of this lemma that
\begin{equation*}
V_{1}^{*}(x_{1};u^{t})\leq \hat{c}(n+R_{\max} )\sqrt{d^{3}H^{6}\iota/K}.
\end{equation*}
Notice that for all $k\in[K]$,
\begin{equation*}
\Lambda_{h}^{k}\preccurlyeq\Lambda_{h},
\end{equation*} especially when $\zeta_1 = \texttt{EWC}$ and $\Lambda_h$ may further grow during the exploitation phase.
Therefore, we have for all $(h,k)\in[H]\times[K]$,
\begin{equation*}
l_{h}^{k}\geq u_{h}^{t}/H
\end{equation*}
whenever $t \geq k$.
Hence, $V_{1}^{*}(x_{1};u^{t}/H)\leq V_{1}^{*}(x_{1};l^{k})$.
Together with Equation \eqref{equa:reward_free_1} and \eqref{equa:reward_free_2}, we obtain
\begin{equation*}
V_{1}^{*}(x_{1};u^{t})=H V_{1}^{*}(x_{1};u^{t}/H)
\leq H\sum_{k=1}^{K}V_{1}^{k}(x_{1})/K\leq \hat{c}(n+R_{\max} )\sqrt{d^{3}H^{6}\iota/K},
\end{equation*}
which concludes the proof.
\end{proof}

Finally, with Lemmas~\ref{lemma:width bound} and~\ref{lemma:reward_free_explore} in mind, we argue how they can be combined to prove the claims in Lemma~\ref{lemma:basic_lemma} for both when $\zeta_1 = \texttt{ETC}$ and when $\zeta_1 = \texttt{EWC}$.

As the proof techniques are largely the same, let $(\fR, \pi)$ be an arbitrary and fixed pair and we discuss only $\widehat{V}_{1}^{t,\pi}\big(x_{1};\fR\big)-V_{1}^{\pi}\big(x_{1};\fR\big)$ to avoid redundancy. Recalling from Lemma~\ref{lemma:width bound}, we know that
\begin{align*}
\widehat{V}_{1}^{t,\pi}\big(x_{1};\fR\big)-V_{1}^{\pi}\big(x_{1};\fR\big) &\leq 2\sum_{h = 1}^H \EE_{\pi}[u_h^t] \leq 2\hat{c}(n+R_{\max} )\sqrt{d^3H^6\iota/K},
\end{align*} where the second inequality comes from Lemma~\ref{lemma:reward_free_explore}.

\section{Proof of Lower Bound} \label{sec:proof-lower}
In this section, we present the 
proof of the lower bound shown in Theorem \ref{theorem:lowerbound}. 
While the work \citet{2020Mechanism} studies the lower bound for the bandit setting, we remark that deriving the lower bound for our problem is non-trivial, which requires different constructions and proof techniques from that of this earlier work. Specifically, our work focuses on the setting of the stochastic rewards and invalidates the Gaussian reward construction in the proof of Theorem 1 in \citet{2020Mechanism} because of the bounded reward assumption in our MDP setting. We use a different construction with the Bernoulli reward and apply a different anti-concentration analysis.
Moreover, our lower bound considers the linear function approximation and the transition dynamics along the finite horizon in the MDP model which cannot be covered by the bandit setting.

We first show several important lemmas for the proof of Theorem \ref{theorem:lowerbound}. 
The following lemma translates the utilities of the seller and agent $i$ into the differences between the value functions according to Markov VCG mechanism.
\begin{lemma}\label{lemma:utilities_agent_seller}When the actions and prices are chosen according to the Markov VCG mechanism, we have
\begin{equation*}
\begin{aligned}
&u_{i *}=V_{1}^{\pi_{*}}\big(x_1;R\big)-V_{1}^{\pi_{*}^{-i}}\big(x_1;R^{-i}\big),\\
&u_{0 *}=\sum_{i=1}^{n}V_{1}^{\pi_{*}^{-i}}\big(x_1;R^{-i}\big)-(n-1)V_{1}^{\pi_{*}}\big(x_1;R\big).\
\end{aligned}
\end{equation*}
\end{lemma}
\begin{proof}
We can deduce the above results by the definition of the utilities of the agents and the seller. For the utility of agent $i$, we have
\begin{equation*}
\begin{aligned}
u_{i *}&=V_{1}^{\pi_{*}}(x_1;r_{i})-p_{i *}\\
&=V_{1}^{\pi_{*}}\big(x_1;r_{i}\big)-\Big[V_{1}^{\pi_{*}^{-i}}\big(x_1;R^{-i}\big)-V_{1}^{\pi_{*}}\big(x_1;R^{-i}\big)\Big]\\
&=V_{1}^{\pi_{*}}\big(x_1;R\big)-V_{1}^{\pi_{*}^{-i}}\big(x_1;R^{-i}\big).
\end{aligned}
\end{equation*}
For the utility of the seller, we have
\begin{equation*}
\begin{aligned}
u_{0 *}&=V_{1}^{\pi_{*}}(x_1;r_{0})+\sum_{i=1}^{n}p_{i *}\\
&=V_{1}^{\pi_{*}}\big(x_1;r_{0}\big)+\sum_{i=1}^{n}\Big[V_{1}^{\pi_{*}^{-i}}\big(x_1;R^{-i}\big)-V_{1}^{\pi_{*}}\big(x_1;R^{-i}\big)\Big]\\
&=\sum_{i=1}^{n}V_{1}^{\pi_{*}^{-i}}\big(x_1;R^{-i}\big)-(n-1)V_{1}^{\pi_{*}}\big(x_1;R\big),
\end{aligned}
\end{equation*}
where the last equation is by $V_{1}^{\pi_{*}}(x_1;R^{-i}) = V_{1}^{\pi_{*}}(x_1;r_0 + \sum_{j\in[n], j\neq i} r_j) = V_{1}^{\pi_{*}}(x_1;r_0) + \sum_{j\in[n], j\neq i} V_{1}^{\pi_{*}}(x_1; r_j)$.
This completes the proof.
\end{proof}

We then define the estimation of $\sum_{i=1}^{n}V_{1}^{\pi_{*}^{-i}}\big(x_1;R^{-i}\big)$ and the error of this estimation as
\begin{equation*}
Y_{T}=\frac{1}{T}\sum_{i=1}^{n}\sum_{t=1}^{T}\Big(p_{i t}+V_{1}^{\pi^{t}}\big(x_1;R^{-i}\big)\Big),\qquad Z_{T}=Y_{T}-\sum_{i=1}^{n}V_{1}^{\pi_{*}^{-i}}\big(x_1;R^{-i}\big).
\end{equation*}

The next lemma states the relationships between different regret terms defined in Equation \eqref{equa:regret}, which supports the proof of our lower bound.
\begin{lemma}\label{lemma:regret_lower_bound}
Let $\mathrm{Reg}_{T}^{W}, \mathrm{Reg}_{0 T}, \mathrm{Reg}_{T}^{\sharp}$ be defined as 
in Equation \eqref{equa:regret}. Then
\begin{equation*}
\mathrm{Reg}_{T}^{\sharp}=n \mathrm{Reg}_{T}^{W}+T Z_{T},\qquad \mathrm{Reg}_{0 T}=-(n-1)\mathrm{Reg}_{T}^{W}-T Z_{T}.
\end{equation*}
\end{lemma}
\begin{proof}
The proof of this lemma relies on the decomposition of these regret terms.
We first define $h_{i t}:=p_{i t}+V_{1}^{\pi^{t}}\big(x_1;R^{-i}\big)$. Then we have $Y_{T}=\frac{1}{T}\sum_{i=1}^{n}\sum_{t=1}^{T}h_{i t}$. For agent $i$, we have
\begin{equation}\label{equa:fact_1}
\begin{aligned}
u_{i t}&=V_{1}^{\pi^{t}}\big(x_1;r_{i}\big)-p_{i t}\\
&=V_{1}^{\pi^{t}}\big(x_1;r_{i}\big)-\Big(h_{i t}-V_{1}^{\pi^{t}}\big(x_1;R^{-i}\big)\Big)\\
&=V_{1}^{\pi^{t}}\big(x_1;R\big)-h_{i t}.
\end{aligned}
\end{equation}
Combining Lemma \ref{lemma:utilities_agent_seller} and Equation \eqref{equa:fact_1}, we can obtain
\begin{equation*}
\begin{aligned}
u_{i *}-u_{i t}&=\Big(V_{1}^{\pi_{*}}\big(x_1;R\big)-V_{1}^{\pi_{*}^{-i}}\big(x_1;R^{-i}\big)\Big)-\Big(V_{1}^{\pi^{t}}\big(x_1;R\big)-h_{i t}\Big)\\
&=\Big(V_{1}^{\pi_{*}}\big(x_1;R\big)-V_{1}^{\pi^{t}}\big(x_1;R\big)\Big)-\Big(V_{1}^{\pi_{*}^{-i}}\big(x_1;R^{-i}\big)-h_{i t}\Big).
\end{aligned}
\end{equation*}
Then by the definition of $\mathrm{Reg}_{T}^{\sharp}$ in Equation \eqref{equa:regret}, we have
\begin{equation*}
\begin{aligned}
\mathrm{Reg}_{T}^{\sharp}&=\sum_{t=1}^{T}\sum_{i=1}^{n}(u_{i *}-u_{i t})\\
&=\sum_{t=1}^{T}\sum_{i=1}^{n}\Big[\Big(V_{1}^{\pi_{*}}\big(x_1;R\big)-V_{1}^{\pi^{t}}\big(x_1;R\big)\Big)-\Big(V_{1}^{\pi_{*}^{-i}}\big(x_1;R^{-i}\big)-h_{i t}\Big)\Big]\\
&=n\sum_{t=1}^{T}\Big(V_{1}^{\pi_{*}}\big(x_1;R\big)-V_{1}^{\pi^{t}}\big(x_1;R\big)\Big)+T\Big(Y_{T}-\sum_{i=1}^{n}V_{1}^{\pi_{*}^{-i}}\big(x_1;R^{-i}\big)\Big)\\
&=n \mathrm{Reg}_{T}^{W}+T Z_{T}.
\end{aligned}
\end{equation*}
This proves the first claim. For the seller, at time $t$, we have the following observation that
\begin{equation}\label{equa:fact_2}
\begin{aligned}
u_{0 t}&=V_{1}^{\pi^{t}}\big(x_1;r_{0}\big)+\sum_{i=1}^{n}p_{i t}\\
&=V_{1}^{\pi^{t}}\big(x_1;r_{0}\big)+\sum_{i=1}^{n}\Big(h_{i t}-V_{1}^{\pi^{t}}\big(x_1;R^{-i}\big)\Big)\\
&=\sum_{i=1}^{n}h_{i t}-(n-1)V_{1}^{\pi^{t}}\big(x_1;R\big).
\end{aligned}
\end{equation}
Similarly, we can now combine Lemma \ref{lemma:utilities_agent_seller} and Equation \eqref{equa:fact_2} and obtain
\begin{equation*}
\begin{aligned}
\mathrm{Reg}_{0 T}&=\sum_{t=1}^{T}(u_{0 *}-u_{0 t})\\
&=\sum_{t=1}^{T}\Big(V_{1}^{\pi_{*}}\big(x_1;R^{-i}\big)-h_{i t}\Big)+(n-1)\sum_{t=1}^{T}\Big(V_{1}^{\pi^{t}}\big(x_1;R\big)-V_{1}^{\pi_{*}}\big(x_1;R\big)\Big)\\
&=-T Z_{T}-(n-1)R_{T}.
\end{aligned}
\end{equation*}
This completes the proof of the second claim.
\end{proof}

The following lemma about relative entropy gives another useful inequality for our proof of the lower bound.
\begin{lemma}(Bretagnolle-Huber Inequality)\label{lemma:BH-inequality}
Let $\QQ_1$ and $\QQ_2$ be probability measures on the same measurable space $(\Omega,\cF)$, and let $A\in\cF$ be an arbitrary event. Then,
\begin{equation}\label{equa:BHI_1}
\QQ_1(A)+ \QQ_2(A^{c})\geq\frac{1}{2}\exp(-\mathrm{KL}(\QQ_1||\QQ_2)),
\end{equation}
where $A^{c}=\Omega\backslash A$ is the complement of $A$.
\end{lemma}


Now we are ready to prove Theorem \ref{theorem:lowerbound}.

\begin{proof}[Proof of Theorem \ref{theorem:lowerbound}]
At the beginning of the proof, we first state a basic inequality here: for any set of real numbers $\{r_{i}\}_{i\geq1}$, and any set of $\{a_{i}\}_{i\geq1}$ such that $\sum_{i\geq1}a_{i}=1$ and $a_i \geq 0$, we have $\max\{r_{i}\}_{i\geq1}\geq\sum_{i\geq1}a_{i}r_{i}$. Combining the above inequality and Lemma \ref{lemma:regret_lower_bound}, we obtain two lower bounds of $\max\{n \mathrm{Reg}_{T}^{W}, \mathrm{Reg}_{T}^{\sharp}, \mathrm{Reg}_{0 T}\}$. The first one is
\begin{equation*}
\begin{aligned}
\max\{n \mathrm{Reg}_{T}^{W}, \mathrm{Reg}_{T}^{\sharp}, \mathrm{Reg}_{0 T}\}&\geq\frac{4}{5}n \mathrm{Reg}_{T}^{W}+\frac{1}{5}\mathrm{Reg}_{0 T}\\
&=\frac{4}{5}n \mathrm{Reg}_{T}^{W}-\frac{1}{5}\big(-(n-1)\mathrm{Reg}_{T}^{W}-T Z_{T}\big)\\
&\geq \frac{2}{5}n \mathrm{Reg}_{T}^{W}-\frac{1}{5}T Z_{T},
\end{aligned}
\end{equation*}
where we use Lemma \ref{lemma:regret_lower_bound} in the first equality and use the fact that $\mathrm{Reg}_{T}^{W} \geq 0$. Moreover, we obtain another lower bound as
\begin{equation*}
\max\{n \mathrm{Reg}_{T}^{W}, \mathrm{Reg}_{T}^{\sharp}, \mathrm{Reg}_{0 T}\}\geq\frac{2}{5}n \mathrm{Reg}_{T}^{W}+\frac{1}{5}T Z_{T}.
\end{equation*}
Comparing the above two lower bounds of $\max\{n \mathrm{Reg}_{T}^{W}, \mathrm{Reg}_{T}^{\sharp}, \mathrm{Reg}_{0 T}\}$, we have
\begin{equation*}
\max\{n \mathrm{Reg}_{T}^{W}, \mathrm{Reg}_{T}^{\sharp}, \mathrm{Reg}_{0 T}\}\geq\frac{2}{5}n \mathrm{Reg}_{T}^{W}+\frac{1}{5}T |Z_{T}|.
\end{equation*}
For brevity, hereafter, we define $S_{T}:=\frac{2}{5}n \mathrm{Reg}_{T}^{W}+\frac{1}{5}T |Z_{T}|$. Our goal is to obtain a lower bound on $\inf_{Alg}\sup_{\Theta}\EE[S_{T}]$ which is also a lower bound on $\max\{n \mathrm{Reg}_{T}^{W}, \mathrm{Reg}_{T}^{\sharp}, \mathrm{Reg}_{0 T}\}$. To achieve this goal, we construct two problems in $\Theta$ and show that no algorithm can work well on these two problems simultaneously.  

We define the underlying MDP $\cM_{0}$ for the first problem $\theta_{0}$ as follows: $\cM_{0}$ is an episodic MDP with horizon $H\geq2$, state space $\cS=\{x_{0},x_{1},x_{2} , \cdots, x_{n+1},x_{n+2}\}$, and action space $\cA=\{b_{1},b_{2}, \cdots,b_{A}\}$ with $|\cA|=A\geq n+2$. We let the initial state be fixed as $x_0$. 
For the transition kernel, at the first step $h=1$, we set
\begin{equation*}
\begin{aligned}
&\cP_{1}(x_{i}|x_{0},b_{i})=1,\quad \text{for all $i\in\{1,2 , \cdots, n+1\}$}, \\
&\cP_{1}(x_{n+2}|x_{0},b_{i})=1\quad \text{for all $i\in\{n+2 , \cdots, A\}$}.
\end{aligned}    
\end{equation*}
Meanwhile, at any subsequent step $h\in\{2 , \cdots, H\}$, we set
\begin{equation*}
\cP_{h}(x_{i}|x_{i},a)=1,\quad \text{for all $a\in{\cA}$},
\end{equation*}
i.e., state $\{x_{i}\}_{i=1}^{n+2}$ are absorbing states. For the reward function, we let $\mathrm{Ber}(p)$ denote a Bernoulli random variable with success probability $p$ and set
\begin{equation}\label{equa:lower_bound_reward_0}
\begin{aligned}
&r_{0 h}(s,a)=0,\quad \text{for all $(h,s,a)\in \{1 , \cdots, H\}\times\cS\times\cA$},\\
&r_{i 1}(x_{0},a)=0,\quad \text{for all $(i,a)\in [n+2]\times\cA$},\\
&r_{j h}(x_{i},a)\sim\mathrm{Ber}(1/2),\quad \text{for all $j\neq i$ and $(i,h,a)\in [n]\times\{2 , \cdots, H\}\times\cA$},\\
&r_{i h}(x_{i},a)=0,\quad \text{for all $(i,h,a)\in [n]\times \{2 , \cdots, H\}\times\cA$},\\
&r_{j h}(x_{n+1},a)\sim\mathrm{Ber}(1/2),\quad \text{for all $(j,h,a)\in [n]\times \{2 , \cdots, H\}\times\cA$},\\
&r_{j h}(x_{n+2},a)\sim\mathrm{Ber}(1/8),\quad \text{for all $(j,h,a)\in [n]\times \{2 , \cdots, H\}\times\cA$},\\
\end{aligned}
\end{equation}
which means the seller's reward is always $0$.
Please see Figure \ref{figure_1} for an illustration of the construction. 

\begin{figure}[!t]
\centering
\includegraphics[scale=0.3]{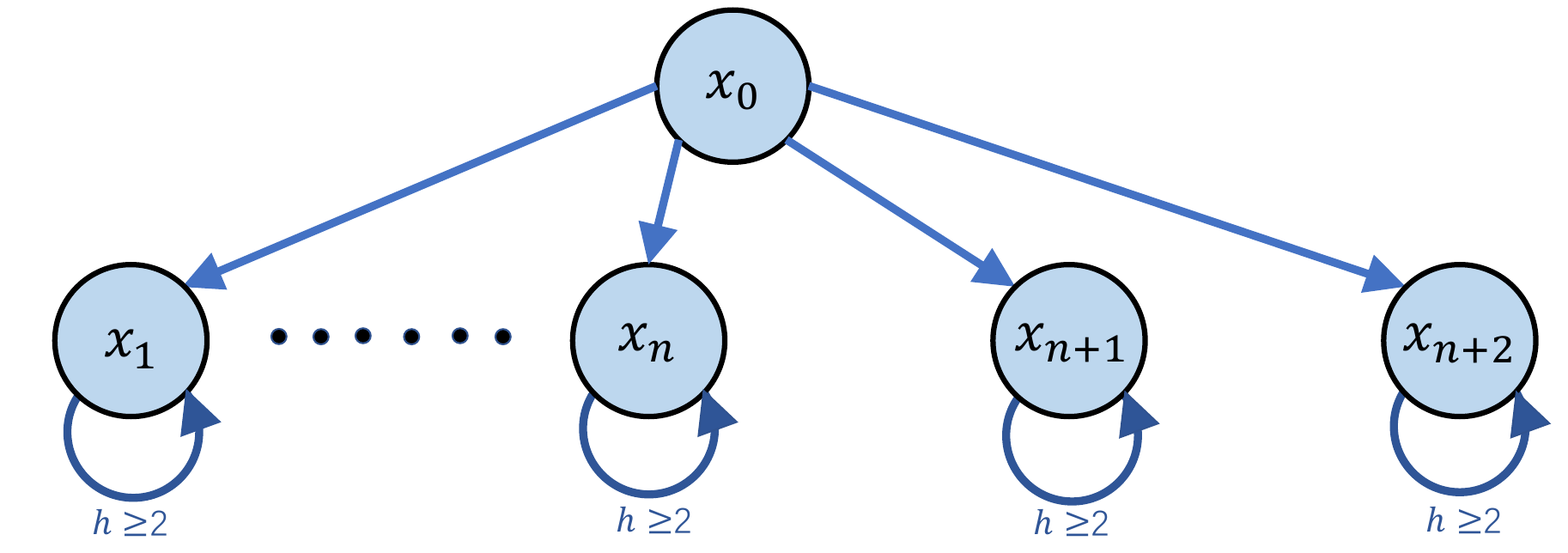}
\caption{An illustration of the episodic MDPs $\cM_{0},\cM_{1}$ with the state space~$\cS = \{ x_0, x_1,  \cdots, x_{n+2}\}$ and action space $\cA= \{ b_j \}_{j=1}^A$. Here we fix the initial state as $x_1 = x_0$, where the agent takes the action $a \in \cA$ and transitions into the second state 
$s_2\in \{x_1 , \cdots, x_{n+2}\}$. In both MDPs, we have the same transition kernel. At the first step $h=1$, the transition kernel satisfies $\cP_{1}(x_{i}|x_{0},b_{i})=1$ for all $i\in\{1,2 , \cdots, n+1\}$ and 
$\cP_{1}(x_{n+2}|x_{0},b_{i})=1$ for all $i\in\{n+2 , \cdots, A\}$.
Also, $x_1, x_2, s x_{n+2} \in \cS$ are the absorbing states. 
The reward functions for $\cM_{0},\cM_{1}$ are showed as in Equations \eqref{equa:lower_bound_reward_0} and \eqref{equa:lower_bound_reward_1}.}
\label{figure_1}
\end{figure}

Note that $\cM_{0}$ is a linear MDP with the dimension $d=n+2$. We set the corresponding feature map $\phi:\cS\times\cA\rightarrow\RR^{d}$ as
\begin{equation*}
\begin{aligned}
&\phi(x_{0},b_{i})=e_{i},\text{for all $i=1,2 , \cdots, n+1$},\\
&\phi(x_{0},b_{i})=e_{n+2},\text{for all $i=n+2 , \cdots, A$},\\
&\phi(x_{i},b_{j})=e_{i},\text{for all $i=1,2 , \cdots, n+1$ and $j\in[A]$},\\
&\phi(x_{i},b_{j})=e_{n+2},\text{for all $i=n+2 , \cdots, A$ and $j\in[A]$},
\end{aligned}
\end{equation*}
where $\{e_{j}\}$ are the canonical basis of $\RR^{n+2}$.
Additionally, if the seller transitions to state $x_{h+1}$, the sum of agents' utilities will be the largest. We can also obtain the following results about problem $\theta_{0}$ directly,
\begin{equation*}
\begin{aligned}
&V_{1}^{\pi_{*}}(x_{0};R)=Q_{1}(x_{0},b_{n+1};R)=\frac{1}{2}n(H-1),\\
&V_{1}^{\pi_{*}^{-i}}(x_{0};R^{-i})=Q_{1}\big(x_{0},b_{i};R^{-i}\big)=\frac{1}{2}(n-1)(H-1),\\
&\sum_{i=1}^{n}V_{1}^{\pi_{*}^{-i}}(x_{0};R^{-i})=\frac{1}{2}n(n-1)(H-1).
\end{aligned}
\end{equation*}
For the rest of this section, we slightly abuse the notation and drop the superscript from the Q-function, as the Q-functions of the different policies we consider are determined by the actions taken by these policies at the first step.

The second problem, i.e., $\theta_{1}$, with the underlying MDP $\cM_{1}$ is nearly the same as $\theta_{0}$ but differs in reward functions at state $x_{i}$ for $i\in[n]$. Then, we define $\theta_{1}$ as
\begin{equation}\label{equa:lower_bound_reward_1}
\begin{aligned}
&r_{j h}(x_{i},a)\sim\mathrm{Ber}(1/2+\delta),\quad \text{for all $j\neq i$ and $(i,h,a)\in [n]\times\{2 , \cdots, H\}\times\cA$},\\
&r_{i h}(x_{i},a)=0,\quad \text{for all $(i,h,a)\in [n]\times \{2 , \cdots, H\}\times\cA$}.
\end{aligned}
\end{equation}
Here we set $\delta\in(0,1/(2n-2))$. The problem $\theta_{1}$ shares the same feature maps $\phi$ and the transition parameters $\mu$ with problem $\theta_{0}$. And the difference lies in the reward parameters. Please see figure \ref{figure_1} for an illustration. Then, we can obtain the following inequalities for problem $\theta_1$,
\begin{equation*}
\begin{aligned}
&V_{1}^{\pi_{*}}(x_{0};R)=Q_{1}(x_{0},b_{n+1};R)=\frac{1}{2}n(H-1),\\
&V_{1}^{\pi_{*}^{-i}}(x_{0};R^{-i})=Q_{1}\big(x_{0},b_{i};R^{-i}\big)=\Big(\frac{1}{2}+\delta\Big)(n-1)(H-1),\\
&\sum_{i=1}^{n}V_{1}^{\pi_{*}^{-i}}(x_{0};R^{-i})=\Big(\frac{1}{2}+\delta\Big)n(n-1)(H-1).
\end{aligned}
\end{equation*}

Specifically, we denote $S_{T}(\theta_{0})$ and $S_{T}(\theta_{1})$ as the $S_{T}$ under problems $\theta_{0}$ and $\theta_{1}$ respectively. 
The expectations and probabilities corresponding to problem $\theta_{i}$ will be denoted as $\EE_{\theta_{i}}$ and $\Pr_{\theta_{i}}$ respectively.
Let $N_{k}(a)=\sum_{\tau=1}^{k}\mathbb{I}\{(a_{1}^{\tau}=a)\}$ denote the number of times that the seller takes action $a$ at the first step in the initial $k$ rounds. Here we rewrite the lower bound of the welfare regret in problem $\theta\in\{\theta_{1},\theta_{2}\}$ as
\begin{equation*}
\begin{aligned}
\EE_{\theta}[\mathrm{Reg}_{T}^{W}]&=\sum_{j=1,j\neq n+1}^{n+2}\big(Q_{1}(x_{0},b_{n+1};R)-Q_{1}(x_{0},b_{j};R)\big)\EE_{\theta}[N_{K}(b_{j})]\\
&\geq \sum_{j=1}^{n}\big(Q_{1}(x_{0},b_{n+1};R)-Q_{1}(x_{0},b_{j};R)\big)\EE_{\theta}[N_{K}(b_{j})].
\end{aligned}
\end{equation*}
Observing that $Q_{1}(x_{0},b_{n+1};R)-Q_{1}(x_{0},b_{j};R)=(H-1)/2$ in problem $\theta_{0}$, and that $|Z_{T}|$ is at least $n(n-1)(H-1)/2$ when $Y_{T}>[n^{2}/2-n/2+n(n-1)\delta/2](H-1)$, we get the following lower bound of $\EE_{\theta_{0}}[S_{T}(\theta_{0})]$ as
\small
\begin{align}
&\EE_{\theta_{0}}[S_{T}(\theta_{0})]\nonumber\\
&\geq \frac{2}{5}n \mathrm{Reg}_{T}^{W}+\frac{1}{5}T|Z_{T}|\nonumber\\
&\geq\frac{2}{5}n \sum_{j=1}^{n}\frac{H-1}{2}\EE_{\theta_{0}}[N_{K}(b_{j})]+\frac{T}{5}\frac{n(n-1)(H-1)\delta}{2}\Pr_{\theta_{0}}\Big(\underbrace{Y_{T}>\Big[\frac{n^{2}}{2}-\frac{n}{2}+\frac{n(n-1)\delta}{2}\Big](H-1)}_{\displaystyle\mathrm{event }~ E}\Big)\nonumber\\
&\geq\frac{n(H-1)}{10}\Big[\sum_{j=1}^{n}2\EE_{\theta_{0}}[N_{K}(b_{j})]+T(n-1)\delta\Pr_{\theta_{0}}(E)\Big].\label{equa:q_lower_bound_1}
\end{align}
\normalsize
In problem $\theta_{1}$, we have $|Z_{T}|$ is at least $n(n-1)(H-1)/2$ when $Y_{T}\leq[n^{2}/2-n/2+n(n-1)\delta/2](H-1)$. We drop the welfare regret, which is positive, in the analysis and use the above statement regarding $Y_{T}$ under the event $E^{c}$ in problem $\theta_{1}$ to obtain
\begin{equation}\label{equa:q_lower_bound_2}
\EE_{\theta_{1}}[S_{T}(\theta_{1})]\geq\frac{n(H-1)}{10}T(n-1)\delta\Pr_{\theta_{1}}(E^{c}).
\end{equation}
Applying Lemma \ref{lemma:BH-inequality} to $\Pr_{\theta_{0}}(E)+\Pr_{\theta_{1}}(E^{c})$, we have
\begin{equation*}
\Pr_{\theta_{0}}(E)+\Pr_{\theta_{1}}(E^{c})\geq\frac{1}{2}\exp(-\mathrm{KL}(\Pr_{\theta_{0}}^T||\Pr_{\theta_{1}}^T)),
\end{equation*}
where we slightly abuse the notation and let $ \Pr_{\theta_{0}}^T$ and $\Pr_{\theta_{1}}^T$ denote the probability distribution of the observed rewards up to time $T$ in problem $\theta_{0}$ and $\theta_{1}$ respectively. We also notice that if the seller takes action $b_{n+1},b_{n+2}$ at the first step, then $\Pr_{\theta_{0}}^T= \Pr_{\theta_{1}}^T$. If the seller take action $b_{i}$ for $i\in\{1,2, s n\}$ in the first step, then the reward distributions of agent $i$ are the same in both $\theta_{0}$ and $\theta_{1}$.However, for other agents $j\neq i$, the KL divergence between the corresponding distributions in the two problems is $-\log(1-4\delta^{2})(H-1)$ since the rewards are mutually independent and the KL divergence between $\mathrm{Ber}(1/2)$ and $\mathrm{Ber}(1/2+\delta)$ is $-\log(1-4\delta^{2})$. Then we have
\begin{equation}\label{equa:KL_1}
\mathrm{KL} (\Pr_{\theta_{0}}^T||\Pr_{\theta_{1}}^T)=-(n-1)(H-1)\log(1-4\delta^{2})\sum_{j=1}^{n}\EE_{\theta_{0}}[N_{K}(b_{j})].
\end{equation}
By combining Equations \eqref{equa:q_lower_bound_1}, \eqref{equa:q_lower_bound_2},\eqref{equa:BHI_1}, and \eqref{equa:KL_1}, we obtain the lower bound for $\EE_{\theta_{0}}[S_{T}(\theta_{0})]+\EE_{\theta_{1}}[S_{T}(\theta_{1})]$ as
\begin{equation*}
\begin{aligned}
&\EE_{\theta_{0}}[S_{T}(\theta_{0})]+\EE_{\theta_{1}}[S_{T}(\theta_{1})]\\
&\qquad\geq\frac{n(H-1)}{10}\Big[\sum_{j=1}^{n}2\EE_{\theta_{0}}[N_{K}(b_{j})]+T(n-1)\delta\big(\Pr_{\theta_{0}}(E)+\Pr_{\theta_{1}}(E^{c})\big)\Big]\\
&\qquad\geq\frac{n(H-1)}{10}\Big[2\sum_{j=1}^{n}\EE_{\theta_{0}}[N_{K}(b_{j})]\\
&\qquad\qquad +\frac{1}{2}T(n-1)\delta\exp\Big((n-1)(H-1)\log(1-4\delta^{2})\sum_{j=1}^{n}\EE_{\theta_{0}}[N_{K}(b_{j})]\Big)\Big]\\
&\qquad\geq\frac{n(H-1)}{10}\min\Big\{\underbrace{2x+\frac{1}{2}T(n-1)\delta\exp\Big((n-1)(H-1)\log(1-4\delta^{2})x\Big)}_{\displaystyle :=f(x)}\Big\},
\end{aligned}
\end{equation*}
where we combine Equation \eqref{equa:q_lower_bound_1} and \eqref{equa:q_lower_bound_2} in the first inequality, and the second inequality is by Equation \eqref{equa:BHI_1} and Equation \eqref{equa:KL_1}. For the last step we substitute $\sum_{j=1}^{n}\EE_{\theta_{0}}[N_{K}(b_{j})]$ by $x$ and turn to find the minimum value of the function $f(x)$. Then, we have $$x_{0}=\frac{-1}{(n-1)(H-1)\log(1-4\delta^{2})}\log\Big(\frac{-T(n-1)^{2}(H-1)\delta\log(1-4\delta^{2})}{4}\Big)$$ as the minimum   of $f(x)$. Thus, we have 
\begin{equation}\label{equa:lower_bound_1_Q}
\begin{aligned}
\EE_{\theta_{0}}[S_{T}(\theta_{0})]+\EE_{\theta_{1}}[S_{T}(\theta_{1})]&\geq\frac{n(H-1)}{10}2 x_{0}\\
&\geq\frac{-1}{5\log(1-4\delta^{2})}\log\Big(\frac{-T(n-1)^{2}(H-1)\delta\log(1-4\delta^{2})}{4}\Big).
\end{aligned}
\end{equation}
Using the basic inequality $x/(1+x)\leq\log(1+x)\leq x$ for $x>-1$, we have
\begin{equation*}
-4\delta^{2}\geq\log(1-4\delta^2)\geq\frac{-4\delta^2}{1-4\delta^2}\geq-8\delta^2,
\end{equation*}
when $0\leq\delta^2\leq 1/8$.
Combining Equation \eqref{equa:lower_bound_1_Q} and the above inequality, we obtain
\begin{equation*}
\begin{aligned}
\EE_{\theta_{0}}[S_{T}(\theta_{0})]+\EE_{\theta_{1}}[S_{T}(\theta_{1})]&\geq\frac{-1}{5(-8\delta^2)}\log\bigg(\frac{-T(n-1)^{2}(H-1)\delta(-4\delta^2)}{4}\bigg)\\
&=\frac{1}{40\delta^2}\log\Big(T(n-1)^{2}(H-1)\delta^3\Big).
\end{aligned}
\end{equation*}
Finally, we choose $\delta=\Big({1}/\big(T(n-1)^{2}(H-1)\big)\Big)^{1/3}$ to obtain the lower bound
\begin{equation*}
\frac{1}{2}\big(\EE_{\theta_{0}}[S_{T}(\theta_{0})]+\EE_{\theta_{1}}[S_{T}(\theta_{1})]\big)\geq c  n^{4/3}H^{2/3}T^{2/3}  ,
\end{equation*}
for some absolute constant $c$. Here $\delta\in(0,1/(2n-2))$ is satisfied when $T\geq 8(n-1)/(H-1)$ and $\delta^2\in(0,1/8)$ is satisfied when $n\geq3$. Observing that
\begin{equation*}
\sup_{\theta\in\Theta}\EE[S_{T}(\theta)]\geq \max\{\big(\EE_{\theta_{0}}[S_{T}(\theta_{0})]+\EE_{\theta_{1}}[S_{T}(\theta_{1})]\big)\}\geq\frac{1}{2}\big(\EE_{\theta_{0}}[S_{T}(\theta_{0})]+\EE_{\theta_{1}}[S_{T}(\theta_{1})]\big)
\end{equation*} we have the conclusion that
\begin{equation*}
\inf_{Alg} \sup_\Theta \;\EE\left[\max\big( n \mathrm{Reg}_{T}^{W},\mathrm{Reg}_{T}^{\sharp},\mathrm{Reg}_{0 T})\right] \geq \Omega(n^{4/3}H^{2/3}T^{2/3}  ).
\end{equation*}

On the other hand, noting that $\max\big( n \mathrm{Reg}_{T}^{W},\mathrm{Reg}_{T}^{\sharp},\mathrm{Reg}_{0 T}) \geq n \mathrm{Reg}_{T}^{W}$ always holds, we have
\begin{align} \label{eq:lowerbound-3}
\max\big( n \mathrm{Reg}_{T}^{W},\mathrm{Reg}_{T}^{\sharp},\mathrm{Reg}_{0 T})\geq n \mathrm{Reg}_{T}^{W} = n \left[T    V_1^{*}(x_{1};R)-\sum_{t=1}^{T}V_{1}^{\pi^{t}}(x_{1};R)\right],
\end{align}
where we recall $V_1^{*}(x_{1};r) := \max_\pi V^{\pi}(x_{1};r)$ for any reward function $r$. Since $R = \sum_{i=0}^n r_i$, we consider a simple hard instance that $ r_1 = r_2= s =r_n =r'$ and $r_0 = R_{\max} \times r'$, where $r': \cS\times\cA\mapsto[0,1]$ is some reward function. In other words, here we consider an instance with the same reward function for all $r_i, 1\leq i \leq n$, and $r_0$ is simply the same reward function scaled by $R_{\max}$. Under this setting, by \eqref{eq:lowerbound-3}, we have
\begin{align*}
\max\big( n \mathrm{Reg}_{T}^{W},\mathrm{Reg}_{T}^{\sharp},\mathrm{Reg}_{0 T})]&\geq  n \left[T   V_1^{*}(x_{1};R)-\sum_{t=1}^{T}V_{1}^{\pi^{t}}(x_{1};R)\right] \\
&=  n(n+R_{\max}) \left[T    V_1^{*}(x_{1};r')-\sum_{t=1}^{T}V_{1}^{\pi^{t}}(x_{1};r')\right].
\end{align*}
The above inequality implies that the lower bound of $\max\big( n \mathrm{Reg}_{T}^{W},\mathrm{Reg}_{T}^{\sharp},\mathrm{Reg}_{0 T})]$ can be further lower bounded by the lower bound of the regret for linear MDPs of dimension $d$ with rewards in $[0,1]$. Theorem 1 in \citet{zhou2020nonstationary} shows that for any algorithm, if $d \geq 4$ and $T \geq  64(d - 3)^2 H$, then there exists at least one linear MDP instance that incurs regret at least $\Omega(d\sqrt{HT})$. Therefore, we can further obtain that under the same assumptions, the minimax lower bound for $\max\big( n \mathrm{Reg}_{T}^{W},\mathrm{Reg}_{T}^{\sharp},\mathrm{Reg}_{0 T})]$ is at least $\Omega\big(n(n + R_{\max})d\sqrt{HT}\big)$, i.e., 
\begin{align*}
\inf_{\mathsf{Alg}} \sup_\Theta \;\EE\left[\max\big( n \mathrm{Reg}_{T}^{W},\mathrm{Reg}_{T}^{\sharp},\mathrm{Reg}_{0 T})\right] \geq \Omega\bigg(n(n + R_{\max})d\sqrt{HT}\bigg).
\end{align*}
Combining the above results together, we have the following lower bound as 
\begin{align*}
\inf_{\mathsf{Alg}} \sup_\Theta \;\EE\left[\max\big( n \mathrm{Reg}_{T}^{W},\mathrm{Reg}_{T}^{\sharp},\mathrm{Reg}_{0 T})\right] \geq \Omega\bigg(n^{4/3} H^{2/3} T^{2/3} +n(n + R_{\max})d\sqrt{HT}\bigg).
\end{align*}
This concludes the proof of Theorem \ref{theorem:lowerbound}. 
\end{proof}



\section{Other Supporting Lemmas}\label{sec:proof-end}
The following lemma from \citet{abbasi2011improved} establishes the concentration of self-normalized processes.

\begin{lemma}[Concentration of Self-Normalized Processes]\label{lemma:concentration_abbasi}
Let $\{\cF_t \}^\infty_{t=0}$ be a filtration and $\{\epsilon_t\}^\infty_{t=1}$ be an $\RR$-valued stochastic process such that $\epsilon_t$ is $\cF_{t} $-measurable for all $t\geq 1$.
Moreover, suppose that conditioning on $\cF_{t-1}$, 
$\epsilon_t $ is a  zero-mean and $\sigma$-sub-Gaussian random variable for all $t\geq 1$, that is,  
\begin{equation*}
\EE[\epsilon_t\given \cF_{t-1}]=0,\qquad \EE\bigl[ \exp(\lambda \epsilon_t) \biggiven \cF_{t-1}\bigr]\leq \exp(\lambda^2\sigma^2/2) , \qquad \forall \lambda \in \RR. 
\end{equation*}
Meanwhile, let $\{\phi_t\}_{t=1}^\infty$ be an $\RR^d$-valued stochastic process such that  $\phi_t $  is $\cF_{t -1}$-measurable for all $ t\geq 1$. 
Also, let  $M_0 \in \RR^{d\times d}$ be a  deterministic positive-definite matrix and 
\begin{equation*}
M_t = M_0 + \sum_{s=1}^t \phi_s\phi_s^\top
\end{equation*}
for all $t\geq 1$. For all $\delta>0$, it holds that
\begin{equation*}
\Big\| \sum_{s=1}^t \phi_s \epsilon_s \Big\|_{ M_t ^{-1}}^2 \leq 2\sigma^2   \log \Bigl( \frac{\det(M_t)^{1/2}  \det(M_0)^{- 1/2}}{\delta} \Bigr)
\end{equation*}
for all $t\ge1$ with probability at least $1-\delta$.
\end{lemma}

\begin{lemma} [\citet{abbasi2011improved}]\label{lemma:telescope}
Let $\{\bphi_t \}_{t\geq 0}$ be a bounded sequence in $\RR^d$  satisfying $\sup_{t\geq 0}\| \bphi_t \| \leq 1$. Let $\Lambda_0  \in \RR^{d\times d}$ be a positive definite matrix.  For any $t\geq 0$, we define $
\Lambda_t = \Lambda_0 + \sum_{j = 1}^t \bphi_ j ^\top   \bphi_j$.  Then, if the smallest eigenvalue of $\Lambda_0$  satisfies $\lambda_{\min}(\Lambda_0) \geq 1 $, we have 
$$
\log \biggl [  \frac{\det(
\Lambda_t )}{\det(\Lambda_0)}\biggr ] \leq \sum_{j=1}^{t} \bphi_j^\top \Lambda_{j-1} ^{-1} \bphi_j \leq 2 \log \biggl [  \frac{\det(
\Lambda_t )}{\det(\Lambda_0)}\biggr ].
$$ 
\end{lemma}

The following lemma from \citet{2019Provably0} depicts the difference between an estimated value function and the value function under a certain policy.

\begin{lemma}[Extended Value Difference \citep{2019Provably0}]\label{lemma:ext_val_diff}
Let $\pi = \{ \pi _h \}_{h =1}^H $ and $\pi' = \{ \pi_h' \}_{ h = 1}^H  $ be any two policies and let $\{ \hat Q_h \}_{h=1}^H $ be any estimated Q-functions. 
For all $h \in [H]$, we define the estimated value function $\hat V_h  \colon \cS\mapsto \RR$  by setting $\hat V_h (x) = \langle \hat Q_h (x,   ), \pi_h (  \given x ) \rangle_{\cA}$ for all $x \in \cS$. 
For all $x \in \cS$, we have 
\begin{equation*}
\begin{aligned}
\widehat{V}_1(x) - V_1^{\pi' }(x) &= \sum_{h=1}^H \EE_{\pi' }\big[ \langle \hat{Q}_h (x_h, ) , \pi_h( \given x_h) - \pi'_h( \given x_h)\rangle_{\cA } \biggiven x_1=x\big]\\
&\qquad + \sum_{h=1}^H\EE_{\pi' }\big[     \hat{Q}_h (x_h,a_h)  - (\BB_h \widehat{V}_{h+1} )(x_h,a_h)  \biggiven x_1=x \big],
\end{aligned}
\end{equation*}
where $\EE_{\pi' } $ is taken with respect to the trajectory generated by $\pi'$, while $\BB_h$ is the Bellman operator defined in Equation \eqref{equa:Bellman operator}.
\end{lemma}

The following lemma controls the norms of the $w$'s generated by either Algorithm~\ref{algorithm:L3} or Algorithm~\ref{algorithm:L4} and is used heavily for the concentration analysis.
\begin{lemma}[Bounded Weights of Value Functions \citep{2020Is}] \label{lemma:bound_weight_of_bellman}
Let $V_{\max}>0$ be an absolute constant. For any function $V:\cS\to [0,V_{\max}], h\in [H],$ and $(\fR, \pi) \in \{(R, \widehat{\pi}), (\Tilde{R}, \Tilde{\pi}^{\ddagger}_t)\} \cup \{(r_i + \Tilde{R}^{-i}, \Tilde{\pi}_t^{\dagger i}), (R^{-i}, *), (\Tilde{R}^{-i}, \dagger), (\Tilde{R}^{-i}, \ddagger), (R^{-i}, \widehat{\pi}^t), (\Tilde{R}^{-i}, \Tilde{\pi}^{\dagger i}_t), (\Tilde{R}^{-i}, \Tilde{\pi}_t^{\ddagger})\}_{i =1 }^n$, we have 
\begin{equation*}
\|w_h\|\leq \|\theta_h\|+V_{\max}\sqrt{d},\qquad \|\hat{w}_{h}^{t,\pi}\|,\|\check{w}_{h}^{t,\pi}\big\|\leq (n+R_{\max})H\sqrt{K d/\lambda},
\end{equation*}
where $\hat{w}_{h}^{t,\pi}, \check{w}_{h}^{t,\pi}$ are the linear weights associated with the pair $(\fR, \pi)$, $w_h$ parameterizes $(\BB_hV)( ,  ; \fR)$, and $\theta_h$ parameterizes $\fR$.
\end{lemma}
\begin{proof}
Observe that in our setting, the absolute value of the empirical observations of $(\BB_hV)( ,  ; \fR)$ is instead
$|\fR_{h}^{\tau}+\widehat{V}_{h+1}^{t,\pi}( ; \fR)|$, which is upper bounded by $2(n+R_{\max})H$. Rescaling the Lemma B.1 of \citet{2020Is} completes the proof.
\end{proof}

\begin{lemma}
For all $h\in [H]$ and all $\varepsilon > 0$,we have 	\label{lem:covering_num}
\begin{equation*}
\log | \cN_h (  \varepsilon; L, B, \lambda)  | \leq d   \log (1+ 4 L /  \varepsilon  ) + d^2     \log\bigl(1+ 8 d^{1/2} B^2 / ( \varepsilon^2\lambda) \bigr),
\end{equation*}
where the function class 
\begin{equation*}
\begin{aligned}
&\cV_{h}(L,B,\lambda) =\big\{V_h(x;\theta,\beta,\Sigma)\colon \cS\to [0,(n+R_{\max})  H]~\text{with}~\|\theta\|\leq L, \beta\in [0,B], \Sigma \succeq \lambda  I   \big\}\\
&\text{with~~}V_h(x;\theta,\beta,\Sigma) = \max_{a\in \cA}  \Bigl\{ \min\bigl \{ \phi(x,a)^\top \theta + \beta  \sqrt{ \phi(x,a)^\top \Sigma^{-1}\phi(x,a) },(n+R_{\max})  H\bigr \} \Bigr\}
\end{aligned}
\end{equation*} and $ \cN_{h} (\varepsilon; L, B, \lambda)$ is the 
$\varepsilon$-cover of $\cV_h(L, B, \lambda)$ with respect to the distance $\dist(V,V^{\prime})=\sup_{x\in\cS}\big\|V(x)-V^{\prime}(x)\big\|$.
\end{lemma}
\begin{proof}
See Lemma D.6 in \citet{2019Provably} for a detailed proof. 
\end{proof}

\begin{lemma}[Concentration of Self-Normalized Processes]\label{lemma:concentration_of_self_norm_pro}
Let $V\colon\cS\mapsto[0,(n+R_{\max}) (H-1)]$ be any fixed function. For any $h\in[H], p\in(0,1)$, and reward function $r$, we have
\begin{equation*}
\Pr \biggl(  \Big\|   \sum_{\tau=1}^{K} \phi(x_h^\tau,a_h^\tau)   \epsilon_h^\tau(V; r) \Big\|_{(\Lambda_h^{t})^{-1}}^2 > (n+R_{\max})^2 H^2  \bigl (  2    \log(1/ p ) + d  \log(1+K/\lambda)\big) 
\biggr) \leq p.
\end{equation*}
\end{lemma}
\begin{proof}
For the fixed $h\in [H]$ and all $\tau \in \{ 0,  s , K \}$, we define the $\sigma$-algebra 
\begin{equation*}
\cF_{h,\tau} = \sigma\big( \{ (x_h^j,a_h^j,x_{h+1}^j) \} _{j=1}^{\tau}  \cup (x_h^{(\tau+1) \wedge K },a_h^{(\tau+1) \wedge K }) \big),
\end{equation*}
where $\sigma(\cdot)$ denotes the $\sigma$-algebra generated by a set of random variables and $(\tau + 1) \wedge K$ denotes $\min \{ \tau + 1, K \}$. For all $\tau \in [K]$, we have $\phi (x_h^\tau, a_h^\tau) \in \cF_{h, \tau - 1 }$, as $(x_h^\tau, a_h^\tau)$ is $\cF_{h, \tau-1}$-measurable.  
Also, for the fixed function $V\colon\cS\mapsto[0,(n+R_{\max})  (H-1)]$ and all $\tau \in [K]$, we have 
\begin{equation*}
\epsilon_{h}^{\tau}(V; r)=r_{h}^{\tau}+V(x_{h+1}^\tau; r) - (\BB_h V) (x_{h}^\tau,a_h^\tau;r)\in \cF_{h, \tau },
\end{equation*}
as $(x_h^\tau,a_h^\tau, x_{h+1}^\tau)$ is $\cF_{h, \tau}$-measurable. Hence,  
$\{ \epsilon_{h}^\tau (V) \}_{\tau= 1}^K $ is a stochastic process adapted to the filtration $\{\cF_{h, \tau}\}_{\tau=0}^K $. Furthermore, we have
\begin{equation*}
\begin{split}
\EE \big[\epsilon_h^\tau (V; r) \biggiven \cF_{h,\tau-1}\big] &= \EE \bigl[r_{h}^{\tau}+V(x_{h+1}^\tau; r) \biggiven \{ (x_h^j,a_h^j,x_{h+1}^j) \} _{j=1}^{\tau-1}  ,(x_h^{\tau},a_h^{\tau}) \bigr]  - (\BB_h V) (x_h^\tau , a_h^\tau ; r)\\
& = \EE \bigl [r_{h}^{\tau}+V(s_{h+1}) \biggiven s_h = x_h^\tau , a_h = a_h^\tau  \bigr ] -  (\BB_h V) (x_h^\tau , a_h^\tau ; r) = 0,
\end{split}
\end{equation*}
where the first step is because $(\BB_h V) (x_h^\tau , a_h^\tau; r )$ is $\cF_{h,\tau-1}$-measurable and the second step follows from the Markov property of the process. Moreover, as $(\BB_h V)(x_h^\tau,a_h^\tau;r)\in [0,(n+R_{\max})  H]$, we have $|\epsilon_h^\tau (V; r) | \leq  (n+R_{\max})  H$. Hence, the random variable $\epsilon_h^\tau (V; r) $ defined in Equation \eqref{equa:define eps V} is mean-zero and $(n+R_{\max})H$-sub-Gaussian conditioning on $\cF_{h, \tau - 1}$.

Invoke Lemma \ref{lemma:concentration_abbasi} with $M_0=\lambda   I$ and $M_k  = \lambda    I + \sum_{\tau =1}^k \phi(x_h^\tau,a_h^\tau)\ \phi(x_h^\tau,a_h^\tau)^\top$
for all $k \in [K]$. We then know that
\begin{equation}\label{equa:concentration_1}
\Pr  \bigg( \Big\|   \sum_{\tau=1}^{K} \phi(x_h^\tau,a_h^\tau)  \epsilon_h^\tau(V; r) \Big\|_{(\Lambda_h^t)^{-1}}^2  >   2 (n+R_{\max})^2 H^2    \log \Big(  \frac{\det(\Lambda_h^t)^{1/2}}{p    \det( \lambda   I) ^{1/2} }  \Big)   \bigg ) \leq   p
\end{equation}
for all $p \in (0,1)$. Here, we use the fact that $\Lambda_{h}^{t}=M_{k}$. To upper bound $\det(\Lambda_{h}^{t})^{1/2}$, we first notice that
\begin{equation*}
\|\Lambda_h^t\|_{\oper}  =   \Big\|\lambda    I + \sum_{\tau=1}^K \phi(x_h^\tau,a_h^\tau)\phi(x_h^\tau,a_h^\tau)^\top \Big\| _{\oper} \leq  \lambda  +  \sum_{\tau = 1} ^K  \| \phi(x_h^\tau,a_h^\tau)\phi(x_h^\tau,a_h^\tau)^\top  \|_{\oper} \leq \lambda  + K,
\end{equation*}
where the first inequality follows from the triangle inequality of operator norm and the second inequality follows from the fact that $\|\phi(x,a)\|\leq 1$ for all $(x,a )\in \cS\times \cA$ by our assumption. This implies $\det(\Lambda_h^t)\leq (\lambda+K)^d$. Combining with the fact that $\det(\lambda  I)=\lambda^d$ and Equation \eqref{equa:concentration_1}, we have
\begin{equation*}
\Pr \biggl(  \Big\|   \sum_{\tau=1}^{K} \phi(x_h^\tau,a_h^\tau)  \epsilon_h^\tau(V; r)\Big\|_{(\Lambda_h^{t})^{-1}}^2 > (n+R_{\max})^2 H^2  \bigl (  2    \log(1/ p ) + d  \log(1+K/\lambda)\big) 
\biggr) \leq p.
\end{equation*}
Therefore, we conclude the proof of Lemma \ref{lemma:concentration_of_self_norm_pro}.
\end{proof}

\end{document}